\documentclass{article}

\usepackage[table]{xcolor}
\usepackage{iclr2027_conference,times}
\usepackage[T1]{fontenc}
\usepackage{amsmath}
\usepackage{amssymb}
\usepackage{amsthm}
\usepackage{graphicx}
\usepackage{booktabs}
\usepackage{makecell}
\usepackage{longtable}
\usepackage{multirow}
\usepackage{pdflscape}
\usepackage{array}
\usepackage{placeins}
\usepackage{enumitem}
\usepackage[mode=buildmissing]{standalone}
\usepackage{wrapfig}
\usepackage{float}
\usepackage{mathtools}

\usepackage{algorithm}
\usepackage{algorithmic}

\usepackage{comment}
\usepackage{etoolbox}
\usepackage{hyperref}
\usepackage{url}
\usepackage{xspace}

\newtheorem{theorem}{Theorem}
\newtheorem{lemma}{Lemma}
\newtheorem{definition}{Definition}
\newtheorem{proposition}{Proposition}
\newtheorem{corollary}{Corollary}
\newtheorem{example}{Example}
\newtheorem{remark}{Remark}

\newcommand{\our}{Cluster-TNN\xspace}
\newcommand{\tb}{TopoBench\xspace}

\newcommand{\graph}{\mathcal{G}}
\newcommand{\vertices}{\mathcal{V}}
\newcommand{\edges}{\mathcal{E}}

\newcommand{\lifting}{\mathcal{L}}
\newcommand{\structlifting}{\mathcal{L}_{S}}
\newcommand{\featurelifting}{\mathcal{L}_{F}}

\title{Batch Before You Lift: Scalable Topological Deep Learning on Large Graphs}

\iclrfinalcopy
\renewcommand{\headrulewidth}{0pt}
\makeatletter
\renewcommand{\@maketitle}{\vbox{\hsize\textwidth
  {\centering\LARGE\sc\@title\par}
  \vskip 0.12in
  {\centering\normalsize\@author\par}
  \vskip 0.3in minus 0.1in}}
\makeatother

\author{%
\textbf{David Leko}\textsuperscript{1,}\thanks{Equal contribution. For individual author contributions, see the \hyperref[sec:author_contributions]{Author contributions} section. Correspondence: david.leko@epfl.ch and lbenic@irb.hr.}\quad
\textbf{Luka Beni\'c}\textsuperscript{2,3,*}\quad
\textbf{Guillermo Bern\'a{}rdez}\textsuperscript{4}\\
\textbf{Nina Miolane}\textsuperscript{4}\quad
\textbf{Olga Fink}\textsuperscript{1}\quad
\textbf{Lev Telyatnikov}\textsuperscript{5}\\
\textsuperscript{1}IMOS Lab, EPFL, Lausanne, Switzerland\\
\textsuperscript{2}Ru\dj{}er Bo\v{s}kovi\'c Institute, Zagreb, Croatia\\
\textsuperscript{3}University of Zagreb, Zagreb, Croatia\\
\textsuperscript{4}UC Santa Barbara, Santa Barbara, USA\\
\textsuperscript{5}Capital One\thanks{Research was completed prior to Capital One.}}
\hypersetup{pdfauthor={David Leko, Luka Beni\'c, Guillermo Bern\'a{}rdez, Nina Miolane, Olga Fink, Lev Telyatnikov},
  pdftitle={Batch Before You Lift: Scalable Topological Deep Learning on Large Graphs}}
\begin{document}

\maketitle

\begin{abstract}
Topological Deep Learning extends graph-based learning to higher-order domains, such as hypergraphs, cellular, and simplicial complexes. These domains are typically constructed from patterns in an input graph through a process of graph lifting. Full-domain training constructs and stores the complete lifted representation before model execution. On large and dense datasets like Reddit ($233$k nodes and $57.3$M edges), this global materialization becomes a severe computational bottleneck, often rendering training infeasible. To address this limitation, we introduce \our, a domain-agnostic framework that avoids this bottleneck by lifting locally instead. After partitioning the input graph during preprocessing, at runtime \our dynamically samples groups of node clusters, reconstructs their induced subgraphs to form mini-batches, and applies the chosen lifting within each mini-batch. Retaining all edges among sampled nodes preserves the connectivity needed to construct higher-order structures across clusters, producing topological mini-batches that existing Topological Neural Networks can process directly. Across $21$ matched comparisons with full-graph execution, \our reduces peak GPU memory in every configuration, by $83.2\%$ on average while maintaining competitive predictive performance. Notably, such a reduction enables, to our knowledge, the first training of multiple different higher-order Topological Neural Networks on large datasets such as Reddit and OGBN Products. These results establish \our as a general strategy for scaling Topological Deep Learning beyond the limitations of global domain construction.

\end{abstract}

\section{Introduction}
\label{sec:introduction}

Graphs are a dominant representation for relational learning: nodes encode entities and edges encode pairwise relations. Graph Neural Networks (GNNs) learn on such structures across diverse machine learning tasks~\citep{kipf2017gcn,Hamilton2017RepresentationLO,gilmer,bronstein2021geometric,Corso2024GNN}. However, many systems involve multi-way interactions, such as coauthorship groups, molecular rings, or many-body couplings, that pairwise graphs cannot represent explicitly. Topological Deep Learning (TDL) extends relational learning to higher-order domains, such as hypergraphs, cellular, and simplicial complexes, on which Topological Neural Networks (TNNs) operate~\citep{position,hajij2023tdl,papillon2023architectures}. By explicitly representing higher-order relations, these domains provide complementary inductive biases for systems whose structure extends beyond pairwise connectivity.


However, scalability in higher-order domains remains a central challenge for TDL (\citet{position}, Open Problem~\(6\)). For graph-supported data, a \emph{lifting} maps patterns in the input graph, such as neighbourhoods, cliques, or cycles, to explicit higher-order elements and their incidence relations. This lifting determines both which higher-order relations become available to the model and the size of the resulting computational domain. Because the number of lifted elements depends on graph connectivity rather than node count alone, the lifted domain can be orders of magnitude larger than the input graph. For example, the Reddit dataset has approximately $233$k nodes and $57.3$M edges, yet clique-complex lifting produces approximately $8.36$B simplicial $2$-cells (filled triangles). Enumerating these elements, their features, and their incidence relations may exceed available memory before training begins, making full-domain construction a scalability bottleneck.

\begin{figure*}[t!]
    \centering
    \includegraphics[width=\textwidth]{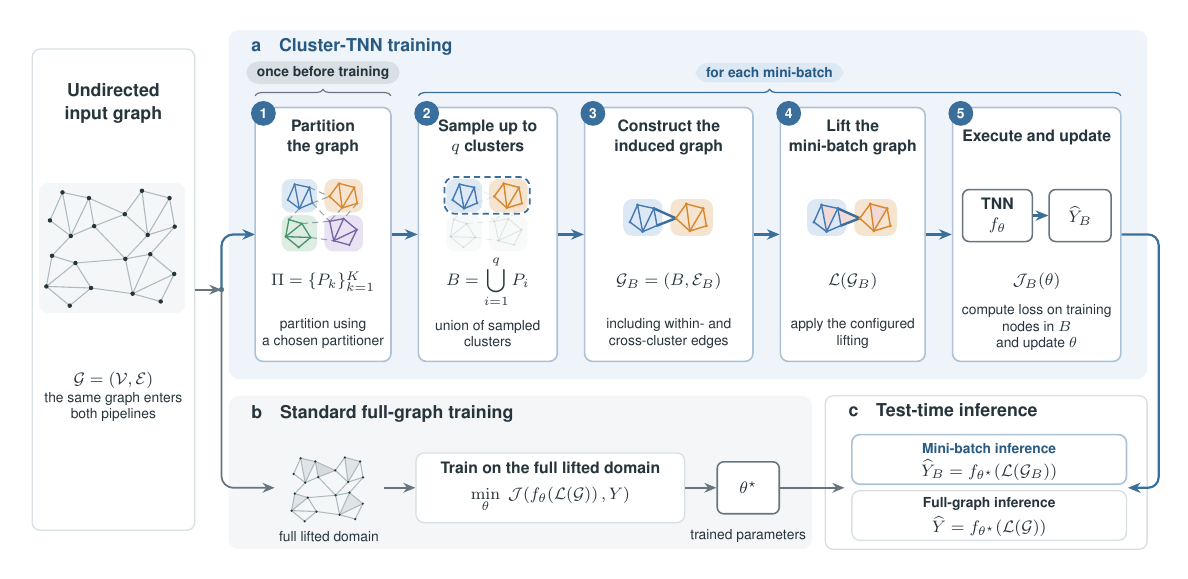}
    \caption{\textbf{Comparison of \our training (ours) and standard full-graph TDL training.}
    Both workflows start from the same featured undirected graph. \textbf{a} \our partitions the graph once per run and at each training step samples up to $q$ clusters, reconstructs their induced subgraph, and performs lifting and model execution only on the resulting mini-batch. \textbf{b} Standard full-graph training instead constructs and processes the complete lifted domain throughout training. \textbf{c} Trained models support both mini-batch and full-graph inference.}
    \label{fig:workflow}
\end{figure*}

Existing mini-batching strategies do not fully address this bottleneck because they typically intervene only after the higher-order domain has been constructed. Sampling or partitioning an already lifted hypergraph, simplicial complex, or cell complex can reduce the memory required by individual forward and backward passes, but it would still assume that the complete set of higher-order elements and their connectivity can be enumerated or materialized beforehand. Consequently, mini-batched model execution may remain infeasible when the preceding lifting operation exceeds the available computational resources. Moreover, batching a lifted representation generally requires domain-specific procedures, limiting reuse across liftings and TNN architectures.

On the other hand, performing mini-batching before lifting could address both limitations: it would restrict lifting to the data required for the current optimization step and perform batching on the graph shared by the different higher-order constructions. However, this change in execution order introduces a structural challenge: if graph partitions are lifted independently, edges crossing partition boundaries are omitted. Consequently, higher-order structures supported by nodes in different partitions, including neighborhood hyperedges, cliques, and cycles, may disappear from the training data. Independently lifting small graph fragments can therefore reduce computational cost at the expense of systematically altering the higher-order domain observed by the model. 

We address this gap by introducing \our, a \textit{batch-before-lift} strategy that avoids full-domain materialization while preserving all original edges among the sampled nodes. As illustrated in Figure~\ref{fig:workflow}, \our partitions the graph once into disjoint node clusters and samples a group of clusters at each optimization step. The proposed workflow then reconstructs the subgraph induced by the combined nodes of these clusters, and retains every original edge among its nodes, also including edges between clusters. Lastly, \our applies the requested lifting only to this reconstructed graph mini-batch and passes the resulting higher-order domain directly to a TNN.

Induced-subgraph reconstruction is central to the approach: unlike lifting each sampled cluster independently, it restores the edges connecting jointly sampled clusters. Under lifting compatibility (Sec.~\ref{subsec:theoretical_guarantees}), a higher-order structure can therefore be recovered whenever all of its supporting nodes occur in the same batch. By recombining clusters across epochs, \our provides repeated opportunities to observe higher-order structures spanning different clusters. 

\textbf{Contributions.} The main features and contributions of our work include:
\begin{itemize}[leftmargin=*,itemsep=2pt,parsep=0pt,topsep=0pt]
    \item \textbf{Batch-before-lift strategy.}
    We introduce \our, a general strategy for scaling TDL beyond the limits imposed by global topology materialization. The same batching mechanism supports different graph-to-higher-order liftings and existing TNN architectures, which we evaluate on Reddit and OGBN Products.

    \item \textbf{Cross-cluster structural recovery of \our.}
    We characterize higher-order structural recovery under repeated cluster recombination (Sec.~\ref{subsec:theoretical_guarantees}) and report empirical measurements of observable hyperedges, cellular 2-cells, and simplicial 2-cells (Figure~\ref{fig:structural_observability} and Appendix~\ref{subsec:exp_cora}).

    \item \textbf{Scalability-computation trade-off via extensive evaluation.}
    We empirically quantify the trade-off introduced by batch-before-lift in terms of full-graph preprocessing feasibility and cost, peak CPU and GPU memory, training time and predictive performance. In particular, across $21$ matched model-dataset comparisons, \our reduces peak GPU memory in every setting by $83.2\%$ on average, while maintaining competitive predictive performance.

    \item \textbf{Reusable open-source integration.}
    We implement \our in \tb{}~\citep{TopoBench}, reusing existing liftings and models across different higher-order domains without requiring architecture-specific modifications. Code is available in our \href{https://anonymous.4open.science/r/cluster-tnn-233C/}{anonymous repository}.
\end{itemize}
\FloatBarrier

\section{Related work}
\label{sec:related_work}


\textbf{Higher-order domains, liftings, and TDL software.} Evaluating TNNs on graph datasets requires constructing the higher-order domain on which the model operates, making the lifting itself part of the model evaluation. Hypergraph benchmarks show that lifting choice affects predictive performance, with no single lifting dominating across the evaluated settings~\citep{montagna2026lift}. TopoX~\citep{hajij2024topox} provides topological data structures and neural operators, while \tb~\citep{TopoBench} integrates datasets, liftings, and models within a common benchmarking framework. This standardization supports comparisons across hypergraph, cellular, and simplicial architectures. However, \citet{TopoBench} report out-of-memory errors during clique and cycle lifting on graphs with fewer than $25\,000$ nodes. Thus, even at this scale, full-domain construction can exhaust available memory before model training begins.

\textbf{Scalability methods for higher-order learning.} Existing approaches reduce model-execution costs. FastHGNN~\citep{fasthgnn} and \citet{telyatnikov2025hypergraph} subsample preconstructed hypergraphs, while Topo-MLP~\citep{ramamurthy2023topomlp} accelerates simplicial inference by removing explicit message passing. TopoMamba~\citep{montagna2025topological} introduces architecture-specific neighbourhood batching using node-to-cell incidence relations. SaNN~\citep{gurugubelli2024sann} precomputes simplicial feature aggregations. HOPSE~\citep{bernardez2025hopse} precomputes positional and structural encodings through Hasse graph decompositions for message-passing-free learning. These methods operate on constructed higher-order domains, leaving a complementary objective: limiting domain construction across lifting families while retaining existing TNNs.

\textbf{Graph mini-batching as the methodological foundation.} GNN methods restrict computation through neighbourhood sampling with GraphSAGE~\citep{reddit_paper}, layer-wise importance sampling with FastGCN~\citep{chen2018fastgcn}, or subgraph sampling with GraphSAINT~\citep{zeng2020graphsaint}. Cluster-GCN~\citep{cluster_gcn} partitions the graph's nodes into disjoint clusters and forms mini-batches from one or more clusters, retaining all edges among the selected nodes, including those connecting different clusters. For graph-supported TDL, preserving this connectivity has an additional significance: missing edges can prevent higher-order structures from being constructed, even when all their supporting nodes are present.


\section{Methods}
\label{sec:methods}


\our moves mini-batching before higher-order domain construction. At each training step, we group up to \(q\) graph clusters, reconstruct their induced subgraph, and apply the selected lifting only to that subgraph. Retaining all original edges among the sampled nodes provides the connectivity needed to construct higher-order structures spanning jointly sampled clusters. 
A TNN for the selected higher-order domain then processes the lifted mini-batch. This execution order makes mini-batching control both domain construction and model execution across higher-order domains, without materializing the full lifted domain.
\subsection{Mini-batch construction and lifting}
\label{subsec:clustered_mini_batch_lifting}
Let $\graph=(\vertices,\edges,F_{\vertices},F_{\edges})$ denote a graph with node set $\vertices$, edge set $\edges$, and corresponding feature maps $F_{\vertices}$ and $F_{\edges}$. Following the multi-cluster batching scheme of Cluster-GCN~\citep{cluster_gcn}, we first partition $\vertices$ into $K$ disjoint clusters $\Pi=\{P_1,\ldots,P_K\}$, with $\vertices=\bigcup_{k=1}^{K}P_k$. We compute this partition using METIS~\citep{metis} or random partitioning and keep it fixed throughout the run (Figure \ref{fig:workflow}a, step~1). At the beginning of each training epoch, \our randomly groups clusters containing at least one training node, referred to as \textit{active} clusters, into mini-batches of at most $q$ clusters
(see step~2). Mini-batch validation and testing similarly use only clusters containing validation and test nodes, respectively. For a group $Z$, the mini-batch node set is $B=\bigcup_{k\in Z}P_k$, from which we extract the induced subgraph $\graph_B=(B,\edges_B)$ with $\edges_B=\{(u,v)\in\edges:u,v\in B\}$. Node features, labels, split masks, and edge features are restricted accordingly. Thus, $\graph_B$ preserves all edges among co-sampled clusters, including inter-cluster edges (step~3).
The selected lifting $\lifting$ maps $\graph_B$ to higher-order domains required by the TNN, constructing higher-order elements and their features from the connectivity of $\graph_B$ (step~4, Appendix~\ref{app:experimental_liftings}). The TNN then processes $\lifting(\graph_B)$ through the same interface used for full-domain execution (step~5), enabling a unified workflow across hypergraph, cellular, and simplicial domains. In contrast, standard full-graph execution processes the entire lifted object $\lifting(\graph)$ (Figure~\ref{fig:workflow}b).

\begin{wrapfigure}[14]{l}{0.4\textwidth}
\vspace{-1.25\baselineskip}
\begin{minipage}{\linewidth}
\footnotesize
\refstepcounter{algorithm}
\label{alg_pseudo}
\addcontentsline{loa}{algorithm}{\protect\numberline{\thealgorithm} \our}
\hrule
\vspace{2pt}
{\raggedright\noindent\textbf{Algorithm \thealgorithm}\quad \our \par}
\vspace{2pt}
\hrule
\vspace{5pt}
\algsetup{indent=1.3em}
\begin{algorithmic}[1]
\REQUIRE Featured graph $\graph$, node labels $Y$, lifting $\lifting$,
         number of clusters $K$, maximum clusters per mini-batch $q$
\ENSURE Trained parameters $\theta$

\STATE Partition $\vertices$ into $K$ clusters

\FOR{each training step}
    \STATE Take at most $q$ active clusters
    \STATE Form mini-batch node set $B$
    \STATE Construct induced subgraph $\graph_B$
    \STATE Lift the mini-batch: $\lifting(\graph_B)$
    \STATE Compute $\mathcal{J}_B(\theta)$
    \STATE Take an optimizer step on $\mathcal{J}_B(\theta)$
\ENDFOR
\end{algorithmic}
\vspace{5pt}
\hrule
\end{minipage}
\end{wrapfigure}

Algorithm~\ref{alg_pseudo} summarizes the \our training protocol. The training objective and inference protocols are detailed in Sec.~\ref{subsec:training_inference}. The workflow applies when the selected graph-to-domain lifting can be evaluated on an induced graph and the downstream TNN can process the variable-size lifted domains through the same interface used for the full-domain execution. For node classification, the model must also produce predictions for the target nodes contained in each mini-batch. Architectures that require operators or state defined on the complete lifted domain require additional adaptation, whereas datasets provided natively as higher-order domains require batching directly in that domain rather than the graph-first construction considered here (see Appendix~\ref{subsec:tdl_scalability_solution_paths}).

\WFclear
\subsection{Training and inference}
\label{subsec:training_inference}
Let $Y$ denote the labels available for the training nodes $\vertices_{\mathrm{train}}\subseteq\vertices$, and let $f_\theta$ denote the model with parameters $\theta$. For each mini-batch, we evaluate the node-classification loss on training nodes contained in $B$, given by
\[
    \mathcal{J}_B(\theta)=\frac{1}{|B\cap\vertices_{\mathrm{train}}|}\!\sum_{v\in B\cap\vertices_{\mathrm{train}}}\!\!\!\!\!\!\ell\left(f_\theta\left(\lifting(\graph_B)\right)_v,Y_v\right),
\]
where $\ell(\cdot,\cdot)$ denotes the cross-entropy loss between the model predictions and the ground truth. Since the loader uses only active train clusters, each mini-batch contains supervised nodes. Each lifted mini-batch produces one optimization step, regardless of the number of supervised nodes. 

At inference, the model can operate on one cluster grouping, average predictions across several independent groupings, or, when memory permits, materialize the full lifted graph. We refer to these modes as \textit{single-pass mini-batch}, \textit{ensemble mini-batch}, and \textit{full-graph} inference, respectively. The primary partitioned results use ensemble mini-batch inference with ten independent cluster recombinations, whereas the inference ablation compares all three modes (Sec.~\ref{subsec:ablations}).

\WFclear
\subsection{Theoretical guarantees}
\label{subsec:theoretical_guarantees}

Each mini-batch is lifted from the subgraph induced by its sampled clusters (Sec.~\ref{subsec:clustered_mini_batch_lifting}). A structure from the full lifted graph cannot appear unless all its supporting nodes are included. This section establishes which higher-order structures \our can recover without full-domain lifting and quantifies how repeated recombination of all partition clusters increases their coverage.

\textbf{A structure can be recovered only if its supporting clusters fit within one mini-batch.}
Let $S^\ast$ denote the finite set of structures produced by lifting the full graph. For a structure $s\in S^\ast$, its \textit{cluster span} $c_s$ counts the clusters containing its supporting nodes. For example, a triangle whose vertices lie in three different clusters requires all three clusters to appear in the same mini-batch. We assume \textit{lifting compatibility}: a full-graph structure appears in a lifted mini-batch exactly when that mini-batch contains all its supporting nodes (Definition~\ref{def:lifting_comp}). Under this assumption, structures with $c_s\leq q$ can be recovered, whereas structures with $c_s>q$ cannot appear in a mini-batch containing $q$ clusters. We call the former structures \textit{$q$-observable} and define $S_q^\ast=\{s\in S^\ast:c_s\leq q\}$.

\begin{wrapfigure}{r}{0.35\columnwidth}
    \vspace{-5pt}
    \centering
    \includegraphics[width=\linewidth]{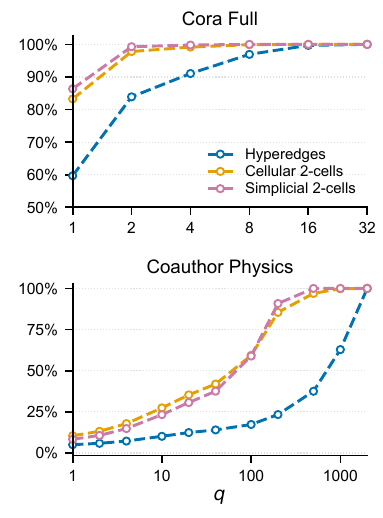}
    \caption{Structural support recovery as a function of $q$. Cora Full uses $K=32$ and Coauthor Physics $K=2\,000$. Curves show means over ten independent reshuffling runs. See Appendix~\ref{subsec:structural_observability_datasets} for more details.}
    \label{fig:structural_observability}
    \vspace{-20pt}
\end{wrapfigure}

\textbf{Cumulative recovery under cluster recombination.} A structure need not appear in every epoch to be observed over the course of training. Independent cluster recombinations provide repeated opportunities for its supporting clusters to occur together. The following result quantifies both its finite-epoch recovery probability and the limiting recovery of the observable structure set. Consider a finite graph with a fixed partition into $K$ clusters. At each epoch, independently draw a uniform random permutation of all $K$ clusters and divide it into groups of $q$ clusters. Assume $q\mid K$ and lifting compatibility. For a full-graph structure $s$ spanning $c_s$ clusters, its per-epoch recovery probability is $\left.p_s=\binom{q-1}{c_s-1}\middle/\binom{K-1}{c_s-1}\right.$ when $c_s\leq q$, and $p_s=0$ otherwise. After $T$ epochs, the probability that $s$ has been observed at least once is $1-(1-p_s)^T$.
Consequently, the probability that every structure in $S_q^\ast$ has been observed approaches one as $T\to\infty$. If $q\geq\max_{s\in S^\ast}c_s$, this gives cumulative recovery of the global structure set $S^\ast$.

When split-specific batching excludes inactive clusters, structures involving nodes from those clusters cannot be recovered (Appendix~\ref{sec:struct_recon}). The per-epoch probability and cumulative-recovery proofs are given in Appendix~\ref{subsec:per_epoch_recovery_prob}-\ref{subsec:cum_prob_conv}. Recovery is cumulative: different structures may be observed across mini-batches and epochs, without materializing the complete lifted domain. For models that exchange messages between incident cells, recovering the relevant cells and incidences makes the direct messages available (Appendix~\ref{subsec:direct_incidence_message_availability}-\ref{subsec:layerwise_incidence_aggregation}).

\textbf{Larger mini-batches expand what can be recovered, while additional epochs provide repeated opportunities.}
A structure spanning more than $q$ clusters cannot be recovered regardless of the number of epochs. Increasing $q$ can make more structures observable and increase their recovery probabilities. Figure~\ref{fig:structural_observability} reports the fraction of reference structures whose supporting nodes are jointly sampled within $200$ epochs. For cellular references, this does not imply identical local cycle-basis recovery. At $q=4$, this fraction exceeds $99\%$ for both cellular and simplicial 2-cells on Cora Full. On Coauthor Physics at $q=100$, the corresponding fractions are only $59.3\%$ for cellular and $58.9\%$ for simplicial supports, although each individual reference 2-cell could fit within a mini-batch of this size. Fitting within a mini-batch therefore does not ensure that the required clusters are sampled together. Complete neighborhood hyperedges reach only $62.8\%$ recovery even at $q=1\,000$. Their slower recovery under these partitions is consistent with larger cluster spans: triangles and retained cycles involve three and at most nine nodes, respectively, whereas a neighborhood hyperedge includes its center and all its neighbors. These larger supports can span more clusters, reducing their joint-sampling probability. See Appendix~\ref{subsec:repeated_structural_exposure} and~\ref{subsec:repeated_exposure_coverage} for more details.

\textbf{Repeated recombination motivates ensemble inference.}
A single cluster grouping can omit higher-order structures that involve a target node, even when those structures are observable. Independent recombinations provide further opportunities to include their supporting nodes. This motivates averaging predictions from the same trained model over several independently sampled cluster groupings, without materializing the full lifted domain (Sec.~\ref{subsec:training_inference}). Ensemble inference achieves a higher mean test metric than single-pass inference in all $21$ reported model-dataset comparisons (Table~\ref{tab:partitioning_ablations}). These results support using additional mini-batch inference passes to improve predictive performance while retaining the memory benefits of processing one lifted mini-batch at a time.

\section{Evaluation}
\label{sec:evaluation}
We evaluate whether batch-before-lift enables larger-scale TNN training and examine its computational costs, predictive performance, and structural recovery through three questions:
\makeatletter
\begin{itemize}[leftmargin=0pt,label={},itemsep=2pt,parsep=0pt,topsep=0pt]

\item \textbf{Q1: Scalability and computational trade-offs.}
Can \our enable training of topological models when full-domain construction or
model execution exceeds available resources, and how does it
affect preprocessing time, peak CPU/GPU memory, training
time, and end-to-end runtime?
\def\@currentlabel{Q1}\label{rq:scalability}

\item \textbf{Q2: Predictive performance.}
How does \our compare with full-graph execution in predictive
performance on the matched benchmarks, and how do higher-order
models trained with \our compare with the GCN baseline on
Reddit and OGBN Products?

\def\@currentlabel{Q2}\label{rq:predictive}

\item \textbf{Q3: Structural recovery and design choices.}
How do mini-batch size and repeated cluster recombination
influence structural recovery, and how do partitioning,
cross-cluster edge retention, and inference strategy affect
predictive performance?

\def\@currentlabel{Q3}\label{rq:structure}

\end{itemize}
\makeatother

\subsection{Experimental setup}
\label{subsec:exp_setup}

\textbf{Datasets.} We evaluate \our on six node-classification benchmarks spanning diverse application domains, and covering three experimental regimes: 
\begin{itemize}[leftmargin=*,itemsep=2pt,parsep=0pt,topsep=0pt]

\item \emph{Controlled full-domain} regime includes Questions~\citep{platonov2023critical} ($48\,921$ nodes), Amazon Ratings~\citep{platonov2023critical} ($24\,492$), and Cora Full~\citep{bojchevski2018deep} ($19\,793$), characterized by low mean degree ($6.28$-$7.60$) and limited lifting expansion ($\leq2.60$ cellular and $\leq4.52$ simplicial 2-cells/node).

\item \emph{Intermediate-density} regime is represented by Coauthor Physics~\citep{pitfalls} ($34{,}493$ nodes, mean degree $14.38$, $13.58$ simplicial 2-cells/node).

\item \emph{Large-scale} regime comprises Reddit~\citep{reddit_paper} ($232\,965$ nodes) and OGBN Products~\citep{hu2020} ($2\,449\,029$ nodes), both containing approximately $60$M edges. These datasets expose complementary scaling regimes: Reddit combines high connectivity (mean degree $491.99$) with substantial structural expansion ($174.20$ cellular and $35\,886.67$ simplicial 2-cells/node), whereas OGBN Products primarily stresses node capacity at moderate expansion (mean degree $50.52$, $13.98$ cellular and $253.60$ simplicial 2-cells/node).

\end{itemize}


\textbf{Models.} We compare a graph-based GCN model~\citep{kipf2017gcn} with representative higher-order models from three domains: hypergraph models EDHNN~\citep{wang2023equivariant_hypergraph_diffusion} and UniGNN~\citep{huang2021unignn}, cellular models CWN~\citep{bodnar2021cwn} and TopoTune~\citep{papillon2025topotune}, and simplicial models SCN~\citep{ebli2020simplicial_neural_networks} and SCCNN~\citep{yang2023convolutional}. Following TopoBench's projected-sum feature lifting~\citep{TopoBench}, we retain input node features. We sum member-node features for hyperedges and successively sum incident lower-rank features for cellular and simplicial elements (Appendix~\ref{app:experimental_liftings}).

\textbf{Setup.} For all \our experiments, we use the same mini-batch lifting workflow across models, removing differences from model-specific batching implementations. We tune the full-graph and partitioned training settings separately and use the same generated random node splits across models and settings\footnote{For OGBN Products, we use generated $70\%/15\%/15\%$ random splits rather than the official sales-ranked split of approximately $8\%/2\%/90\%$~\citep{hu2020}.} 
and report test performance as mean $\pm$ standard deviation. The controlled full-domain regime enables direct comparison between full-graph and \our training. For the intermediate-density and large-scale regimes, full-domain execution of at least one higher-order TNN exceeds the available resources. Therefore, we evaluate those regimes using \our training. Appendix~\ref{sec:tdl_scalability_appendix} provides a detailed analysis of the resulting resource requirements, including contributions from structural expansion, feature storage, and model-required operators.

\subsection{Computational efficiency}
\label{subsec:comp_efficiency}

\begin{figure}[!t]
    \centering
    \includegraphics[width=\textwidth]{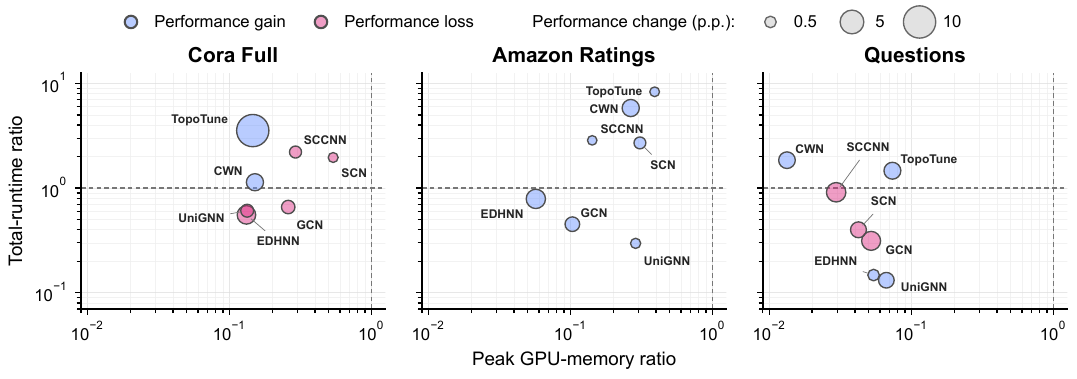}
    \caption{Cluster-TNN reduces peak GPU memory across all model-dataset pairs. Each circle represents one model, positioned by the mean ratios of \our to full-graph peak GPU memory and total runtime across five seeds. Moving left indicates memory savings, whereas moving down indicates shorter runtime. The lower-left region marks improvements in both measures. Marker size represents the absolute test-performance difference in percentage points, and color distinguishes performance gains from losses. Dashed lines mark parity with full-graph execution.}
    \label{fig:cwn_scalability_memory_runtime}
\end{figure}

\begin{wrapfigure}[24]{r}{0.325\columnwidth}
    \vspace{-2pt}
    \centering
    \includegraphics[width=\linewidth]{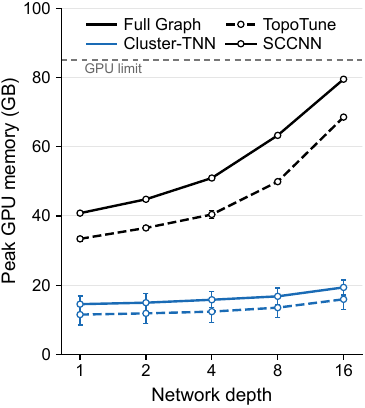}
    \caption{GPU memory scaling on Cora Full with 512 hidden channels for $K=64$ and $q=8$. Peak GPU memory is reported as mean $\pm$ standard deviation across five seeds. Network depth counts TopoTune blocks and SCCNN layers.}
    \label{fig:memory_with_depth}
    \vspace{-20pt}
\end{wrapfigure}

To address \ref{rq:scalability}, we primarily evaluate the memory savings enabled by \our execution, while also reporting its effects on preprocessing cost and total runtime. In the controlled full-domain regime \our{} reduces run-wide peak GPU memory in all $21$ comparisons with full-graph execution, with a mean reduction of $83.2\%$ and a median reduction of $86.7\%$. Comparison-level reductions range from $46.4\%$ to $98.7\%$, based on the reported means across five seeds (Figure~\ref{fig:cwn_scalability_memory_runtime} and Table~\ref{tab:tdl_end_to_end_memory}). Furthermore, Figure~\ref{fig:appendix_memory_runtime_other_models} separates CPU preprocessing peaks from GPU training and evaluation peaks. For CWN on Cora Full, for example, mean peak GPU memory decreases from $24.1$ to $3.7$~GB (Table~\ref{tab:tdl_end_to_end_memory}). A separate experiment with model hyperparameters fixed shows that processing a smaller fraction of clusters per mini-batch, $q/K$, reduces peak GPU memory during training and validation (Figure~\ref{fig:qk-gpu-memory}, Appendix~\ref{sec:additional_results}). Batching can therefore control memory requirements without reducing model size.

Memory savings are particularly important for computationally expensive topological configurations, where full-domain preprocessing can itself exceed practical memory budgets. On Reddit, storing full-graph simplicial features alone is estimated to require $20.3$~TB (Table~\ref{tab:tdl_static_feature_storage}). Full-graph preprocessing for TopoTune takes $6\,965.2$~s and reaches $480.1$~GB peak CPU memory (Table~\ref{tab:tdl_preprocessing_cost}). With $K=7\,000$ partitions selected by hyperparameter optimization (Table~\ref{tab:optuna_hparams}), \our preprocessing takes $500.4$~s (Table~\ref{tab:tdl_end_to_end_runtime}), and run-wide peak CPU memory is $11.81$~GB (Table~\ref{tab:tdl_partitioned_memory}).

At a fixed width of $512$ hidden channels, peak GPU memory grows substantially more slowly with network depth under \our for both TopoTune and SCCNN (Figure~\ref{fig:memory_with_depth}). At depth $16$, their mean peaks remain below $20$~GB, compared with approximately $69$--$80$~GB for full-graph execution, showing that the memory advantage extends to deeper architectures.

Across the $21$ matched comparisons, preprocessing time decreases in $16$, increases in four, and remains unchanged in one. Among $16$ reductions, the mean is $53.9\%$ and median $53.1\%$, including the graph-partitioning cost (Table~\ref{tab:tdl_end_to_end_runtime}).

For EDHNN on Questions, mean training time decreases from $203.1$ to $19.5$~s and total runtime from $208.8$ to $30.4$~s, an $85.4\%$ reduction, while mean test AUROC increases from $75.73$ to $76.23$ (Tables~\ref{tab:accuracy_main} and~\ref{tab:tdl_end_to_end_runtime}).
Training time, including periodic validation, decreases in $12$ of $21$ matched comparisons. Total runtime decreases in $11$, with a median reduction of $55.6\%$, and increases in $10$, with a median increase of $136.9\%$ (Table~\ref{tab:tdl_end_to_end_runtime}). For CWN on Cora Full, training time falls by $14.7\%$, but final evaluation rises from $14.6$ to $272.4$~s, increasing total runtime by $12.4\%$. Final evaluation includes ensemble inference over ten independent cluster recombinations, together with batch construction and lifting. Its additional cost can therefore outweigh training-time savings. Appendix~\ref{subsec:tdl_scalability_runtime} defines each runtime component and how percentage changes are aggregated across matched configurations.

\subsection{Feasibility and predictive performance}
\label{subsec:feasibility_predictive_perf}
Table~\ref{tab:accuracy_main} addresses the feasibility aspect of \ref{rq:scalability} and the predictive-performance comparison in \ref{rq:predictive}. It compares \our and full-graph training to assess whether \our preserves predictive performance while enabling larger-scale training where full-graph execution is not possible.

Across $21$ model-dataset pairs for which results are available under both training settings, \our training achieves a higher reported mean of the corresponding test metric in $13$ comparisons, while full-graph training performs better in the remaining eight. On Cora Full, partitioned training yields a slightly better result for CWN and a notably better result for TopoTune compared with full-graph training. On Amazon Ratings, partitioned training achieves better results for all seven evaluated models, with particularly pronounced differences for EDHNN and CWN. Similarly, on Questions, partitioned training yields better performance for four of the evaluated models. Across all three datasets, a higher-order TNN achieves the highest mean test metric under both full-graph and partitioned training. These results highlight two observations. First, they demonstrate the potential benefits of higher-order models over the GCN baseline on these datasets. Second, \our achieves higher mean test metrics than full-graph training across multiple model-dataset pairs.

\begin{table}[!t]
\centering
\caption{Test performance under full-graph and \our settings. We report AUROC for Questions and accuracy for all other datasets. \our uses ten-recombination ensemble mini-batch inference, whereas full-graph runs use full-graph inference. Values are mean $\pm$ sample standard deviation across five seeds. Boldface indicates best-performing model for each setting.}
\label{tab:accuracy_main}
\small
\setlength{\tabcolsep}{1.5pt}
\renewcommand{\arraystretch}{1.1}
\setlength{\extrarowheight}{1.1pt}
\resizebox{\textwidth}{!}{%
\begin{tabular}{lccccccccc}
\toprule
& \multicolumn{2}{c}{\makecell{\textbf{Cora} \\ \textbf{Full}}}
& \multicolumn{2}{c}{\makecell{\textbf{Amazon} \\ \textbf{Ratings}}}
& \multicolumn{2}{c}{\textbf{Questions}}
& \makecell{\textbf{Coauthor} \\ \textbf{Physics}}
& \textbf{Reddit}
& \makecell{\textbf{OGBN} \\ \textbf{Products}}\\
\cmidrule(lr){2-3}
\cmidrule(lr){4-5}
\cmidrule(lr){6-7}
\cmidrule(lr){8-8}
\cmidrule(lr){9-9}
\cmidrule(lr){10-10}
\textbf{Model}
& \textbf{Full-graph}
& \cellcolor{gray!14}\textbf{\our}
& \textbf{Full-graph}
& \cellcolor{gray!14}\textbf{\our}
& \textbf{Full-graph}
& \cellcolor{gray!14}\textbf{\our}
& \cellcolor{gray!14}\textbf{\our}
& \cellcolor{gray!14}\textbf{\our}
& \cellcolor{gray!14}\textbf{\our} \\
\midrule

GCN
& $70.18 \pm 0.42$
& \cellcolor{gray!14}$69.19 \pm 0.66$
& $48.29 \pm 1.31$
& \cellcolor{gray!14}$49.62 \pm 0.96$
& $76.69 \pm 1.18$
& \cellcolor{gray!14}$73.83 \pm 2.01$
& \cellcolor{gray!14}$\mathbf{96.68 \pm 0.08}$
& \cellcolor{gray!14}$\mathbf{92.62 \pm 0.10}$
& \cellcolor{gray!14}$85.81 \pm 0.04$ \\

EDHNN
& $69.74 \pm 0.89$
& \cellcolor{gray!14}$67.06 \pm 1.19$
& $48.41 \pm 1.18$
& \cellcolor{gray!14}$\mathbf{51.41 \pm 0.72}$
& $75.73 \pm 0.92$
& \cellcolor{gray!14}$\mathbf{76.23 \pm 1.05}$
& \cellcolor{gray!14}$96.31 \pm 0.32$
& \cellcolor{gray!14}$92.17 \pm 0.73$
& \cellcolor{gray!14}$86.51 \pm 0.25$ \\

UniGNN
& $68.44 \pm 0.79$
& \cellcolor{gray!14}$67.61 \pm 0.56$
& $47.05 \pm 1.76$
& \cellcolor{gray!14}$47.24 \pm 0.70$
& $69.44 \pm 1.68$
& \cellcolor{gray!14}$71.09 \pm 0.98$
& \cellcolor{gray!14}$94.97 \pm 0.37$
& \cellcolor{gray!14}$91.28 \pm 0.16$
& \cellcolor{gray!14}$85.87 \pm 0.07$ \\

CWN
& $60.50 \pm 0.60$
& \cellcolor{gray!14}$62.66 \pm 1.11$
& $44.07 \pm 0.64$
& \cellcolor{gray!14}$46.25 \pm 0.67$
& $68.21 \pm 3.48$
& \cellcolor{gray!14}$70.20 \pm 0.88$
& \cellcolor{gray!14}$95.80 \pm 0.24$
& \cellcolor{gray!14}$91.29 \pm 0.18$
& \cellcolor{gray!14}$83.93 \pm 0.06$ \\

TopoTune
& $46.40 \pm 0.47$
& \cellcolor{gray!14}$56.45 \pm 1.42$
& $43.18 \pm 1.03$
& \cellcolor{gray!14}$43.27 \pm 0.68$
& $71.53 \pm 1.73$
& \cellcolor{gray!14}$73.61 \pm 1.88$
& \cellcolor{gray!14}$95.22 \pm 0.20$
& \cellcolor{gray!14}$89.55 \pm 0.09$
& \cellcolor{gray!14}$83.79 \pm 0.55$ \\

SCN
& $70.45 \pm 0.35$
& \cellcolor{gray!14}$\mathbf{70.32 \pm 0.46}$
& $50.08 \pm 0.36$
& \cellcolor{gray!14}$50.65 \pm 0.64$
& $74.77 \pm 1.34$
& \cellcolor{gray!14}$73.04 \pm 1.84$
& \cellcolor{gray!14}$96.22 \pm 0.22$
& \cellcolor{gray!14}$91.80 \pm 0.13$
& \cellcolor{gray!14}$87.03 \pm 0.05$ \\

SCCNN
& $\mathbf{70.81 \pm 0.58}$
& \cellcolor{gray!14}$70.20 \pm 0.29$
& $\mathbf{51.20 \pm 1.10}$
& \cellcolor{gray!14}$51.24 \pm 0.79$
& $\mathbf{77.14 \pm 0.99}$
& \cellcolor{gray!14}$74.10 \pm 1.85$
& \cellcolor{gray!14}$96.00 \pm 0.25$
& \cellcolor{gray!14}$92.45 \pm 0.25$
& \cellcolor{gray!14}$\mathbf{88.63 \pm 0.05}$ \\

\bottomrule
\end{tabular}}
\end{table}

On Coauthor Physics, representing the intermediate-density regime, the six higher-order models achieve mean accuracies of $94.97$--$96.31\%$, compared with $96.68\%$ for GCN.
Experiments in the large-scale regime further demonstrate the practical advantages of \our. The workflow enables, to our knowledge, the first memory-efficient training of multiple higher-order TNN models on datasets at this scale. On Reddit, we successfully train higher-order TNNs across different domains while achieving accuracies comparable to the GCN baseline. 
On OGBN Products, SCCNN reaches $88.63\%$ accuracy, compared with $85.81\%$ for GCN. Both hypergraph models and SCN also achieve higher means than GCN, whereas the cellular models do not. These differences motivate comparing higher-order representations across datasets, a task enabled at scale by the common \our workflow. Tables~\ref{tab:accuracy_main}, \ref{tab:tdl_end_to_end_runtime}, and~\ref{tab:tdl_partitioned_memory} report predictive results and resource requirements.


\subsection{Ablation studies}
\label{subsec:ablations}
For \ref{rq:structure}, the structural analysis distinguishes structures that can fit in a mini-batch from those whose supporting nodes are jointly sampled within a finite epoch budget (Figure~\ref{fig:structural_observability} and Appendix~\ref{subsec:structural_observability_datasets}). Table~\ref{tab:partitioning_ablations} complements this analysis with predictive ablations of partitioning, cross-cluster edge retention, and inference strategy (Sec.~\ref{subsec:training_inference}) on the three controlled benchmarks.

\textbf{Partitioning method.} METIS yields the higher mean test metric in $15$ of $21$ comparisons, including six of seven models on both Cora Full and Amazon Ratings, but only three on Questions (Table~\ref{tab:partitioning_ablations}). Its predictive advantage therefore depends on the model and dataset. In GNN experiments, random node sampling can outperform Cluster-GCN and full-graph training~\citep{wang2026}.

\textbf{Cross-cluster edges.} Preserving cross-cluster edges yields the higher mean test metric in $19$ of $21$ comparisons, with the largest difference of $10.01$ percentage points for TopoTune on Cora Full. These results support retaining the full induced connectivity among co-sampled clusters.

\textbf{Inference strategy.} Ensemble mini-batch inference yields a higher mean test metric than a single pass in all $21$ comparisons, with gains from $0.04$ to $4.65$ percentage points. Full-graph inference improves on a single pass in $20$ comparisons. These results support repeated inference across sampled contexts without full-domain materialization. These gains require additional inference passes, whose computational cost is included in the runtime comparison for \ref{rq:scalability}.
\suppressfloats[t]
\begin{table}[!t]
\centering
\caption{Partitioning and inference ablations (shaded columns). \emph{Partitioner} compares METIS with random partitioning, and \emph{Edges} compares preserving with removing cross-cluster edges. \emph{Ensemble} and \emph{Full-graph} compare the respective inference protocols with single-pass mini-batch inference. Entries report mean paired differences in test performance (percentage points) over five seeds.}
\label{tab:partitioning_ablations}
\footnotesize
\setlength{\tabcolsep}{2pt}
\renewcommand{\arraystretch}{1.12}
\resizebox{\textwidth}{!}{%
\begin{tabular}{@{}lcccc@{\hspace{6pt}}cccc@{\hspace{6pt}}cccc@{}}
\toprule
& \multicolumn{4}{c}{\textbf{Cora Full}}
& \multicolumn{4}{c}{\textbf{Amazon Ratings}}
& \multicolumn{4}{c}{\textbf{Questions}} \\
\cmidrule(lr){2-5}\cmidrule(lr){6-9}\cmidrule(lr){10-13}
\textbf{Model}
& {\scriptsize\textbf{Partitioner}} & {\scriptsize\textbf{Edges}} & \cellcolor{gray!7}{\scriptsize\textbf{Ensemble}} & \cellcolor{gray!7}{\scriptsize\textbf{Full-graph}}
& {\scriptsize\textbf{Partitioner}} & {\scriptsize\textbf{Edges}} & \cellcolor{gray!7}{\scriptsize\textbf{Ensemble}} & \cellcolor{gray!7}{\scriptsize\textbf{Full-graph}}
& {\scriptsize\textbf{Partitioner}} & {\scriptsize\textbf{Edges}} & \cellcolor{gray!7}{\scriptsize\textbf{Ensemble}} & \cellcolor{gray!7}{\scriptsize\textbf{Full-graph}} \\
\midrule
GCN & $+0.98$ & $+0.08$ & \cellcolor{gray!7}$+2.51$ & \cellcolor{gray!7}$+3.44$ & $+2.24$ & $+0.25$ & \cellcolor{gray!7}$+0.05$ & \cellcolor{gray!7}$+0.08$ & $+0.52$ & $+0.99$ & \cellcolor{gray!7}$+0.65$ & \cellcolor{gray!7}$+2.71$ \\
EDHNN & $+6.28$ & $+1.72$ & \cellcolor{gray!7}$+1.55$ & \cellcolor{gray!7}$+2.59$ & $+2.88$ & $-0.44$ & \cellcolor{gray!7}$+0.45$ & \cellcolor{gray!7}$+0.33$ & $+1.25$ & $+1.04$ & \cellcolor{gray!7}$+0.42$ & \cellcolor{gray!7}$+0.22$ \\
UniGNN & $+6.45$ & $+0.63$ & \cellcolor{gray!7}$+0.04$ & \cellcolor{gray!7}$+0.30$ & $+2.47$ & $-0.43$ & \cellcolor{gray!7}$+0.33$ & \cellcolor{gray!7}$+0.34$ & $-1.01$ & $+1.10$ & \cellcolor{gray!7}$+1.35$ & \cellcolor{gray!7}$+0.99$ \\
CWN & $+1.91$ & $+0.66$ & \cellcolor{gray!7}$+2.57$ & \cellcolor{gray!7}$+5.37$ & $+0.18$ & $+0.23$ & \cellcolor{gray!7}$+0.10$ & \cellcolor{gray!7}$+0.08$ & $-0.09$ & $+0.56$ & \cellcolor{gray!7}$+0.14$ & \cellcolor{gray!7}$+0.15$ \\
TopoTune & $-3.86$ & $+10.01$ & \cellcolor{gray!7}$+4.65$ & \cellcolor{gray!7}$+4.75$ & $-2.40$ & $+0.41$ & \cellcolor{gray!7}$+1.56$ & \cellcolor{gray!7}$+0.48$ & $+1.15$ & $+1.49$ & \cellcolor{gray!7}$+0.30$ & \cellcolor{gray!7}$-0.58$ \\
SCN & $+1.78$ & $+0.42$ & \cellcolor{gray!7}$+1.27$ & \cellcolor{gray!7}$+1.78$ & $+2.51$ & $+0.27$ & \cellcolor{gray!7}$+0.71$ & \cellcolor{gray!7}$+0.76$ & $-0.61$ & $+1.27$ & \cellcolor{gray!7}$+0.34$ & \cellcolor{gray!7}$+1.54$ \\
SCCNN & $+5.32$ & $+0.84$ & \cellcolor{gray!7}$+1.40$ & \cellcolor{gray!7}$+2.22$ & $+2.99$ & $+0.20$ & \cellcolor{gray!7}$+0.31$ & \cellcolor{gray!7}$+0.51$ & $-0.18$ & $+1.07$ & \cellcolor{gray!7}$+1.37$ & \cellcolor{gray!7}$+2.16$ \\
\bottomrule
\end{tabular}%
}
\end{table}

\section{Conclusions and discussion}
\label{sec:conclusion}

We introduced \our, a batch-before-lift workflow for graph-supported TDL that constructs higher-order domains and executes TNNs within mini-batches. This work aims to make higher-order learning scalable across domains while reusing existing liftings and architectures. Across $21$ matched full-graph comparisons, \our reduces peak GPU memory by $83.2\%$ on average while retaining competitive predictive performance. It also enables training of six higher-order models on Reddit and OGBN Products, extending TDL beyond the limits of global domain construction. Our structural analysis characterizes how mini-batch size constrains recoverable structures and how cluster recombination provides further opportunities to encounter them. Ensemble inference also achieves a higher mean test metric than single-pass inference in all $21$ controlled comparisons.

\textbf{Limitations.} Structural recovery does not establish equivalence to full-graph predictions or gradients. Its role in the ensemble gains remains unresolved. Repeated lifting and ensemble inference can increase runtime despite lower peak memory. We evaluate selected liftings for graph-supported node classification, without establishing which representation best suits each dataset and model.

\textbf{Future outlook.} Extending to native higher-order data requires batching their given cells and incidence relations. Selective caching could reduce repeated construction, but its storage and recombination trade-offs require evaluation. The changing model rankings across datasets motivate studying lifting and batching jointly, assessing whether broader structural context improves prediction enough to justify its construction and inference costs under fixed memory and end-to-end runtime budgets.

\newpage
\section*{Author contributions}
\label{sec:author_contributions}
D.L. and L.B. led the development of the project, including code implementation and the design, execution, and analysis of the final experiments. G.B. and L.T. conceived the initial research idea and coordinated the project. The idea was further developed by D.L., L.B., G.B., and L.T. D.L. and L.B. derived the main theoretical results, with contributions from L.T. L.T. also contributed substantially to the analysis of the final experiments. D.L., L.B., and L.T. jointly led the writing and revision of the manuscript. G.B. made substantial contributions to the experimental work  and to the writing and revision of the manuscript. N.M. and O.F. provided infrastructure and project support and contributed to manuscript editing.

\section*{Acknowledgments}
This work grew out of the research activities of the Topological Intelligence team. We thank members of its community for helpful discussions and feedback.

L.B. acknowledges financial support from the European Regional Development Fund for the project ‘Materials for clean energy, advanced sensors and quantum technologies’ (Grant No. PK.1.1.10.0002). G. B. and N. M. acknowledge partial support from NSF grant 2602079, the Noyce foundation, and Arlequin AI.

\bibliography{references}
\bibliographystyle{iclr2027_conference}

\clearpage

\appendix

\section*{Appendix contents}
\addcontentsline{toc}{section}{Appendix contents}
\vspace{1em}

\begingroup
\makeatletter
\setlength{\parindent}{0pt}
\setlength{\parskip}{0pt}
\footnotesize
\hypersetup{linkcolor=blue}

\newcommand{\scalappcontentsection}[2]{%
  \@dottedtocline{1}{0em}{2.8em}{\hyperref[#1]{\textbf{\ref*{#1}\hspace{0.6em}#2}}}{\pageref{#1}}%
}
\newcommand{\scalappcontentunnumbered}[2]{%
  \@dottedtocline{1}{0em}{2.8em}{\hyperref[#1]{\textbf{#2}}}{\pageref{#1}}%
}
\newcommand{\scalappcontentsubsection}[2]{%
  \@dottedtocline{2}{2.8em}{3.4em}{\hyperref[#1]{\ref*{#1}\hspace{0.6em}#2}}{\pageref{#1}}%
}
\newcommand{\scalappcontentsubsubsection}[2]{%
  \@dottedtocline{3}{6.2em}{4.2em}{\hyperref[#1]{\ref*{#1}\hspace{0.6em}#2}}{\pageref{#1}}%
}

\scalappcontentunnumbered{app:reproducility_statement}{Reproducibility statement}
\vspace{0.2\baselineskip}

\scalappcontentunnumbered{app:ai_use_statement}{AI use statement}
\vspace{0.2\baselineskip}

\scalappcontentunnumbered{app:ethics_statement}{Ethics statement}
\vspace{0.2\baselineskip}

\scalappcontentunnumbered{app:notation_guide}{Notation guide}
\vspace{0.2\baselineskip}

\scalappcontentsection{sec:domains_liftings}{Featured topological domains and lifting maps}
\scalappcontentsubsection{app:topological-domains}{Featured topological domains}
\scalappcontentsubsection{app:liftings}{Lifting maps}
\scalappcontentsubsubsection{app:experimental_liftings}{Lifting maps used in the experiments}
\vspace{0.2\baselineskip}

\scalappcontentsection{sec:tdl_scalability_appendix}{Practical scalability of graph-supported TDL: bottlenecks and solution paths}
\scalappcontentsubsection{subsec:tdl_scalability_pipeline}{Pipeline view of graph-supported TDL scalability}
\scalappcontentsubsection{subsec:tdl_scalability_representation}{Representation growth: from structure to storage}
\scalappcontentsubsubsection{subsec:tdl_scalability_liftings}{Lifting-specific mechanisms of structural expansion}
\scalappcontentsubsubsection{subsec:tdl_scalability_benchmark_coverage}{Benchmark scale regimes and coverage}
\scalappcontentsubsubsection{subsec:tdl_scalability_feature_storage}{Static feature-matrix storage}
\scalappcontentsubsection{subsec:tdl_scalability_preprocessing}{Full-domain CPU preprocessing: construction time and memory}
\scalappcontentsubsection{subsec:tdl_scalability_execution}{Phase-specific memory during preprocessing and TNN execution}
\scalappcontentsubsection{subsec:tdl_scalability_end_to_end}{End-to-end cost}
\scalappcontentsubsubsection{subsec:tdl_scalability_runtime}{End-to-end runtime}
\scalappcontentsubsubsection{subsec:tdl_scalability_run_wide_memory}{Run-wide peak memory}
\scalappcontentsubsection{subsec:tdl_scalability_statistical}{Structural exposure and statistical efficiency}
\scalappcontentsubsection{subsec:tdl_scalability_solution_paths}{Solution paths and practical guidance}
\vspace{0.2\baselineskip}

\scalappcontentsection{sec:struct_recon}{Structural reconstruction}
\scalappcontentsubsection{subsec:setup_notation}{Setup and notation}
\scalappcontentsubsection{subsec:per_epoch_recovery_prob}{Per-epoch recovery probability}
\scalappcontentsubsection{subsec:repeated_structural_exposure}{Repeated structural exposure}
\scalappcontentsubsection{subsec:cum_prob_conv}{Cumulative recovery probability}
\scalappcontentsubsection{subsec:struct_cov_conv}{Structural coverage convergence}
\scalappcontentsubsection{subsec:repeated_exposure_coverage}{Repeated-exposure structural coverage}
\scalappcontentsubsection{subsec:mean_structural_exposure_rate}{Mean structural exposure rate}
\scalappcontentsubsection{subsec:repeated_exposure_epoch}{Exposure epoch distribution}
\scalappcontentsubsection{subsec:direct_incidence_message_availability}{Direct-incidence message availability}
\scalappcontentsubsection{subsec:repeated_direct_message_exposure}{Repeated direct-incidence message exposure}
\scalappcontentsubsection{subsec:layerwise_incidence_aggregation}{Layerwise incidence aggregation}
\scalappcontentsubsection{subsec:entropy_cov_conv}{Cumulative recovery entropy}
\scalappcontentsubsection{subsec:exp_cora}{Empirical analysis of structural recovery}
\scalappcontentsubsubsection{subsec:structural_observability_datasets}{Structural recovery across datasets}
\scalappcontentsubsubsection{subsec:recovery_epochs_cora_full}{Recovery across epochs on Cora Full}
\vspace{0.2\baselineskip}

\scalappcontentsection{sec:additional_results}{Additional results}
\vspace{0.2\baselineskip}

\scalappcontentsection{sec:hyper_opt}{Experimental protocol and hyperparameter optimization}

\makeatother
\endgroup
\newpage

\section*{Reproducibility statement}
\label{app:reproducility_statement}
Sections~\ref{subsec:clustered_mini_batch_lifting} and~\ref{subsec:training_inference} describe mini-batch construction, training, and inference. Appendix~\ref{app:experimental_liftings} specifies the evaluated liftings, and Appendix~\ref{sec:hyper_opt} reports the experimental protocol, hyperparameter search spaces, and selected configurations. Appendix~\ref{subsec:structural_observability_datasets} describes the structural-recovery experiments. Experiments were conducted on a Linux machine equipped with dual AMD EPYC 7543 processors, 64 physical CPU cores, 1 TB of system memory, and 8 NVIDIA A100 GPUs, each with 80 GiB of GPU memory. To facilitate reproducibility, we make our code implementation available at \url{https://anonymous.4open.science/r/cluster-tnn-233C/}.

\section*{AI use statement}
\label{app:ai_use_statement}
In this work, AI assistance was used as a supplementary tool during the research and writing process. Specifically, AI tools were used to polish and improve the clarity, grammar, and organization of the manuscript, identify and locate potentially relevant related work, assist in drafting and restructuring selected sections, and improve the presentation and exposition of some mathematical proofs. All AI-assisted content was reviewed, revised, and verified by the authors, who remain responsible for the accuracy, originality, interpretation, and final presentation of the work. AI tools were not used as a substitute for the authors' scientific judgment, experimental analysis, or verification of the reported results.

\section*{Ethics statement}
\label{app:ethics_statement}
The work presented in this paper aims to advance the field of Topological Deep Learning. While it may have broader societal implications, we do not identify any that warrant specific discussion here.

\section*{Notation guide}
\label{app:notation_guide}

The table collects the symbols used across multiple appendix sections. Symbols introduced only inside individual proofs remain local to those proofs.

\begin{center}
\small
\renewcommand{\arraystretch}{1.16}
\begin{tabular}{@{}>{\raggedright\arraybackslash}p{0.24\textwidth}>{\raggedright\arraybackslash}p{0.70\textwidth}@{}}
\toprule
\textbf{Symbol} & \textbf{Description} \\
\midrule
$\graph=(\vertices,\edges,F_{\vertices},F_{\edges})$ & Featured input graph used by the evaluated liftings. \\
$\mathcal T=(\mathcal X,F)$ & Featured topological domain. \\
$\mathcal X_r$, $F_r$, $\mathbf X_r$, $d_r$ & Rank-$r$ elements, feature map, feature matrix, and feature width, respectively. \\
$\lifting$, $\structlifting$, $\featurelifting$ & Complete lifting and its structural and feature components, respectively. \\
$\mathbf B_{r,r+1}$, $\mathbf H$ & Incidence matrix between ranks $r$ and $r+1$, and node-hyperedge incidence matrix, respectively. \\
$K$ & Total number of clusters in the fixed graph partition. \\
$\Pi=\{P_k\}_{k=1}^{K}$ & Fixed partition of the graph nodes. \\
$q$, $Z$, $B$, $\graph_B$ & $q$ is the number of clusters per mini-batch. $Z$ is a set of cluster indices, $B=\bigcup_{k\in Z}P_k$ is the resulting node set, and $\graph_B$ is the graph induced by $B$. \\
\bottomrule
\end{tabular}
\end{center}

\section{Featured topological domains and lifting maps}
\label{sec:domains_liftings}

We adopt the featured-domain and lifting formalism of \tb{}~\citep{TopoBench}, restricting the presentation to the objects used in this study. The input to every lifting is a featured graph, and the output is a featured graph, hypergraph, simplicial complex, or cell complex.

\subsection{Featured topological domains}
\label{app:topological-domains}

\begin{definition}[Featured graph]
A featured graph is a tuple
\[
    \graph=(\vertices,\edges,F_{\vertices},F_{\edges}),
    \qquad
    F_{\vertices}:\vertices\rightarrow\mathbb{R}^{d_{\vertices}},
    \qquad
    F_{\edges}:\edges\rightarrow\mathbb{R}^{d_{\edges}}.
\]
Here, $(\vertices,\edges)$ specifies the underlying graph, with $\vertices$ being the set of nodes and $\edges$ the set of edges. While $F_{\vertices}$ and $F_{\edges}$ assign the node and edge features, respectively. We omit $F_{\edges}$ when the edge features are not available.
\end{definition}

\paragraph{Topological domains.}
A topological domain is a discrete representation of a system in terms of its constituent elements and the relations among them. Beyond pairwise graph edges, these relations may connect several nodes or organize elements through containment, incidence, or boundary relations. The term therefore provides a common language for graphs and their higher-order generalizations~\citep{hajij2023tdl,papillon2023architectures,TopoBench}.

\paragraph{Higher-order domains.}
We consider three families of higher-order domains. Hypergraphs represent set-type relations by allowing a hyperedge to connect any number of nodes, without imposing a hierarchy among hyperedges. Simplicial complexes add a strict part-whole hierarchy in which every simplex contains all of its lower-dimensional faces. Cell complexes retain a boundary-based hierarchy but admit non-simplicial cells, including polygonal faces that can represent graph cycles directly~\citep{hajij2023tdl,papillon2023architectures,TopoBench}.

\begin{definition}[Hypergraph]
Let $\mathcal{P}(\vertices)$ denote the power set of a non-empty node set $\vertices$. A hypergraph on $\vertices$ is a pair $H=(\vertices,\edges_H)$, where $\edges_H\subseteq\mathcal{P}(\vertices)\setminus\{\emptyset\}$.
\end{definition}

\begin{definition}[Simplicial complex]
A simplicial complex on a non-empty node set $\vertices$ is a pair $S=(\vertices,\mathcal{X})$, where $\mathcal{X}\subseteq\mathcal{P}(\vertices)\setminus\{\emptyset\}$ and
\[
    \sigma\in\mathcal{X},\quad
    \emptyset\neq\tau\subseteq\sigma
    \quad\Longrightarrow\quad
    \tau\in\mathcal{X}.
\]
The elements of $\mathcal{X}$ are called simplices.
\end{definition}

\begin{definition}[Cell complex]
A regular cell complex is a topological space $S$ partitioned into cells $\{x_\alpha\}_{\alpha\in P_S}$, where $P_S$ is an index set, such that:
\begin{enumerate}
    \item $S=\bigcup_{\alpha\in P_S}\operatorname{int}(x_\alpha)$.
    \item for every $\alpha\in P_S$, the interior of $x_\alpha$ is homeomorphic to $\mathbb{R}^{n_\alpha}$ for some $n_\alpha\in\mathbb{Z}_{\geq0}$, and $n_\alpha$ is the dimension of $x_\alpha$.
    \item the boundary of each cell is a union of finitely many cells of strictly lower dimension.
\end{enumerate}
\end{definition}

The experimental liftings instantiate these domains using node neighbourhoods as hyperedges, cliques as simplicial cells, and graph cycles as cellular 2-cells, as detailed in Appendix~\ref{app:experimental_liftings}.

\begin{definition}[Featured topological domain]
Let $\mathcal T=(\mathcal{X},F)$, where $\mathcal{X}$ is a topological domain, and
\[
    F=\{F_r\}_{r\geq0},
    \qquad
    F_r:\mathcal{X}_r\rightarrow\mathbb{R}^{d_r}
\]
assigns feature vectors to its rank-$r$ elements. For hypergraphs the feature maps instead distinguish nodes and hyperedges.
\end{definition}

\subsection{Lifting maps}
\label{app:liftings}

\begin{definition}[Lifting between featured topological domains]
Let $\mathcal T_{\mathrm{src}}=(\mathcal{X}_{\mathrm{src}},F_{\mathrm{src}})$ be the featured source domain and $\mathcal T_{\mathrm{tgt}}=(\mathcal{X}_{\mathrm{tgt}},F_{\mathrm{tgt}})$ the featured target domain. A lifting $\lifting:\mathcal T_{\mathrm{src}}\rightarrow\mathcal T_{\mathrm{tgt}}$ consists of a structural map $\structlifting$ and a feature map $\featurelifting$:
\begin{equation}
    \mathcal{X}_{\mathrm{tgt}}=\structlifting(\mathcal{X}_{\mathrm{src}},F_{\mathrm{src}}),
    \qquad
    F_{\mathrm{tgt}}=\featurelifting(\mathcal{X}_{\mathrm{src}},F_{\mathrm{src}}),
    \qquad
    \lifting(\mathcal T_{\mathrm{src}})=\mathcal T_{\mathrm{tgt}}.
    \label{eq:lifting_components}
\end{equation}
The structural map constructs the elements and incidence relations of $\mathcal{X}_{\mathrm{tgt}}$, while the feature map transfers or constructs their features $F_{\mathrm{tgt}}$.
\end{definition}

This definition follows the lifting formalism of \tb{}~\citep{TopoBench}. In our experiments, $\mathcal T_{\mathrm{src}}$ is a featured graph and $\mathcal T_{\mathrm{tgt}}$ is the representation required by the model. For example, the clique-complex lifting constructs simplicial cells from graph cliques and assigns features to the resulting elements. All structural liftings used in this study derive the target domain from graph connectivity.

For $B\subseteq\vertices$, the featured graph provided to a mini-batch lifting is
\[
    \graph_B
    =\left(B,\edges_B,
    \left.F_{\vertices}\right|_B,
    \left.F_{\edges}\right|_{\edges_B}\right),
    \qquad
    \edges_B=\{(u,v)\in\edges:u,v\in B\},
\]
with the edge-feature term omitted when unavailable. During mini-batch training and inference, the workflow constructs $\lifting(\graph_B)$ directly without first materializing $\lifting(\graph)$.

\subsubsection{Lifting maps used in the experiments}
\label{app:experimental_liftings}

The experiments instantiate each lifting on a featured graph \(\graph=(\vertices,\edges,F_{\vertices},F_{\edges})\). We first specify its structural component \(\structlifting\). The identity lifting preserves the input features, whereas the higher-order liftings use a common feature map that assigns features to the constructed elements. Each construction applies to either the full graph or an induced mini-batch graph, with \(\graph_B\) substituted for \(\graph\). The lifting-specific scale parameters and their experimental values are stated below.

\paragraph{Identity graph lifting.}
The identity lifting preserves the graph and its features,
\[
    \lifting_{\mathrm{id}}(\graph)=\graph.
\]
We use it with GCN.

\paragraph{$k$-hop neighbourhood hypergraph lifting.}
The $k$-hop lifting assigns a hyperedge to the closed $k$-hop neighbourhood of each node~\citep{TopoBench}. For an integer $k\geq1$ and a node $v\in\vertices$, define
\[
    N_{\graph}^{(k)}[v]
    =\{u\in\vertices:\operatorname{dist}_{\graph}(u,v)\leq k\},
\]
where $\operatorname{dist}_{\graph}$ denotes shortest-path distance in $\graph$. The resulting node-indexed hyperedge family and structural lifting are
\[
    \edges_H^{(k)}=\bigl(e_v\bigr)_{v\in\vertices},
    \qquad
    e_v=N_{\graph}^{(k)}[v],
    \qquad
    \mathcal{L}_{S,k}^{\mathrm{hypergraph}}(\graph)
    =\left(\vertices,\edges_H^{(k)}\right).
\]
The center-node index is retained, so the implementation creates exactly one incidence-matrix column per node even when two indexed hyperedges have identical memberships.
The parameter $k$ determines the graph radius represented by each hyperedge. We set $k=1$ and use this lifting with EDHNN and UniGNN.

\paragraph{Cycle-basis cell lifting.}
Cycle-based cellular liftings attach 2-cells along selected graph cycles~\citep{bodnar2021cwn,TopoBench}. Let $\mathcal{B}(\graph)$ denote a fundamental cycle basis obtained using Paton's algorithm~\citep{paton1969cycles}. For a maximum cycle length $\ell_{\max}$, we retain
\[
    \mathcal{B}_{\leq\ell_{\max}}(\graph)
    =\{c\in\mathcal{B}(\graph):|c|\leq\ell_{\max}\}.
\]
Attaching one 2-cell along each retained cycle gives
\[
    \mathcal{L}_{S,\ell_{\max}}^{\mathrm{cell}}(\graph)
    =(\mathcal{C}_0,\mathcal{C}_1,\mathcal{C}_2),
    \qquad
    \mathcal{C}_0=\vertices,
    \qquad
    \mathcal{C}_1=\edges,
    \qquad
    \mathcal{C}_2=\mathcal{B}_{\leq\ell_{\max}}(\graph).
\]
We set $\ell_{\max}=9$ and use this lifting with CWN and TopoTune.

\paragraph{Clique-complex lifting.}
The clique-complex lifting maps every clique of size $r+1$ to an $r$-simplex~\citep{bodnar2021cwn,TopoBench}. For a maximum simplex dimension $d_{\max}$, define the truncated clique complex
\[
    \mathrm{Cl}_{\leq d_{\max}}(\graph)
    =\left\{\emptyset\neq\sigma\subseteq\vertices:
    |\sigma|\leq d_{\max}+1,
    (u,v)\in\edges\ \text{for all distinct }u,v\in\sigma\right\}.
\]
The corresponding structural lifting is
\[
    \mathcal{L}_{S,d_{\max}}^{\mathrm{simplicial}}(\graph)
    =\left(\vertices,\mathrm{Cl}_{\leq d_{\max}}(\graph)\right).
\]
We set $d_{\max}=2$, yielding nodes, edges, and filled triangles. We use this lifting with SCN and SCCNN.

\paragraph{Feature lifting.}
The preceding maps specify the structural component $\structlifting$ of each lifting. In the experiments, the common feature component $\featurelifting$ retains existing features and constructs missing higher-rank features by summing over incident lower-rank elements~\citep{TopoBench}. For a ranked domain, let $\mathbf{B}_{r,r+1}$ be the incidence matrix between rank-$r$ and rank-$(r+1)$ elements. The feature assignment from $\mathbf{X}_r$ to $\mathbf{X}_{r+1}$ is
\[
    \mathbf{X}_{r+1}
    =\left|\mathbf{B}_{r,r+1}\right|^{\mathsf{T}}\mathbf{X}_r.
\]
The absolute value removes orientation signs, and the operation is applied successively across ranks. For a hypergraph with node-hyperedge incidence matrix $\mathbf{H}$, the corresponding assignment is $\mathbf{X}_{H}=|\mathbf{H}|^{\mathsf{T}}\mathbf{X}_0$, where $\mathbf{X}_0$ is the input node-feature matrix. Each new element therefore receives the sum of the features on its incident source elements. The feature assignment does not alter the structure produced by $\structlifting$.

\section{Practical scalability of graph-supported TDL: bottlenecks and solution paths}
\label{sec:tdl_scalability_appendix}

\subsection{Pipeline view of graph-supported TDL scalability}
\label{subsec:tdl_scalability_pipeline}

The scalability of graph-supported TDL depends on the entire computational pipeline, not only on the TNN architecture. The input graph must first be lifted by identifying higher-order structures and constructing their feature matrices and connectivity operators. In the evaluated implementation, the constructed representation is held in CPU memory and transferred to the GPU as needed for training or inference. A bottleneck at any one of these stages can make an otherwise efficient architecture impractical.

Where these costs occur depends on the execution regime. Full-domain execution constructs the topology and lifts features once during preprocessing, then reuses the complete lifted representation throughout training and evaluation. \our instead partitions the graph during preprocessing and reconstructs and lifts individual induced graph mini-batches during training and inference. This shift trades full-domain storage for repeated local construction and therefore changes preprocessing, execution time, and memory jointly. Both regimes use the same parameterized lifting procedure and structural restrictions. Applying that procedure to induced subgraphs, however, need not produce the restriction of the full-domain lift.

We use the notation introduced in Appendix~\ref{sec:domains_liftings}. For a lifting \(\lifting=(\structlifting,\featurelifting)\), the stored representation depends on the number of generated elements \(|\mathcal X_r|\), the number of entries in its incidence or neighbourhood operators, and the feature width \(d_r\) at each rank. The structural counts in Table~\ref{tab:tdl_structural_counts} determine the number of feature rows. The feature widths convert those row counts into the storage estimates in Table~\ref{tab:tdl_static_feature_storage}. Table~\ref{tab:tdl_preprocessing_cost} distinguishes models that request different operators from the same lifted domain.

\subsection{Representation growth: from structure to storage}
\label{subsec:tdl_scalability_representation}

Structural expansion and feature storage are related but distinct. We first explain how liftings create higher-order structures and represent their connectivity (Sec.~\ref{subsec:tdl_scalability_liftings}). We then compare the structural characteristics of the benchmarks and their lifted domains (Sec.~\ref{subsec:tdl_scalability_benchmark_coverage}). Finally, we combine these structural counts with feature dimensions to estimate feature-matrix storage (Sec.~\ref{subsec:tdl_scalability_feature_storage}).

\subsubsection{Lifting-specific mechanisms of structural expansion}
\label{subsec:tdl_scalability_liftings}

Different liftings can produce very different numbers of higher-order structures from the same graph. The connectivity information also varies, including node memberships in hyperedges and the edges forming cell boundaries. We illustrate these sources of growth using the neighbourhood-hypergraph, cycle-based cellular, and clique-complex liftings used in our experiments, which are \tb{}'s default graph mappings to the corresponding domains.

\paragraph{Neighbourhood-hypergraph lifting.}

For the evaluated \(k=1\) neighbourhood-hypergraph lifting, each node defines one indexed hyperedge containing its closed one-hop neighbourhood. The number of hyperedges is therefore exactly \(|\vertices|\) for every dataset and provides no additional distinction beyond the node count. Connectivity nevertheless determines the incidence size. Across all neighbourhood hyperedges, the lifting stores \(|\vertices|+2|\edges|\) node-hyperedge memberships, corresponding to an average hyperedge size of \(1+2|\edges|/|\vertices|\).

\paragraph{Cycle-based cellular lifting.}

The cycle-based cellular lifting first computes a fundamental cycle basis and retains cycles of length at most \(\ell_{\max}=9\), attaching one cellular 2-cell along every retained cycle. Each retained cycle contributes one 2-cell and one nonzero entry in the edge-cell incidence matrix for each edge on its boundary. Its storage requirements therefore depend on both the number and lengths of the retained cycles. The retained cycle basis can also change when the lifting is recomputed on an induced subgraph, as discussed in Appendix~\ref{subsec:exp_cora}.\footnote{The evaluated implementation computes the complete cycle basis before discarding cycles longer than \(\ell_{\max}\). The length restriction therefore bounds the retained cell complex, but not necessarily the time or temporary memory required to construct the basis.}

\paragraph{Clique-complex lifting.}

The clique-complex lifting fills every graph triangle with a simplicial 2-cell. Triangle-rich graphs can therefore produce far more 2-cells than input nodes or edges. In the stored representation, each generated 2-cell contributes one feature-vector row and three nonzero entries to the edge-cell incidence matrix, one for each boundary edge.

\subsubsection{Benchmark scale regimes and coverage}
\label{subsec:tdl_scalability_benchmark_coverage}

Table~\ref{tab:tdl_structural_counts} reports the numbers of input nodes and edges, together with the numbers of 2-cells generated by the cellular and simplicial liftings.\footnote{For the evaluated \(k=1\) neighbourhood-hypergraph lifting, the hyperedge count is \(|\vertices|\) and the node-hyperedge membership count is \(|\vertices|+2|\edges|\). Both follow from the input-graph columns and are omitted from the table.} Mean node degree, calculated as \(2|\edges|/|\vertices|\), describes how many neighbours an input node has on average. It distinguishes graphs with similar edge counts but substantially different node counts.

For each lifting, the \emph{Per node} column divides the total number of generated 2-cells by the number of input nodes. This shows how large the generated 2-cell set is relative to the original node set. The absolute counts describe the size of the lifted domain, while these ratios allow us to compare structural expansion across graphs of different sizes. They are ratios of total counts, not counts of the cells incident to an average node.

\begin{table}[h]
\centering
\caption{Structural characteristics of the input graphs and their lifted domains. The edge column counts each undirected relation once, and mean node degree is derived as \(2|\edges|/|\vertices|\). Per-node values are the corresponding 2-cell counts divided by the number of input nodes. Cellular 2-cells are retained basis cycles of length at most nine, whereas simplicial 2-cells are filled graph triangles.}
\label{tab:tdl_structural_counts}
\small
\setlength{\tabcolsep}{4pt}
\resizebox{\textwidth}{!}{%
\begin{tabular}{lrrrrrrr}
\toprule
\multirow{2}{*}{\textbf{Dataset}}
& \multicolumn{3}{c}{\textbf{Input graph}}
& \multicolumn{2}{c}{\textbf{Cellular 2-cells}}
& \multicolumn{2}{c}{\textbf{Simplicial 2-cells}} \\
\cmidrule(lr){2-4}\cmidrule(lr){5-6}\cmidrule(lr){7-8}
& \textbf{Nodes}
& \textbf{Edges}
& \makecell{\textbf{Mean node}\\\textbf{degree}}
& \textbf{Count}
& \makecell{\textbf{Per}\\\textbf{node}}
& \textbf{Count}
& \makecell{\textbf{Per}\\\textbf{node}} \\
\midrule
Questions        & $48\,921$    & $153\,540$    & $6.28$   & $24\,405$     & $0.50$   & $110\,209$       & $2.25$ \\
Amazon Ratings   & $24\,492$    & $93\,050$     & $7.60$   & $63\,667$     & $2.60$   & $110\,765$       & $4.52$ \\
Cora Full        & $19\,793$    & $63\,421$     & $6.41$   & $24\,306$     & $1.23$   & $48\,386$        & $2.44$ \\
Coauthor Physics & $34\,493$    & $247\,962$    & $14.38$  & $102\,152$    & $2.96$   & $468\,550$       & $13.58$ \\
Reddit           & $232\,965$   & $57\,307\,946$ & $491.99$ & $40\,582\,059$ & $174.20$ & $8\,360\,338\,411$ & $35\,886.67$ \\
OGBN Products    & $2\,449\,029$ & $61\,859\,140$ & $50.52$  & $34\,226\,557$ & $13.98$  & $621\,082\,120$   & $253.60$ \\
\bottomrule
\end{tabular}%
}
\end{table}

\paragraph{Input and lifted scale regimes.}

Questions, Amazon Ratings, and Cora Full have mean node degrees between \(6.28\) and \(7.60\). Coauthor Physics reaches \(14.38\) and produces \(13.58\) simplicial 2-cells per node, placing it between the smaller controlled benchmarks and the two scalability-oriented cases.

Reddit and OGBN Products expose complementary scalability regimes. OGBN Products contains approximately \(10.5\times\) more nodes, yet the two graphs have comparable edge counts. Reddit consequently has a mean node degree of \(491.99\), compared with \(50.52\) for OGBN Products. The higher-order constructions also differ in their expansion. Reddit contains \(174.20\) cellular and \(35886.67\) simplicial 2-cells per node, whereas OGBN Products contains \(13.98\) and \(253.60\), respectively. These datasets therefore contrast greater higher-order expansion per node with a substantially larger input node count at comparable edge counts.

\paragraph{Why the benchmarks are complementary.}

The same progression gives the datasets complementary experimental roles. Questions, Amazon Ratings, and Cora Full support controlled comparisons between full-domain and \our partitioned execution because every evaluated lifting can be constructed on the complete graph. Coauthor Physics represents an intermediate case in which a moderate node count already produces substantially greater higher-order expansion. Reddit and OGBN Products extend the suite to graphs with tens of millions of edges and contrasting node counts. Together, these datasets support controlled comparisons with full-domain execution and allow us to examine how graph size, connectivity, and lifting-induced structural expansion affect computational cost.

The benchmarks also cover distinct application domains and relation semantics. Questions and Reddit represent online interactions, although Questions connects users through question-answer activity and Reddit connects posts through shared commenters~\citep{platonov2023critical,reddit_paper}. Cora Full and Coauthor Physics represent scientific-information networks through citation and coauthorship relations~\citep{bojchevski2018deep,pitfalls}. Amazon Ratings and OGBN Products represent product co-purchasing networks at different scales~\citep{platonov2023critical,hu2020}. Together, these datasets expose the pipeline to graph structures generated by different relational processes while keeping the downstream task fixed as node classification.

\subsubsection{Static feature-matrix storage}
\label{subsec:tdl_scalability_feature_storage}

Here, static feature-matrix storage denotes the persistent rank-wise input feature matrices of the full lifted domain. It excludes incidence tensors, derived operators, temporary construction objects, activations, gradients, and optimizer state.

Feature width provides a further source of variation across the benchmarks. Cora Full and Coauthor Physics have input widths of \(8\,710\) and \(8\,415\), respectively, compared with \(602\) for Reddit and \(100\) for OGBN Products. Their wider features increase storage per generated element, so a smaller graph need not have a proportionally smaller feature-storage requirement.

The memory required for the rank-wise feature matrices follows directly from the structural counts. At rank \(r\), the feature matrix contains \(|\mathcal X_r|\) rows and \(d_r\) values per row. We assume single-precision floating-point storage (float32), which requires four bytes per value. Storing the feature matrices through rank two therefore requires
\begin{equation}
M_{\mathrm{features}}
=4\left(
|\mathcal X_0|d_0
+|\mathcal X_1|d_1
+|\mathcal X_2|d_2
\right)\ \text{bytes}.
\label{eq:tdl_static_feature_payload}
\end{equation}
For these estimates, each generated element receives a feature vector with the dataset's input feature width, consistent with the feature assignment described in Appendix~\ref{app:experimental_liftings}. For cellular and simplicial domains, this gives \(d_0=d_1=d_2\) and \(M_{\mathrm{features}}=4d_0\sum_r|\mathcal X_r|\). For hypergraphs, we count the node and hyperedge feature matrices. Table~\ref{tab:tdl_static_feature_storage} uses the structural counts in Table~\ref{tab:tdl_structural_counts}. Other numerical precisions may be appropriate depending on the features and numerical requirements. For unchanged matrix dimensions, storage scales directly with the number of bytes per value. These estimates account only for input feature matrices, not total preprocessing or training memory.


\begin{table}[h]
\centering
\caption{Estimated memory required to store the float32 input feature matrices of each full-domain lifting. Values use the input feature dimension of each dataset. Observed preprocessing memory is reported in Table~\ref{tab:tdl_preprocessing_cost}.}
\label{tab:tdl_static_feature_storage}
\small
\setlength{\tabcolsep}{7pt}
\resizebox{\textwidth}{!}{%
\begin{tabular}{lcccc}
\toprule
\textbf{Dataset} & \makecell{\textbf{Input width}\\\textbf{$d_0$}} & \makecell{\textbf{$k=1$}\\\textbf{hypergraph}} & \makecell{\textbf{Cycle cell}\\\textbf{complex}} & \makecell{\textbf{Clique simplicial}\\\textbf{complex}} \\
\midrule
Questions        & $301$   & $0.118$ GB & $0.273$ GB   & $0.376$ GB \\
Amazon Ratings   & $300$   & $0.059$ GB & $0.217$ GB   & $0.274$ GB \\
Cora Full        & $8\,710$ & $1.379$ GB & $3.746$ GB   & $4.585$ GB \\
Coauthor Physics & $8\,415$ & $2.322$ GB & $12.946$ GB  & $25.279$ GB \\
Reddit           & $602$   & $1.122$ GB & $236.280$ GB & $20.270$ TB \\
OGBN Products    & $100$   & $1.959$ GB & $39.414$ GB & $274.156$ GB \\
\bottomrule
\end{tabular}%
}
\end{table}

For OGBN Products, substituting \(d_0=d_1=d_2=100\) and the structural counts from Table~\ref{tab:tdl_structural_counts} into Eq.~\eqref{eq:tdl_static_feature_payload} gives 39.414~GB for the cellular feature matrices and 274.156~GB for the simplicial feature matrices.

For Reddit, the cellular calculation is explicit:
\begin{align}
M_{\mathrm{features}}^{\mathrm{cell}}
&=4\times602\times
\left(232\,965+57\,307\,946+40\,582\,059\right) \notag\\
&=236\,280\,111\,760\ \text{bytes}
\approx236.280\ \text{GB}.
\label{eq:reddit_cell_feature_payload}
\end{align}
Replacing the cycle-cell count with Reddit's $8.36$B simplicial 2-cells in the same calculation gives the 20.270~TB simplicial estimate reported in Table~\ref{tab:tdl_static_feature_storage}. Thus, before adding sparse incidence tensors or any neural-network state, the feature matrices alone can require hundreds of gigabytes for the cellular lifting and tens of terabytes for the simplicial lifting. This estimate results from storing a 602-dimensional float32 feature vector for every node, edge, and generated higher-order element.

\subsection{Full-domain CPU preprocessing: construction time and memory}
\label{subsec:tdl_scalability_preprocessing}

The analytical estimate above accounts only for persistent rank-wise feature matrices. The evaluated implementation stores incidence relations sparsely and constructs only the additional operators requested by a model.\footnote{See the \hyperref[app:reproducility_statement]{reproducibility statement} for implementation and reproduction information.} Cellular models may derive adjacency or Laplacian operators from cell incidences, while simplicial models may derive Laplacians from simplicial incidences. Full-domain preprocessing can therefore hold incidence matrices, derived operators, temporary construction objects, and implementation overhead in memory at the same time. We measure peak CPU memory as the largest resident-set size (RSS) sampled for the main Python process during full-domain preprocessing. It should be interpreted as the largest sampled resident-memory footprint of the main process, rather than as the size of any single tensor. Sampling may miss short-lived peaks and does not include memory held exclusively by child processes. Table~\ref{tab:tdl_preprocessing_cost} reports the available measurements for the evaluated dataset-lifting-model configurations.

\begingroup
\setlength{\LTcapwidth}{\textwidth}
\makeatletter
\patchcmd{\LT@makecaption}{\reset@font}{\reset@font\normalsize}{}{}
\makeatother
\scriptsize
\setlength{\tabcolsep}{2pt}
\renewcommand{\arraystretch}{1.12}
\setlength{\LTleft}{0pt}
\setlength{\LTright}{0pt}
\begin{longtable}{@{}>{\raggedright\arraybackslash}p{0.13\textwidth}>{\raggedright\arraybackslash}p{0.15\textwidth}>{\raggedright\arraybackslash}p{0.12\textwidth}>{\raggedright\arraybackslash}p{0.30\textwidth}r@{\hspace{6pt}}r@{}}
\caption{Available full-graph preprocessing time and peak CPU memory measurements for the evaluated dataset-lifting-model configurations. Models are distinguished because they may request different operators from the same lifted domain. Values are reported as the mean $\pm$ standard deviation across five runs. Memory usage is reported in decimal GB.}
\label{tab:tdl_preprocessing_cost}\\
\toprule
\textbf{Dataset} &
\textbf{Lifting} &
\textbf{Model} &
\textbf{Additional operators} &
\makecell[c]{\textbf{Preprocessing}\\\textbf{time (s)}} &
\makecell[c]{\textbf{Peak CPU}\\\textbf{memory (GB)}} \\
\midrule
\endfirsthead
\multicolumn{6}{c}{\tablename~\thetable{} continued} \\
\toprule
\textbf{Dataset} &
\textbf{Lifting} &
\textbf{Model} &
\textbf{Additional operators} &
\makecell[c]{\textbf{Preprocessing}\\\textbf{time (s)}} &
\makecell[c]{\textbf{Peak CPU}\\\textbf{memory (GB)}} \\
\midrule
\endhead
\midrule
\multicolumn{6}{r}{Continued on next page} \\
\endfoot
\bottomrule
\endlastfoot
Amazon Ratings & $k=1$ neighbourhood hypergraph & EDHNN & None beyond the node-hyperedge incidence & $0.42 \pm 0.07$ & $1.00 \pm 0.01$ \\
Amazon Ratings & $k=1$ neighbourhood hypergraph & UniGNN & None beyond the node-hyperedge incidence & $0.39 \pm 0.02$ & $1.01 \pm 0.01$ \\
Amazon Ratings & Cycle-basis cell & CWN & Rank-1 upper adjacency & $5.7 \pm 0.2$ & $1.38 \pm 0.00$ \\
Amazon Ratings & Cycle-basis cell & TopoTune & Rank-1 and two-hop rank-0 upper adjacencies & $8.5 \pm 0.2$ & $1.39 \pm 0.02$ \\
Amazon Ratings & Dimension-2 clique complex & SCN & None beyond the rank-1 and rank-2 incidences & $4.3 \pm 0.3$ & $1.44 \pm 0.00$ \\
Amazon Ratings & Dimension-2 clique complex & SCCNN & None beyond the rank-1 and rank-2 incidences & $4.3 \pm 0.3$ & $1.44 \pm 0.00$ \\
Questions & $k=1$ neighbourhood hypergraph & EDHNN & None beyond the node-hyperedge incidence & $0.63 \pm 0.09$ & $1.09 \pm 0.02$ \\
Questions & $k=1$ neighbourhood hypergraph & UniGNN & None beyond the node-hyperedge incidence & $0.61 \pm 0.05$ & $1.10 \pm 0.03$ \\
Questions & Cycle-basis cell & CWN & Rank-1 upper adjacency & $6.1 \pm 0.1$ & $1.52 \pm 0.08$ \\
Questions & Cycle-basis cell & TopoTune & Rank-1 and two-hop rank-0 upper adjacencies & $7.3 \pm 0.4$ & $1.51 \pm 0.07$ \\
Questions & Dimension-2 clique complex & SCN & None beyond the rank-1 and rank-2 incidences & $6.2 \pm 0.3$ & $1.67 \pm 0.03$ \\
Questions & Dimension-2 clique complex & SCCNN & None beyond the rank-1 and rank-2 incidences & $6.2 \pm 0.4$ & $1.65 \pm 0.06$ \\
Cora Full & $k=1$ neighbourhood hypergraph & EDHNN & None beyond the node-hyperedge incidence & $3.3 \pm 0.1$ & $2.98 \pm 0.00$ \\
Cora Full & $k=1$ neighbourhood hypergraph & UniGNN & None beyond the node-hyperedge incidence & $3.4 \pm 0.1$ & $2.98 \pm 0.00$ \\
Cora Full & Cycle-basis cell & CWN & Rank-1 upper adjacency & $12.7 \pm 1.5$ & $6.14 \pm 0.01$ \\
Cora Full & Cycle-basis cell & TopoTune & Rank-1 and two-hop rank-0 upper adjacencies & $13.6 \pm 0.7$ & $6.14 \pm 0.01$ \\
Cora Full & Dimension-2 clique complex & SCN & None beyond the rank-1 and rank-2 incidences & $13.7 \pm 0.7$ & $7.28 \pm 0.01$ \\
Cora Full & Dimension-2 clique complex & SCCNN & None beyond the rank-1 and rank-2 incidences & $14.8 \pm 3.3$ & $7.28 \pm 0.01$ \\
Coauthor Physics & $k=1$ neighbourhood hypergraph & EDHNN & None beyond the node-hyperedge incidence & $6.2 \pm 0.7$ & $4.48 \pm 0.01$ \\
Coauthor Physics & $k=1$ neighbourhood hypergraph & UniGNN & None beyond the node-hyperedge incidence & $6.5 \pm 0.8$ & $4.48 \pm 0.01$ \\
Coauthor Physics & Cycle-basis cell & CWN & Rank-1 upper adjacency & $51.0 \pm 9.3$ & $19.79 \pm 0.01$ \\
Coauthor Physics & Cycle-basis cell & TopoTune & Rank-1 and two-hop rank-0 upper adjacencies & $55.3 \pm 7.0$ & $19.79 \pm 0.01$ \\
Coauthor Physics & Dimension-2 clique complex & SCN & None beyond the rank-1 and rank-2 incidences & $75.9 \pm 5.6$ & $42.42 \pm 0.02$ \\
Coauthor Physics & Dimension-2 clique complex & SCCNN & None beyond the rank-1 and rank-2 incidences & $77.3 \pm 5.3$ & $42.42 \pm 0.01$ \\
Reddit & $k=1$ neighbourhood hypergraph & EDHNN & None beyond the node-hyperedge incidence & $137.7 \pm 7.6$ & $20.77 \pm 0.08$ \\
Reddit & $k=1$ neighbourhood hypergraph & UniGNN & None beyond the node-hyperedge incidence & $136.9 \pm 6.9$ & $20.79 \pm 0.12$ \\
Reddit & Cycle-basis cell & CWN & Rank-1 upper adjacency & $4684.8 \pm 304.8$ & $479.96 \pm 0.12$ \\
Reddit & Cycle-basis cell & TopoTune & Rank-1 and two-hop rank-0 upper adjacencies & $6965.2 \pm 682.4$ & $480.06 \pm 0.01$ \\
Reddit & Dimension-2 clique complex & SCN & None beyond the rank-1 and rank-2 incidences & OOM & OOM \\
Reddit & Dimension-2 clique complex & SCCNN & None beyond the rank-1 and rank-2 incidences & OOM & OOM \\
\end{longtable}
\endgroup

For OGBN Products, the higher-order structures can be constructed, but retaining their complete feature matrices exceeds the available CPU RAM in the affected full-graph configurations. Even if this preprocessing bottleneck were overcome, full-graph training is expected to exceed the available GPU memory. We therefore report partitioned results with \our, while omitting incomplete full-graph preprocessing measurements from Table~\ref{tab:tdl_preprocessing_cost}.

\paragraph{Reddit full-domain bottlenecks.}

For Reddit, Eq.~\eqref{eq:reddit_cell_feature_payload} predicts 236.280~GB for the cellular feature matrices alone, before accounting for incidence matrices, model-requested operators, and temporary allocations. Across five runs, CWN preprocessing takes $4684.8\pm304.8$~s and reaches a peak CPU memory usage of $479.96\pm0.12$~GB.

Storing the full Reddit simplicial input feature matrices at the original feature width in float32 would require an estimated 20.270~TB (Table~\ref{tab:tdl_static_feature_storage}).

\subsection{Phase-specific memory during preprocessing and TNN execution}
\label{subsec:tdl_scalability_execution}

The following paired analyses focus on Questions, Amazon Ratings, and Cora Full, the three controlled benchmarks with complete full-graph and \our measurements across all seven evaluated models. Together, they provide 21 matched model-dataset comparisons. Each execution setting uses its independently selected hyperparameters, as reported in Appendix~\ref{sec:hyper_opt}. These comparisons therefore characterize the resource requirements and predictive performance of the tuned configurations.

Successfully constructing a lifted full domain does not imply that the complete representation can be processed by a TNN in one optimization step. A TNN layer may update several ranks and retain intermediate messages for backpropagation. For example, an edge representation can receive separate messages from incident nodes and 2-cells. These intermediate tensors add to the GPU memory required by the input features and connectivity operators.

Figure~\ref{fig:appendix_memory_runtime_other_models} places preprocessing CPU-memory peaks alongside training and evaluation GPU-memory peaks for all seven evaluated models, complementing the joint memory-runtime comparison in Figure~\ref{fig:cwn_scalability_memory_runtime}.

For the depth-scaling experiment in Figure~\ref{fig:memory_with_depth}, peak reserved GPU memory is measured over five seeds, with each run lasting three epochs. This experiment uses $512$ hidden channels, $K=64$, and $q=8$.

\begin{figure*}[!ht]
    \centering
    \includegraphics[width=\textwidth]{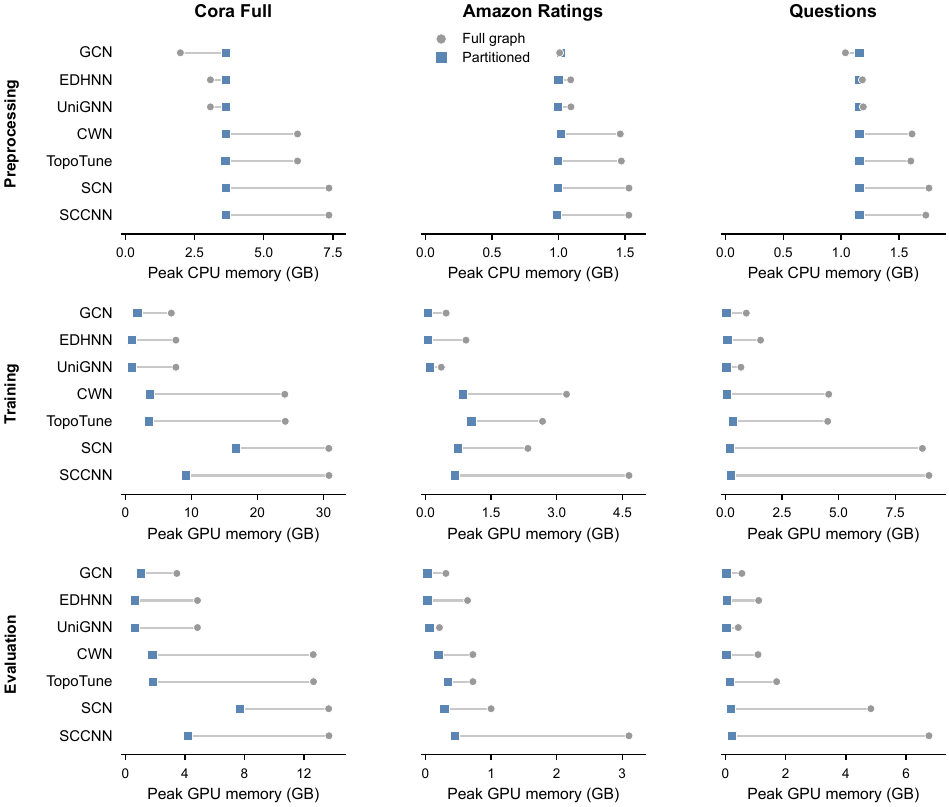}
    \caption{Cluster-TNN consistently reduces peak GPU memory during training and evaluation, whereas preprocessing memory varies by configuration and can increase when partition construction exceeds the cost of full-graph lifting. Values are means across five seeds.}
    \label{fig:appendix_memory_runtime_other_models}
\end{figure*}
\FloatBarrier

Across the displayed comparisons, \our execution lowers peak GPU memory during both training and evaluation, whereas preprocessing CPU memory varies by configuration. These phase-specific measurements locate where memory pressure occurs. The next subsection reports elapsed time and peak CPU and GPU memory over a complete run.

\subsection{End-to-end cost}
\label{subsec:tdl_scalability_end_to_end}

End-to-end scalability has two distinct empirical dimensions: elapsed time across workflow phases and the peak CPU and GPU memory reached anywhere during a run. We report them separately because \our execution can reduce peak memory while increasing repeated lifting, validation, or inference time.

\subsubsection{End-to-end runtime}
\label{subsec:tdl_scalability_runtime}

An end-to-end comparison must account for when each method incurs its costs. Full-domain execution constructs the lifted domain once before training, whereas \our constructs the graph partition once and lifts mini-batches repeatedly. Table~\ref{tab:tdl_end_to_end_runtime} reports preprocessing, training including validation passes, final validation and test inference, and total runtime for the $21$ matched model-dataset comparisons, together with partitioned-only results on Coauthor Physics, Reddit, and OGBN Products. Table~\ref{tab:tdl_preprocessing_cost} reports the separate full-domain preprocessing measurements.

{
\scriptsize
\setlength{\tabcolsep}{3.2pt}
\renewcommand{\arraystretch}{1.02}
\setlength{\LTcapwidth}{\textwidth}
\makeatletter
\patchcmd{\LT@makecaption}{\reset@font}{\reset@font\normalsize}{}{}
\makeatother
\begin{longtable}{@{}lllcc*{4}{r}@{}}
\caption{End-to-end runtime in seconds under full-graph and \our settings. Preprocessing, training, and final evaluation are reported separately. Training is the complete fit call, including periodic validation. Final evaluation spans the end of fitting through checkpoint validation and test inference. Total is the sum of these non-overlapping components. For \our inference, we use the ensemble protocol and average predictions over ten cluster recombinations. Values are mean $\pm$ sample standard deviation in seconds across five seeds.}
\label{tab:tdl_end_to_end_runtime}\\
\toprule
\textbf{Dataset} & \textbf{Model} & \textbf{Setting} & $\boldsymbol{K}$ & $\boldsymbol{q}$ &
\textbf{Preprocessing} & \textbf{Train} & \makecell{\textbf{Final}\\ \textbf{evaluation}} & \textbf{Total} \\
\midrule
\endfirsthead
\multicolumn{9}{c}{\tablename~\thetable{} continued} \\
\toprule
\textbf{Dataset} & \textbf{Model} & \textbf{Setting} & $\boldsymbol{K}$ & $\boldsymbol{q}$ &
\textbf{Preprocessing} & \textbf{Train} & \makecell{\textbf{Final}\\ \textbf{evaluation}} & \textbf{Total} \\
\midrule
\endhead
\midrule
\multicolumn{9}{r}{Continued on next page} \\
\endfoot
\bottomrule
\endlastfoot

& GCN & Full-graph & - & - & $4.0 \pm 2.5$ & $222.3 \pm 57.4$ & $3.4 \pm 0.2$ & $229.7 \pm 57.5$ \\*
& & \our & $32$ & $8$ & $11.2 \pm 1.1$ & $83.7 \pm 8.8$ & $47.2 \pm 0.5$ & $142.1 \pm 8.8$ \\*
 & EDHNN & Full-graph & - & - & $8.0 \pm 0.8$ & $438.7 \pm 114.3$ & $7.2 \pm 0.3$ & $453.9 \pm 113.8$ \\*
& & \our & $32$ & $4$ & $11.2 \pm 1.1$ & $180.5 \pm 40.3$ & $48.7 \pm 0.9$ & $240.3 \pm 40.4$ \\*
 & UniGNN & Full-graph & - & - & $8.4 \pm 0.6$ & $343.8 \pm 71.9$ & $7.1 \pm 0.1$ & $359.3 \pm 71.8$ \\*
& & \our & $32$ & $4$ & $11.2 \pm 1.1$ & $150.7 \pm 21.9$ & $50.4 \pm 1.0$ & $212.2 \pm 22.3$ \\*
Cora Full & CWN & Full-graph & - & - & $18.7 \pm 1.8$ & $909.6 \pm 117.6$ & $14.6 \pm 0.5$ & $942.8 \pm 118.9$ \\*
& & \our & $32$ & $4$ & $11.2 \pm 1.1$ & $776.2 \pm 134.2$ & $272.4 \pm 8.9$ & $1059.8 \pm 135.7$ \\*
 & TopoTune & Full-graph & - & - & $19.4 \pm 1.4$ & $401.1 \pm 53.1$ & $14.4 \pm 0.4$ & $434.9 \pm 51.9$ \\*
& & \our & $64$ & $8$ & $11.4 \pm 2.0$ & $1211.6 \pm 316.8$ & $312.7 \pm 10.6$ & $1535.8 \pm 309.1$ \\*
 & SCN & Full-graph & - & - & $22.0 \pm 4.2$ & $473.5 \pm 97.5$ & $17.0 \pm 0.9$ & $512.5 \pm 100.7$ \\*
& & \our & $32$ & $16$ & $11.2 \pm 1.1$ & $562.9 \pm 37.0$ & $405.2 \pm 25.6$ & $979.3 \pm 38.4$ \\*
 & SCCNN & Full-graph & - & - & $21.6 \pm 0.9$ & $458.9 \pm 22.7$ & $17.3 \pm 0.8$ & $497.9 \pm 22.6$ \\*
& & \our & $32$ & $8$ & $11.2 \pm 1.1$ & $627.1 \pm 118.0$ & $455.6 \pm 64.4$ & $1093.9 \pm 173.4$ \\
\midrule
& GCN & Full-graph & - & - & $2.6 \pm 0.8$ & $81.3 \pm 16.6$ & $1.3 \pm 0.2$ & $85.2 \pm 15.8$ \\*
& & \our & $64$ & $4$ & $2.6 \pm 0.9$ & $29.0 \pm 8.7$ & $6.2 \pm 0.4$ & $37.8 \pm 8.9$ \\*
 & EDHNN & Full-graph & - & - & $3.4 \pm 0.3$ & $143.9 \pm 40.5$ & $1.5 \pm 0.1$ & $148.8 \pm 40.3$ \\*
& & \our & $128$ & $4$ & $1.8 \pm 0.1$ & $100.5 \pm 23.0$ & $12.9 \pm 2.7$ & $115.3 \pm 24.6$ \\*
 & UniGNN & Full-graph & - & - & $3.5 \pm 0.2$ & $119.7 \pm 54.2$ & $1.6 \pm 0.0$ & $124.8 \pm 54.1$ \\*
& & \our & $32$ & $8$ & $1.8 \pm 0.2$ & $26.8 \pm 6.8$ & $6.0 \pm 0.2$ & $34.7 \pm 6.9$ \\*
Amazon Ratings & CWN & Full-graph & - & - & $10.6 \pm 1.6$ & $150.7 \pm 19.5$ & $2.3 \pm 0.1$ & $163.5 \pm 19.2$ \\*
& & \our & $64$ & $16$ & $2.6 \pm 0.9$ & $849.4 \pm 112.4$ & $84.0 \pm 8.9$ & $936.1 \pm 105.9$ \\*
 & TopoTune & Full-graph & - & - & $13.4 \pm 1.3$ & $68.3 \pm 11.9$ & $2.4 \pm 0.2$ & $84.1 \pm 12.4$ \\*
& & \our & $32$ & $8$ & $1.8 \pm 0.2$ & $561.2 \pm 108.0$ & $120.7 \pm 10.3$ & $683.7 \pm 112.1$ \\*
 & SCN & Full-graph & - & - & $7.6 \pm 0.3$ & $240.0 \pm 41.4$ & $2.4 \pm 0.1$ & $250.0 \pm 41.5$ \\*
& & \our & $32$ & $8$ & $1.8 \pm 0.2$ & $572.4 \pm 180.9$ & $61.2 \pm 3.2$ & $635.4 \pm 180.0$ \\*
 & SCCNN & Full-graph & - & - & $7.7 \pm 0.3$ & $209.5 \pm 47.8$ & $2.3 \pm 0.1$ & $219.5 \pm 47.8$ \\*
& & \our & $32$ & $4$ & $1.8 \pm 0.2$ & $525.1 \pm 89.3$ & $76.4 \pm 5.2$ & $603.3 \pm 93.8$ \\
\midrule
& GCN & Full-graph & - & - & $4.6 \pm 2.1$ & $164.5 \pm 29.2$ & $1.3 \pm 0.1$ & $170.4 \pm 28.9$ \\*
& & \our & $500$ & $20$ & $3.8 \pm 0.4$ & $42.0 \pm 8.9$ & $7.0 \pm 0.1$ & $52.9 \pm 9.1$ \\*
 & EDHNN & Full-graph & - & - & $4.0 \pm 0.5$ & $203.1 \pm 42.2$ & $1.7 \pm 0.1$ & $208.8 \pm 42.4$ \\*
& & \our & $500$ & $50$ & $3.8 \pm 0.4$ & $19.5 \pm 4.7$ & $7.1 \pm 0.2$ & $30.4 \pm 5.0$ \\*
 & UniGNN & Full-graph & - & - & $3.8 \pm 0.3$ & $260.0 \pm 10.0$ & $1.7 \pm 0.1$ & $265.5 \pm 10.2$ \\*
& & \our & $1000$ & $40$ & $4.2 \pm 0.1$ & $23.2 \pm 11.6$ & $7.3 \pm 0.2$ & $34.8 \pm 11.6$ \\*
Questions & CWN & Full-graph & - & - & $9.8 \pm 0.4$ & $68.4 \pm 23.6$ & $2.5 \pm 0.1$ & $80.7 \pm 23.8$ \\*
& & \our & $1000$ & $40$ & $4.2 \pm 0.1$ & $116.4 \pm 31.1$ & $22.8 \pm 0.4$ & $143.4 \pm 30.9$ \\*
 & TopoTune & Full-graph & - & - & $12.1 \pm 1.9$ & $170.3 \pm 21.6$ & $2.8 \pm 0.2$ & $185.2 \pm 20.8$ \\*
& & \our & $500$ & $50$ & $3.8 \pm 0.4$ & $226.6 \pm 61.0$ & $32.9 \pm 1.0$ & $263.4 \pm 61.0$ \\*
 & SCN & Full-graph & - & - & $10.2 \pm 0.3$ & $378.6 \pm 119.1$ & $2.8 \pm 0.1$ & $391.6 \pm 119.1$ \\*
& & \our & $500$ & $50$ & $3.8 \pm 0.4$ & $107.6 \pm 9.8$ & $30.6 \pm 1.5$ & $142.0 \pm 11.2$ \\*
 & SCCNN & Full-graph & - & - & $10.2 \pm 0.5$ & $131.8 \pm 14.5$ & $2.9 \pm 0.2$ & $144.9 \pm 14.7$ \\*
& & \our & $500$ & $30$ & $3.8 \pm 0.4$ & $91.5 \pm 9.9$ & $36.1 \pm 1.5$ & $131.5 \pm 10.5$ \\
\midrule
 & GCN & \our & $500$ & $20$ & $10.4 \pm 1.1$ & $126.2 \pm 52.2$ & $91.4 \pm 1.4$ & $228.1 \pm 51.4$ \\*
 & EDHNN & \our & $2\,000$ & $30$ & $11.2 \pm 0.1$ & $284.2 \pm 56.8$ & $55.3 \pm 9.1$ & $350.6 \pm 61.6$ \\*
 & UniGNN & \our & $2\,000$ & $20$ & $11.2 \pm 0.1$ & $258.4 \pm 52.6$ & $42.9 \pm 2.2$ & $312.5 \pm 52.7$ \\*
Coauthor Physics & CWN & \our & $2\,000$ & $20$ & $11.2 \pm 0.1$ & $316.7 \pm 114.0$ & $80.3 \pm 6.6$ & $408.2 \pm 110.9$ \\*
 & TopoTune & \our & $2\,000$ & $20$ & $11.2 \pm 0.1$ & $263.4 \pm 38.4$ & $86.3 \pm 6.3$ & $360.9 \pm 32.9$ \\*
 & SCN & \our & $2\,000$ & $20$ & $11.2 \pm 0.1$ & $339.0 \pm 113.2$ & $96.8 \pm 7.0$ & $447.0 \pm 111.2$ \\*
 & SCCNN & \our & $2\,000$ & $20$ & $11.2 \pm 0.1$ & $599.8 \pm 215.2$ & $161.3 \pm 17.6$ & $772.2 \pm 211.1$ \\
\midrule
 & GCN & \our & $7\,000$ & $25$ & $500.4 \pm 3.7$ & $449.7 \pm 61.9$ & $106.3 \pm 3.1$ & $1056 \pm 62$ \\*
 & EDHNN & \our & $7\,000$ & $20$ & $500.4 \pm 3.7$ & $1165 \pm 255$ & $170.8 \pm 4.6$ & $1836 \pm 253$ \\*
 & UniGNN & \our & $7\,000$ & $25$ & $500.4 \pm 3.7$ & $967.1 \pm 208.2$ & $116.8 \pm 3.4$ & $1584 \pm 202$ \\*
Reddit & CWN & \our & $7\,000$ & $20$ & $500.4 \pm 3.7$ & $27148 \pm 1703$ & $1659 \pm 26$ & $29307 \pm 1703$ \\*
 & TopoTune & \our & $7\,000$ & $20$ & $500.4 \pm 3.7$ & $18696 \pm 2834$ & $2617 \pm 35$ & $21813 \pm 2858$ \\*
 & SCN & \our & $7\,000$ & $25$ & $500.4 \pm 3.7$ & $26897 \pm 981$ & $4342 \pm 95$ & $31739 \pm 950$ \\*
 & SCCNN & \our & $7\,000$ & $20$ & $500.4 \pm 3.7$ & $28571 \pm 8714$ & $5145 \pm 128$ & $34217 \pm 8659$ \\
\midrule
 & GCN & \our & $15\,000$ & $25$ & $245.2 \pm 5.6$ & $3665 \pm 153$ & $2729 \pm 24$ & $6639 \pm 161$ \\*
 & EDHNN & \our & $20\,000$ & $25$ & $263.7 \pm 3.5$ & $4169 \pm 1019$ & $3229 \pm 18$ & $7662 \pm 1010$ \\*
 & UniGNN & \our & $15\,000$ & $20$ & $245.2 \pm 5.6$ & $5170 \pm 195$ & $2691 \pm 10$ & $8106 \pm 202$ \\*
OGBN Products & CWN & \our & $15\,000$ & $20$ & $245.2 \pm 5.6$ & $234389 \pm 1997$ & $17057 \pm 140$ & $251691 \pm 2042$ \\*
 & TopoTune & \our & $15\,000$ & $25$ & $245.2 \pm 5.6$ & $207711 \pm 63892$ & $25215 \pm 398$ & $233172 \pm 63530$ \\*
 & SCN & \our & $15\,000$ & $20$ & $245.2 \pm 5.6$ & $275955 \pm 3781$ & $23405 \pm 1106$ & $299605 \pm 3287$ \\*
 & SCCNN & \our & $15\,000$ & $20$ & $245.2 \pm 5.6$ & $345176 \pm 3392$ & $19851 \pm 491$ & $365272 \pm 3407$ \\
\end{longtable}
}

For configuration $i$, let $\bar t_i^{\mathrm{FG}}$ and $\bar t_i^{\mathrm{P}}$ denote the mean total runtimes under full-graph and partitioned execution reported in Table~\ref{tab:tdl_end_to_end_runtime}. Its relative change is $\Delta_i=(\bar t_i^{\mathrm{P}}-\bar t_i^{\mathrm{FG}})/\bar t_i^{\mathrm{FG}}\times100\%$, with negative values indicating reductions. We calculate changes in individual runtime components analogously. The reported means and medians give each configuration equal weight within the relevant subset.

The preprocessing component is shorter under \our partitioned execution in $16$ comparisons, longer in four, and unchanged in one. Among the $16$ decreases, the mean and median reductions are $53.9\%$ and $53.1\%$. Among the four increases, the mean and median increases are $66.0\%$ and $36.7\%$. These partitioned preprocessing measurements include the one-off graph-partitioning cost. Training time is lower in $12$ comparisons by $60.2\%$ on average (median $63.3\%$) and higher in nine by $203.9\%$ on average (median $138.5\%$). Mean total runtime is lower for all matched GCN and hypergraph configurations, higher for all cellular configurations, and lower for simplicial configurations only on Questions. This pattern describes the evaluated configurations rather than an intrinsic runtime ordering of the domains.

To complement these comparison-level summaries, Table~\ref{tab:tdl_pooled_runtime} pools mean total runtimes within each subset. For a set of matched configurations $\mathcal I$, the pooled relative change is
{
\begin{equation*}
\Delta_{\mathrm{pooled}}(\mathcal I)
=
\frac{\sum_{i\in\mathcal I}\bar t_i^{\mathrm{P}}
-\sum_{i\in\mathcal I}\bar t_i^{\mathrm{FG}}}
{\sum_{i\in\mathcal I}\bar t_i^{\mathrm{FG}}}
\times100\%.
\end{equation*}
}
This aggregate weights each configuration's relative change by its full-graph runtime. Across the matched benchmark suite, the additional runtime of slower configurations outweighs the savings of faster configurations. It describes the evaluated suite's combined runtime rather than the change for a typical configuration.

\begin{table}[!htbp]
\centering
\caption{Pooled total-runtime changes calculated from the configuration-level means in Table~\ref{tab:tdl_end_to_end_runtime}. Each sum combines the five-seed mean total runtime of the configurations in the indicated subset. Subsets are defined by the direction of the total-runtime change. Partitioned-only configurations are excluded.}
\label{tab:tdl_pooled_runtime}
\small
\begin{tabular}{@{}lrrrr@{}}
\toprule
\textbf{Subset} & \textbf{Count} & \makecell{\textbf{Full-graph}\\\textbf{sum (s)}} & \makecell{\textbf{\our}\\\textbf{sum (s)}} & \makecell{\textbf{Pooled}\\\textbf{change}} \\
\midrule
Shorter partitioned runtime & 11 & $2\,582.9$ & $1\,174.0$ & $-54.5\%$ \\
Longer partitioned runtime & 10 & $3\,371.1$ & $7\,934.1$ & $+135.4\%$ \\
All matched configurations & 21 & $5\,954.0$ & $9\,108.1$ & $+53.0\%$ \\
\bottomrule
\end{tabular}
\end{table}

CWN on Cora Full illustrates why training time and total runtime must be distinguished. Under partitioned execution, mean training time, including periodic validation, decreases from \(909.6\) to \(776.2\) seconds, a reduction of \(14.7\%\). However, final evaluation increases from \(14.6\) to \(272.4\) seconds, resulting in a \(12.4\%\) increase in total runtime. In this configuration, the additional evaluation cost outweighs the training-time reduction.

Across the $21$ matched comparisons, final evaluation is longer under partitioned execution in every case, with a median partitioned-to-full-graph time ratio of $11.75\times$. These measurements include checkpoint validation and test inference, using ten-recombination ensemble mini-batch inference for \our. This measured difference includes both repeated evaluation and batch-construction costs. It is not an intrinsic lower bound for partitioned inference.

\paragraph{Validation caching convention.}

In these experiments, we do not cache lifted validation batches. We reconstruct and lift each validation batch online.
We include validation-batch construction costs in the training component for validation during fitting
and in the final-evaluation component for checkpoint validation.
Final evaluation includes checkpoint validation and test inference.

\subsubsection{Run-wide peak memory}
\label{subsec:tdl_scalability_run_wide_memory}

Unlike Table~\ref{tab:tdl_preprocessing_cost}, which isolates full-domain preprocessing, Table~\ref{tab:tdl_end_to_end_memory} and \ref{tab:tdl_partitioned_memory} report peak CPU and GPU memory over the complete execution of each full-graph and \our run. CPU memory is the maximum sampled RSS of the main Python process, while GPU memory is the maximum CUDA memory allocated by the process.


\begin{table}[h]
\centering
\caption{Run-wide peak CPU and GPU memory under full-graph and \our execution. Values are mean $\pm$ sample standard deviation in decimal GB across five seeds. CPU values are main-process RSS peaks. GPU values are CUDA allocated-memory peaks.}
\label{tab:tdl_end_to_end_memory}
\scriptsize
\setlength{\tabcolsep}{3.2pt}
\renewcommand{\arraystretch}{1.05}
\resizebox{\textwidth}{!}{%
\begin{tabular}{llrr*{4}{r}}
\toprule
& & & & \multicolumn{2}{c}{\textbf{Full-graph}} & \multicolumn{2}{c}{\textbf{\our}} \\
\cmidrule(lr){5-6}\cmidrule(lr){7-8}
\textbf{Dataset} & \textbf{Model} & $\boldsymbol{K}$ & $\boldsymbol{q}$ &
\makecell[c]{\textbf{Peak CPU}\\\textbf{memory (GB)}} &
\makecell[c]{\textbf{Peak GPU}\\\textbf{memory (GB)}} &
\makecell[c]{\textbf{Peak CPU}\\\textbf{memory (GB)}} &
\makecell[c]{\textbf{Peak GPU}\\\textbf{memory (GB)}} \\
\midrule
 & GCN & $32$ & $8$ & $2.96 \pm 0.01$ & $6.95 \pm 0.00$ & $3.64 \pm 0.03$ & $1.80 \pm 0.01$ \\
 & EDHNN & $32$ & $4$ & $3.68 \pm 0.02$ & $7.65 \pm 0.00$ & $3.64 \pm 0.03$ & $1.01 \pm 0.00$ \\
 & UniGNN & $32$ & $4$ & $3.71 \pm 0.01$ & $7.65 \pm 0.00$ & $3.64 \pm 0.03$ & $1.02 \pm 0.00$ \\
Cora Full & CWN & $32$ & $4$ & $7.63 \pm 0.01$ & $24.14 \pm 0.00$ & $3.89 \pm 0.07$ & $3.65 \pm 0.09$ \\
 & TopoTune & $64$ & $8$ & $7.61 \pm 0.03$ & $24.21 \pm 0.00$ & $3.98 \pm 0.12$ & $3.53 \pm 0.09$ \\
 & SCN & $32$ & $16$ & $8.47 \pm 0.02$ & $30.61 \pm 0.00$ & $8.33 \pm 0.06$ & $16.42 \pm 0.23$ \\
 & SCCNN & $32$ & $8$ & $8.47 \pm 0.03$ & $30.62 \pm 0.00$ & $5.74 \pm 0.05$ & $8.92 \pm 0.26$ \\
\midrule
 & GCN & $64$ & $4$ & $1.60 \pm 0.01$ & $0.47 \pm 0.00$ & $1.66 \pm 0.01$ & $0.05 \pm 0.00$ \\
 & EDHNN & $128$ & $4$ & $1.66 \pm 0.01$ & $0.93 \pm 0.00$ & $1.68 \pm 0.01$ & $0.05 \pm 0.00$ \\
 & UniGNN & $32$ & $8$ & $1.66 \pm 0.01$ & $0.36 \pm 0.00$ & $1.72 \pm 0.02$ & $0.10 \pm 0.00$ \\
Amazon Ratings & CWN & $64$ & $16$ & $2.01 \pm 0.01$ & $3.22 \pm 0.00$ & $2.00 \pm 0.02$ & $0.86 \pm 0.01$ \\
 & TopoTune & $32$ & $8$ & $2.03 \pm 0.02$ & $2.68 \pm 0.00$ & $2.19 \pm 0.01$ & $1.05 \pm 0.01$ \\
 & SCN & $32$ & $8$ & $2.11 \pm 0.00$ & $2.10 \pm 0.00$ & $2.06 \pm 0.04$ & $0.65 \pm 0.00$ \\
 & SCCNN & $32$ & $4$ & $2.13 \pm 0.01$ & $4.38 \pm 0.00$ & $2.16 \pm 0.02$ & $0.62 \pm 0.00$ \\
\midrule
 & GCN & $500$ & $20$ & $1.70 \pm 0.01$ & $0.92 \pm 0.00$ & $1.69 \pm 0.02$ & $0.05 \pm 0.00$ \\
 & EDHNN & $500$ & $50$ & $1.76 \pm 0.01$ & $1.55 \pm 0.00$ & $1.73 \pm 0.01$ & $0.08 \pm 0.00$ \\
 & UniGNN & $1000$ & 40 & $1.76 \pm 0.01$ & $0.69 \pm 0.00$ & $1.72 \pm 0.01$ & $0.05 \pm 0.00$ \\
Questions & CWN & $1000$ & $40$ & $2.25 \pm 0.03$ & $4.56 \pm 0.00$ & $1.73 \pm 0.01$ & $0.06 \pm 0.00$ \\
 & TopoTune & $500$ & $50$ & $2.27 \pm 0.01$ & $4.52 \pm 0.00$ & $1.87 \pm 0.02$ & $0.33 \pm 0.01$ \\
 & SCN & $500$ & $50$ & $2.31 \pm 0.02$ & $3.40 \pm 0.00$ & $1.81 \pm 0.02$ & $0.14 \pm 0.00$ \\
 & SCCNN & $500$ & $30$ & $2.30 \pm 0.01$ & $6.10 \pm 0.00$ & $1.88 \pm 0.01$ & $0.18 \pm 0.00$ \\
\bottomrule
\end{tabular}}
\end{table}

\begin{table}[!htbp]
\centering
\setlength{\abovecaptionskip}{4pt}
\caption{\our run-wide peak memory on Coauthor Physics, Reddit, and OGBN Products, using the measurement definitions of Table~\ref{tab:tdl_end_to_end_memory}. CPU values are main-process RSS peaks. GPU values are CUDA allocated-memory peaks. Both are reported in decimal GB.}
\label{tab:tdl_partitioned_memory}
\footnotesize
\setlength{\tabcolsep}{4pt}
\renewcommand{\arraystretch}{1.02}
\begin{tabular*}{\textwidth}{@{\extracolsep{\fill}}llrr*{2}{r}@{}}
\toprule
\textbf{Dataset} & \textbf{Model} & $\boldsymbol{K}$ & $\boldsymbol{q}$ & \makecell{\textbf{Peak CPU}\\\textbf{memory (GB)}} & \makecell{\textbf{Peak GPU}\\\textbf{memory (GB)}} \\
\midrule
 & GCN & $500$ & $20$ & $6.68 \pm 0.03$ & $0.51 \pm 0.00$ \\
 & EDHNN & $2\,000$ & $30$ & $6.07 \pm 0.05$ & $0.24 \pm 0.00$ \\
 & UniGNN & $2\,000$ & $20$ & $6.07 \pm 0.05$ & $0.16 \pm 0.00$ \\
Coauthor Physics & CWN & $2\,000$ & $20$ & $6.07 \pm 0.05$ & $0.38 \pm 0.01$ \\
 & TopoTune & $2\,000$ & $20$ & $6.07 \pm 0.05$ & $0.38 \pm 0.01$ \\
 & SCN & $2\,000$ & $20$ & $6.07 \pm 0.05$ & $0.67 \pm 0.05$ \\
 & SCCNN & $2\,000$ & $20$ & $6.07 \pm 0.05$ & $0.72 \pm 0.04$ \\
\midrule
 & GCN & $7\,000$ & $25$ & $11.81 \pm 0.02$ & $0.07 \pm 0.00$ \\
 & EDHNN & $7\,000$ & $20$ & $11.81 \pm 0.02$ & $0.13 \pm 0.01$ \\
 & UniGNN & $7\,000$ & $25$ & $11.81 \pm 0.02$ & $0.04 \pm 0.00$ \\
Reddit & CWN & $7\,000$ & $20$ & $11.81 \pm 0.02$ & $0.51 \pm 0.04$ \\
 & TopoTune & $7\,000$ & $20$ & $11.81 \pm 0.02$ & $0.57 \pm 0.04$ \\
 & SCN & $7\,000$ & $25$ & $11.81 \pm 0.02$ & $6.84 \pm 1.91$ \\
 & SCCNN & $7\,000$ & $20$ & $11.81 \pm 0.02$ & $7.45 \pm 2.47$ \\
\midrule
 & GCN & $15\,000$ & $25$ & $12.88 \pm 0.02$ & $0.19 \pm 0.00$ \\
 & EDHNN & $20\,000$ & $25$ & $13.19 \pm 0.02$ & $0.38 \pm 0.01$ \\
 & UniGNN & $15\,000$ & $20$ & $12.87 \pm 0.02$ & $0.05 \pm 0.00$ \\
OGBN Products & CWN & $15\,000$ & $20$ & $12.87 \pm 0.02$ & $1.45 \pm 0.02$ \\
 & TopoTune & $15\,000$ & $25$ & $12.87 \pm 0.02$ & $1.46 \pm 0.02$ \\
 & SCN & $15\,000$ & $20$ & $12.87 \pm 0.02$ & $5.42 \pm 0.52$ \\
 & SCCNN & $15\,000$ & $20$ & $12.87 \pm 0.02$ & $11.94 \pm 1.10$ \\
\bottomrule
\end{tabular*}
\end{table}

Across the 21 matched full-graph and partitioned comparisons in Table~\ref{tab:tdl_end_to_end_memory}, run-wide peak CPU memory is lower under partitioned execution in 15 comparisons by $14.8\%$ on average (median $2.4\%$) and higher in six by $6.8\%$ on average (median $3.7\%$). CPU memory therefore varies by configuration rather than exhibiting the uniform decrease observed for GPU memory. Peak GPU memory is lower in all 21 comparisons by $83.2\%$ on average (median $86.7\%$).

Table~\ref{tab:tdl_partitioned_memory} extends the measurements to Coauthor Physics, Reddit, and OGBN Products, reporting absolute memory requirements where matched full-graph results are unavailable. The larger-dataset results show that large input graphs can be processed with modest GPU memory requirements under \our execution. On Reddit, mean peak GPU memory ranges from $0.04$~GB for UniGNN to $7.45$~GB for SCCNN. On OGBN Products, TopoTune and SCN have mean peaks of $1.46$ and $5.42$~GB, respectively. Memory requirements therefore depend strongly on the sampled graph context and model configuration, rather than input graph size alone.

\clearpage
\subsection{Structural exposure and statistical efficiency}
\label{subsec:tdl_scalability_statistical}

Memory savings are only one consequence of batch-before-lift execution. Because each update is formed from an induced graph batch, the method also changes which full-graph structures can occur together and how often the model encounters them. We refer to this availability across cluster recombinations as structural exposure.

An epoch under \our visits all active training clusters once and groups them into mini-batches of at most \(q\) clusters. Its number of optimization steps is therefore the number of active clusters divided into groups of size at most \(q\). The theory in Appendix~\ref{sec:struct_recon} considers the simplified case in which all \(K\) clusters are active and \(q\mid K\), giving exactly \(M=K/q\) steps. Increasing \(q\) uses more memory per mini-batch but can expose more cross-cluster higher-order structures. Consequently, equal epoch counts need not represent equal update counts, wall-clock time, or structural information.

The recovery theory describes when structures and messages are available. It does not establish convergence to the model trained on the full lifted domain. We therefore assess statistical efficiency through validation trajectories, optimization steps, and test performance. The primary validation metric selects the checkpoint for each run. Test predictions from that checkpoint use full-graph inference for full-graph setting and ten-recombination ensemble inference for \our setting.

Table~\ref{tab:tdl_statistical_efficiency} and \ref{tab:tdl_partitioned_statistical_efficiency} summarize each validation trajectory by the checkpoint attaining the best primary validation metric. For each run, we associate this checkpoint with its completed optimizer-step count and its final test result.

\begin{table*}[!htbp]
\centering
\caption{Statistical-efficiency for full-graph and \our training. $E^{*}$ is the epoch selected by the validation metric, and $U^{*}$ is the number of optimizer updates completed at that checkpoint. Validation and test results are reported. The metric is accuracy for Cora Full and Amazon Ratings and AUROC for Questions. Full-graph uses one full-domain test pass, whereas \our uses the ten-pass ensemble protocol. Each entry is the mean $\pm$ sample standard deviation across five seeds.}
\label{tab:tdl_statistical_efficiency}
\scriptsize
\setlength{\tabcolsep}{2.4pt}
\renewcommand{\arraystretch}{1.04}
\resizebox{\textwidth}{!}{%
\begin{tabular}{llrr*{4}{r}*{4}{r}}
\toprule
& & & & \multicolumn{4}{c}{\textbf{Full-graph}} & \multicolumn{4}{c}{\textbf{\our}} \\
\cmidrule(lr){5-8}\cmidrule(lr){9-12}
\textbf{Dataset} & \textbf{Model} & $\boldsymbol{K}$ & $\boldsymbol{q}$ & $\boldsymbol{E^{*}}$ & $\boldsymbol{U^{*}}$ & \textbf{Val.} & \textbf{Test} & $\boldsymbol{E^{*}}$ & $\boldsymbol{U^{*}}$ & \textbf{Val.} & \textbf{Test} \\
\midrule
\multirow{7}{*}{Cora Full} & GCN & 32 & 8 & $94 \pm 30$ & $94 \pm 30$ & $70.02 \pm 1.23$ & $70.18 \pm 0.42$ & $44 \pm 10$ & $176 \pm 39$ & $68.42 \pm 1.63$ & $69.19 \pm 0.66$ \\
 & EDHNN & $32$ & $4$ & $82 \pm 29$ & $82 \pm 29$ & $70.07 \pm 0.97$ & $69.74 \pm 0.89$ & $59 \pm 19$ & $472 \pm 153$ & $66.89 \pm 1.15$ & $67.06 \pm 1.19$ \\
 & UniGNN & $32$ & $4$ & $58 \pm 18$ & $58 \pm 18$ & $68.92 \pm 1.32$ & $68.44 \pm 0.79$ & $44 \pm 10$ & $352 \pm 82$ & $67.43 \pm 1.23$ & $67.61 \pm 0.56$ \\
 & CWN & $32$ & $4$ & $83 \pm 12$ & $83 \pm 12$ & $61.29 \pm 0.82$ & $60.50 \pm 0.60$ & $57 \pm 3$ & $456 \pm 22$ & $62.37 \pm 1.04$ & $62.66 \pm 1.11$ \\
 & TopoTune & $64$ & $8$ & $22 \pm 6$ & $22 \pm 6$ & $46.87 \pm 0.62$ & $46.40 \pm 0.47$ & $67 \pm 21$ & $536 \pm 171$ & $55.92 \pm 0.92$ & $56.45 \pm 1.42$ \\
 & SCN & $32$ & $16$ & $23 \pm 10$ & $23 \pm 10$ & $70.81 \pm 0.53$ & $70.45 \pm 0.35$ & $14 \pm 2$ & $28 \pm 5$ & $70.01 \pm 0.82$ & $70.32 \pm 0.46$ \\
 & SCCNN & $32$ & $8$ & $21 \pm 2$ & $21 \pm 2$ & $71.19 \pm 0.71$ & $70.81 \pm 0.58$ & $19 \pm 11$ & $76 \pm 43$ & $69.52 \pm 0.95$ & $70.20 \pm 0.29$ \\
\midrule
\multirow{7}{*}{\makecell[c]{Amazon\\Ratings}} & GCN & 64 & 4 & $120 \pm 29$ & $120 \pm 29$ & $49.36 \pm 1.39$ & $48.29 \pm 1.31$ & $75 \pm 31$ & $1200 \pm 493$ & $49.70 \pm 1.03$ & $49.62 \pm 0.96$ \\
 & EDHNN & $128$ & $4$ & $181 \pm 58$ & $181 \pm 58$ & $48.86 \pm 1.08$ & $48.41 \pm 1.18$ & $134 \pm 29$ & $4288 \pm 936$ & $51.71 \pm 1.21$ & $51.41 \pm 0.72$ \\
 & UniGNN & $32$ & $8$ & $151 \pm 82$ & $151 \pm 82$ & $47.35 \pm 1.32$ & $47.05 \pm 1.76$ & $69 \pm 20$ & $276 \pm 82$ & $47.45 \pm 0.95$ & $47.24 \pm 0.70$ \\
 & CWN & $64$ & $16$ & $106 \pm 20$ & $106 \pm 20$ & $45.13 \pm 1.02$ & $44.07 \pm 0.64$ & $78 \pm 14$ & $312 \pm 58$ & $46.63 \pm 0.97$ & $46.25 \pm 0.67$ \\
 & TopoTune & $32$ & $8$ & $32 \pm 10$ & $32 \pm 10$ & $43.41 \pm 1.00$ & $43.18 \pm 1.03$ & $21 \pm 9$ & $84 \pm 36$ & $43.60 \pm 0.85$ & $43.27 \pm 0.68$ \\
 & SCN & $32$ & $8$ & $174 \pm 37$ & $174 \pm 37$ & $50.36 \pm 0.74$ & $50.08 \pm 0.36$ & $78 \pm 34$ & $312 \pm 137$ & $50.99 \pm 1.01$ & $50.65 \pm 0.64$ \\
 & SCCNN & $32$ & $4$ & $141 \pm 37$ & $141 \pm 37$ & $51.33 \pm 0.83$ & $51.20 \pm 1.10$ & $67 \pm 10$ & $536 \pm 83$ & $51.45 \pm 0.71$ & $51.24 \pm 0.79$ \\
\midrule
\multirow{7}{*}{Questions} & GCN & $500$ & $20$ & $254 \pm 50$ & $254 \pm 50$ & $76.32 \pm 1.25$ & $76.69 \pm 1.18$ & $98 \pm 30$ & $2450 \pm 748$ & $72.85 \pm 1.66$ & $73.83 \pm 2.01$ \\
 & EDHNN & $500$ & $50$ & $211 \pm 52$ & $211 \pm 52$ & $76.54 \pm 1.27$ & $75.73 \pm 0.92$ & $42 \pm 20$ & $420 \pm 196$ & $75.78 \pm 1.01$ & $76.23 \pm 1.05$ \\
 & UniGNN & $1000$ & $40$ & $300 \pm 0$ & $300 \pm 0$ & $70.41 \pm 1.41$ & $69.44 \pm 1.68$ & $27 \pm 28$ & $675 \pm 694$ & $71.25 \pm 1.58$ & $71.09 \pm 0.98$ \\
 & CWN & $1000$ & $40$ & $26 \pm 18$ & $26 \pm 18$ & $68.56 \pm 2.56$ & $68.21 \pm 3.48$ & $27 \pm 14$ & $675 \pm 360$ & $71.62 \pm 0.99$ & $70.20 \pm 0.88$ \\
 & TopoTune & $500$ & $50$ & $95 \pm 17$ & $95 \pm 17$ & $69.98 \pm 1.88$ & $71.53 \pm 1.73$ & $54 \pm 22$ & $540 \pm 222$ & $71.97 \pm 1.90$ & $73.61 \pm 1.88$ \\
 & SCN & $500$ & $50$ & $236 \pm 90$ & $236 \pm 90$ & $73.96 \pm 1.29$ & $74.77 \pm 1.34$ & $11 \pm 2$ & $110 \pm 22$ & $71.96 \pm 1.26$ & $73.04 \pm 1.84$ \\
 & SCCNN & $500$ & $30$ & $56 \pm 8$ & $56 \pm 8$ & $76.26 \pm 1.15$ & $77.14 \pm 0.99$ & $6 \pm 2$ & $102 \pm 38$ & $73.13 \pm 1.50$ & $74.10 \pm 1.85$ \\
\bottomrule
\end{tabular}%
}
\end{table*}

\begin{table}[!htbp]
\centering
\setlength{\abovecaptionskip}{4pt}
\caption{\our statistical efficiency on Coauthor Physics, Reddit, and OGBN Products, using the checkpoint and inference definitions of Table~\ref{tab:tdl_statistical_efficiency}. $E^{*}$ is the selected epoch, $U^{*}$ the completed optimizer updates, and validation and test scores are accuracies (\%).}
\label{tab:tdl_partitioned_statistical_efficiency}
\footnotesize
\setlength{\tabcolsep}{4pt}
\renewcommand{\arraystretch}{1.02}
\begin{tabular*}{\textwidth}{@{\extracolsep{\fill}}llrr*{4}{r}@{}}
\toprule
\textbf{Dataset} & \textbf{Model} & $\boldsymbol{K}$ & $\boldsymbol{q}$ & $\boldsymbol{E^{*}}$ & $\boldsymbol{U^{*}}$ & \textbf{Val.} & \textbf{Test} \\
\midrule
 & GCN & $500$ & $20$ & $39 \pm 27$ & $975 \pm 681$ & $96.32 \pm 0.12$ & $96.68 \pm 0.08$ \\
 & EDHNN & $2\,000$ & $30$ & $58 \pm 20$ & $3886 \pm 1332$ & $96.13 \pm 0.11$ & $96.31 \pm 0.32$ \\
 & UniGNN & $2\,000$ & $20$ & $68 \pm 17$ & $6800 \pm 1681$ & $94.66 \pm 0.19$ & $94.97 \pm 0.37$ \\
Coauthor Physics & CWN & $2\,000$ & $20$ & $34 \pm 24$ & $3400 \pm 2434$ & $95.69 \pm 0.10$ & $95.80 \pm 0.24$ \\
 & TopoTune & $2\,000$ & $20$ & $18 \pm 7$ & $1800 \pm 671$ & $94.70 \pm 0.27$ & $95.22 \pm 0.20$ \\
 & SCN & $2\,000$ & $20$ & $32 \pm 20$ & $3200 \pm 1956$ & $95.84 \pm 0.20$ & $96.22 \pm 0.22$ \\
 & SCCNN & $2\,000$ & $20$ & $60 \pm 35$ & $6000 \pm 3500$ & $95.84 \pm 0.31$ & $96.00 \pm 0.25$ \\
\midrule
 & GCN & $7\,000$ & $25$ & $62 \pm 8$ & $17360 \pm 2123$ & $92.42 \pm 0.15$ & $92.62 \pm 0.10$ \\
 & EDHNN & $7\,000$ & $20$ & $134 \pm 40$ & $46900 \pm 14076$ & $91.73 \pm 0.75$ & $92.17 \pm 0.73$ \\
 & UniGNN & $7\,000$ & $25$ & $108 \pm 37$ & $30240 \pm 10458$ & $90.73 \pm 0.11$ & $91.28 \pm 0.16$ \\
Reddit & CWN & $7\,000$ & $20$ & $113 \pm 8$ & $39550 \pm 2654$ & $91.02 \pm 0.19$ & $91.29 \pm 0.18$ \\
 & TopoTune & $7\,000$ & $20$ & $33 \pm 8$ & $11550 \pm 2928$ & $89.12 \pm 0.14$ & $89.55 \pm 0.09$ \\
 & SCN & $7\,000$ & $25$ & $56 \pm 2$ & $15680 \pm 626$ & $91.44 \pm 0.14$ & $91.80 \pm 0.13$ \\
 & SCCNN & $7\,000$ & $20$ & $65 \pm 29$ & $22750 \pm 9977$ & $92.16 \pm 0.19$ & $92.45 \pm 0.25$ \\
\midrule
 & GCN & $15\,000$ & $25$ & $128 \pm 3$ & $76800 \pm 1643$ & $85.50 \pm 0.08$ & $85.81 \pm 0.04$ \\
 & EDHNN & $20\,000$ & $25$ & $95 \pm 28$ & $76000 \pm 22627$ & $86.27 \pm 0.28$ & $86.51 \pm 0.25$ \\
 & UniGNN & $15\,000$ & $20$ & $125 \pm 11$ & $93750 \pm 8385$ & $85.61 \pm 0.08$ & $85.87 \pm 0.07$ \\
OGBN Products & CWN & $15\,000$ & $20$ & $128 \pm 3$ & $96000 \pm 2054$ & $83.43 \pm 0.09$ & $83.93 \pm 0.06$ \\
 & TopoTune & $15\,000$ & $25$ & $61 \pm 22$ & $36600 \pm 13315$ & $83.09 \pm 0.48$ & $83.79 \pm 0.55$ \\
 & SCN & $15\,000$ & $20$ & $127 \pm 3$ & $95250 \pm 2054$ & $86.57 \pm 0.06$ & $87.03 \pm 0.05$ \\
 & SCCNN & $15\,000$ & $20$ & $167 \pm 3$ & $125250 \pm 2054$ & $88.14 \pm 0.04$ & $88.63 \pm 0.05$ \\
\bottomrule
\end{tabular*}
\end{table}

Based on the reported means, partitioned training reaches the selected checkpoint in fewer epochs in 19 of 21 comparisons, but uses more optimizer updates in 20. Because a partitioned epoch contains multiple updates, epoch counts alone do not establish greater optimization efficiency. They should be interpreted alongside optimizer updates, elapsed time, and predictive performance.

\newpage
\subsection{Solution paths and practical guidance}
\label{subsec:tdl_scalability_solution_paths}

The measurements above show a phase-specific trade-off rather than a uniform reduction in computational cost. Partitioned execution lowers peak GPU memory in all 21 direct comparisons, while preprocessing time, training time, and run-wide CPU memory change in both directions across configurations. Final evaluation is longer under the evaluated ten-recombination protocol, and statistical outcomes vary by dataset and model. A suitable remedy must therefore be matched to the stage that constitutes the binding bottleneck.

\paragraph{Alternative solution paths.}

Different remedies address different stages of the workflow. Restricting neighbourhood radius, simplex dimension, or retained cycle length can reduce the number of retained higher-order elements or incidences, but may discard useful structure. For the evaluated neighbourhood lifting, reducing the radius changes membership volume rather than the fixed count of one indexed hyperedge per node. For the evaluated cycle lifting, the length cap filters retained cells after a complete cycle basis is computed, so it need not reduce basis-construction time or temporary memory. Reducing feature width or numerical precision lowers feature storage. Computing higher-order features when needed can avoid retaining their complete matrices, but introduces repeated computation and does not by itself remove the need to provide the required connectivity.

When full-domain construction exceeds CPU memory, batch-before-lift, distributed construction, or out-of-core processing address the construction stage in different ways. Distributed construction spreads work and allocations across workers, while out-of-core processing moves part of the representation to secondary storage. Neither reduces the representation itself. When the domain can be constructed but model execution exceeds GPU memory, sampling within the higher-order domain becomes a viable alternative. When repeated local lifting dominates runtime, caching a fixed set of lifted batches can reduce reconstruction work at the cost of storage and a fixed collection of cluster combinations.

Published lift-then-batch methods provide concrete examples in the hypergraph setting. Telyatnikov et al.~\citep{telyatnikov2025hypergraph} sample hyperedges and then nodes within the sampled hyperedges, while FastHGNN~\citep{fasthgnn} performs layer-wise node and hyperedge sampling. These methods reduce model-execution cost after a hypergraph is available, but do not address the cost of constructing a lifted domain from a source graph. They also do not directly specify how to sample simplicial or cellular domains while preserving the face or boundary relations used by those models.

\paragraph{Batch-before-lift: scope and trade-offs.}

Batch-before-lift instead limits how much of the source graph is presented to the lifting and the TNN at one time. It is operationally applicable to graph-to-higher-order liftings that can be evaluated on induced subgraphs. This condition does not guarantee that locally lifted structures coincide with restrictions of the full-domain lift. For non-compatible liftings, local recomputation can alter the selected structures in addition to omitting structures whose supporting nodes do not appear together. Its principal costs are therefore repeated local lifting and reduced or altered structural exposure. The number of clusters grouped into a mini-batch controls part of this trade-off: larger batches can expose more cross-cluster structure, but require more CPU and GPU memory. Lift-then-batch methods address a complementary setting, particularly when data are observed natively as higher-order domains, but require sampling rules that preserve the relations used by the target hypergraph, simplicial, or cellular model. Standardized, intrinsically higher-order benchmark datasets remain comparatively scarce~\citep{position,rieck2026have}. Our implementation realizes batch-before-lift through \tb{}'s existing dataset, lifting, and model interfaces.

\paragraph{Reporting guidance.}

At minimum, scalable graph-supported TDL studies should report lifting-induced structural growth, phase-wise time and peak memory with the device and measurement scope identified, end-to-end runtime, and predictive results. Where applicable, they should also report partition and mini-batch parameters, caching and inference protocols, the number of cluster recombinations, and optimizer updates. Together, these quantities distinguish whether a method reduces the representation itself, constructs it incrementally or on demand, limits only the model input, or trades repeated computation against storage. The realized memory savings and structural exposure depend on the dataset and lifting, as well as method-specific choices such as the partition, total number of clusters, mini-batch grouping, and number of recombinations. Appendix~\ref{sec:struct_recon} next states the compatibility assumptions and formalizes the structural-exposure side of this trade-off.
\section{Structural reconstruction}
\label{sec:struct_recon}

This appendix analyzes which full-graph structures \our workflow recovers, how often it exposes them, and when their direct-incidence messages are available. The graph partition is fixed, and the only randomness comes from independently reshuffling its cluster indices at each epoch. Throughout the paper, $K$ denotes the total number of clusters in this partition. The implementation retains only clusters containing at least one target node from the current split (Sec.~\ref{subsec:clustered_mini_batch_lifting}). For example, a cluster containing validation or test nodes but no training nodes is inactive for training. The number of active clusters can therefore be smaller than $K$. The analysis below assumes that all $K$ partition clusters participate in recombination. For notational simplicity, we assume $q\mid K$, so every mini-batch contains exactly $q$ clusters. This removes a possible smaller remainder mini-batch without changing the qualitative conclusions. The following results state the compatibility assumptions required for structural recovery, repeated exposure, and direct-message availability.

The recovery analysis assumes compatibility with induced subgraphs, as formalized in Definition~\ref{def:lifting_comp}. This condition holds for the identity lifting, the evaluated $k$-hop neighbourhood liftings, and the clique-complex lifting. It can fail for the cycle-basis lifting because restricting the graph can change the selected basis. The cycle-basis recovery probabilities therefore require stable basis membership. Appendix~\ref{subsec:exp_cora} examines this condition empirically on Cora Full dataset.

\subsection{Setup and notation}
\label{subsec:setup_notation}

Because features do not affect structural recovery, let $\graph=(\vertices,\edges)$ denote the underlying undirected graph. Let
\[
    \Pi=\{P_{1},\ldots,P_{K}\}
\]
be the fixed partition of $\vertices$ into $K$ clusters considered in this analysis. We write $[K]=\{1,\ldots,K\}$ and use $q\leq K$ clusters per mini-batch. Under the divisibility assumption above, each epoch contains $M=K/q$ mini-batches.

At each epoch $t\in\{1,\ldots,T\}$, we independently draw a uniformly random permutation of $[K]$. We divide the permutation into the mini-batch index sets
\[
    \mathcal{Z}^{(t)}=\{Z_{1}^{(t)},\ldots,Z_{M}^{(t)}\},
    \qquad
    Z_i^{(t)}\subseteq [K],
    \qquad
    |Z_i^{(t)}|=q.
\]
For each mini-batch index set $Z_{i}^{(t)}$, we define the induced mini-batch graph
\[
    \graph_i^{(t)}=\big(\vertices_i^{(t)},\edges_i^{(t)}\big),
    \qquad
    \vertices_i^{(t)}=\!\!\bigcup_{k\in Z_{i}^{(t)}}\!\!P_{k},
    \qquad
    \edges_i^{(t)}=\{(u,v)\in\edges:u,v\in\vertices_i^{(t)}\}.
\]
Thus, $\graph_i^{(t)}$ is the graph induced by the mini-batch node set $B=\vertices_i^{(t)}$.

Let $\structlifting$ be a deterministic structural lifting from graphs to a chosen topological domain. Let $\mathcal{S}$ denote the set of possible topological structures. For any graph $\graph$, we write $\mathrm{Str}(\structlifting(\graph))\subseteq\mathcal{S}$ for the structures in the resulting lifted domain. The set observed in epoch $t$ is
\begin{equation}
    S^{(t)}=\bigcup_{i=1}^{M}\mathrm{Str}\left(\structlifting\left(\graph_{i}^{(t)}\right)\right)~.
    \label{eq:s_one_epoch}
\end{equation}
The cumulative observed set of structures after $T$ epochs is $\widehat{S}_{T}=\displaystyle\bigcup_{t=1}^{T}S^{(t)}$.

The global structure set is $S^{\ast}=\mathrm{Str}\left(\structlifting\left(\graph\right)\right)$, obtained by lifting the full graph. We assume that $\graph$, and hence $S^\ast$, is finite. Mini-batch lifting may also produce structures outside $S^{\ast}$. We therefore define the observed global structures as
\[
    \widehat{S}_{T}^{\ast}=\widehat{S}_{T}\cap S^{\ast}.
\]
Thus, the recovery analysis only tracks global structures and ignores mini-batch-local artifacts.

For $s\in S^{\ast}$, given its vertex set $\vertices_s\subseteq\vertices$, we define its partition-index support as
\begin{equation*}
    I_s=\{i\in[K]:P_i\cap\vertices_s\neq\emptyset\}
\end{equation*}
and its cluster span $c_s=|I_s|$ as the number of clusters that $s$ intersects. Since every structure contains at least one vertex, we have $c_s\geq1$. A structure with $c_s=1$ is fully contained in one cluster.
\begin{definition}[Lifting compatibility]
    The lifting map is compatible with induced subgraphs if the following equivalence holds for every $s\in S^\ast$ and every $Z_i^{(t)}$. The global structure $s$ appears in the lifted mini-batch exactly when that mini-batch contains all its supporting cluster indices:
    \begin{equation*}
        s\in\mathrm{Str}\left(\structlifting\left(\graph_i^{(t)}\right)\right)
        \iff
        I_s\subseteq Z_{i}^{(t)}~.
    \end{equation*}
    \label{def:lifting_comp}
\end{definition}
\subsection{Per-epoch recovery probability}
\label{subsec:per_epoch_recovery_prob}

We first compute the probability that a fixed global structure $s\in S^\ast$ is recovered during one epoch. By Definition~\ref{def:lifting_comp}, all indices in $I_s$ must occur in the same mini-batch. Thus, recovery depends on the cluster span $c_s=|I_s|$. It is impossible when $c_s>q$ and becomes a combinatorial co-occurrence event when $c_s\leq q$.

Let
\[
    p_s^{(t)}=\mathbb{P}\left(s\in S^{(t)}\right)
\]
denote this per-epoch recovery probability. Since the reshuffling distribution is identical across epochs, $p_s^{(t)}$ does not depend on $t$, and we write $p_s$.
\begin{lemma}[Per-epoch recovery probability]
    For any global structure $s\in S^\ast$,
    \begin{gather*}
    p_{s}=
    \begin{cases}
        \displaystyle
        \frac{\binom{q-1}{c_s-1}}{\binom{K-1}{c_s-1}},
        & c_s\leq q,\\[1.1em]
        0,
        & c_s>q.
    \end{cases}
    \end{gather*}
    \label{lemma:per_epoch_recovery}
\end{lemma}
\begin{proof}
    If $c_s>q$, no mini-batch can contain all $c_s$ supporting cluster indices, so $p_s=0$. Now suppose $c_s\leq q$. Fix $k\in I_s$, and let $Z^{(t)}(k)$ be its unique mini-batch in epoch $t$. Under a uniform shuffle, the other $q-1$ members of $Z^{(t)}(k)$ form a uniformly random subset of $[K]\setminus\{k\}$. Recovery requires this subset to contain the fixed set $I_s\setminus\{k\}$. Therefore,
    \begin{equation*}
        p_{s}=\frac{\binom{K-c_s}{q-c_s}}{\binom{K-1}{q-1}} = \frac{\binom{q-1}{c_s-1}}{\binom{K-1}{c_s-1}}~.
    \end{equation*}
\end{proof}
Structures contained in one cluster are recovered every epoch. When $q<K$, the recovery probability decreases with cluster span, and structures with $c_s>q$ remain unobservable.

\subsection{Repeated structural exposure}
\label{subsec:repeated_structural_exposure}

Per-epoch recovery determines whether a structure is available during one epoch. To track repeated availability, for each $s\in S^\ast$ we define the exposure indicator
\begin{equation}
    X_s^{(t)}=\mathbf{1}\{s\in S^{(t)}\}~,
    \label{eq:exposure_indicator}
\end{equation}
where $\mathbf{1}\{\cdot\}$ denotes the indicator function, and the number of epochs in which $s$ has been observed by epoch $T$,
\begin{equation}
    N_s(T)=\sum_{t=1}^{T}X_s^{(t)}~.
    \label{eq:exposure_count}
\end{equation}
The mini-batches partition the $K$ clusters in each epoch. A compatible global structure can therefore occur in at most one mini-batch per epoch. Thus, $N_s(T)$ counts both exposure epochs and the corresponding optimization steps containing $s$.

\begin{proposition}[Repeated structural exposure]
    \label{prop:repeated_structural_exposure}
    Under independent epoch reshuffling and lifting compatibility, for every fixed $s\in S^\ast$,
    \begin{equation}
        X_s^{(t)}\overset{\mathrm{i.i.d.}}{\sim}\operatorname{Bernoulli}(p_s),
        \qquad
        N_s(T)\sim\operatorname{Binomial}(T,p_s).
        \label{eq:binomial_exposure}
    \end{equation}
    Consequently,
    \begin{equation}
        \mathbb{E}[N_s(T)]=Tp_s,
        \qquad
        \operatorname{Var}[N_s(T)]=Tp_s(1-p_s),
        \label{eq:exposure_moments}
    \end{equation}
    and
    \begin{equation}
        \frac{N_s(T)}{T}\longrightarrow p_s
        \qquad\text{almost surely as }T\to\infty.
        \label{eq:exposure_slln}
    \end{equation}
\end{proposition}

\begin{proof}
    By Lemma~\ref{lemma:per_epoch_recovery}, each $X_s^{(t)}$ is Bernoulli with success probability $p_s$. It is a function of the permutation drawn in epoch $t$. Independent and identically distributed epoch permutations therefore produce independent and identically distributed exposure indicators. Their sum is binomial, giving Eq.~\eqref{eq:binomial_exposure} and the moments in Eq.~\eqref{eq:exposure_moments}. Eq.~\eqref{eq:exposure_slln} follows from the strong law of large numbers.
\end{proof}

Proposition~\ref{prop:repeated_structural_exposure} gives $p_s$ two interpretations. It is the one-epoch recovery probability and the asymptotic fraction of epochs containing $s$. If $c_s=1$, then $N_s(T)=T$ almost surely. If $c_s>q$, then $N_s(T)=0$ almost surely. When $q<K$, structures with $1<c_s\leq q$ remain observable but occur less frequently than structures contained in one cluster.

Since every epoch contains $M=K/q$ optimization steps, Eq.~\eqref{eq:exposure_slln} also gives the asymptotic fraction of all training steps in which $s$ is present:
\begin{equation}
    \frac{N_s(T)}{MT}\longrightarrow\frac{p_s}{M}
    \qquad\text{almost surely as }T\to\infty.
    \label{eq:exposure_step_frequency}
\end{equation}
The distinction is important: $p_s$ is a per-epoch exposure probability, whereas $p_s/M$ is the corresponding frequency among optimization steps.

\begin{remark}[Dependence between structures]
    Proposition~\ref{prop:repeated_structural_exposure} concerns one fixed structure across epochs. Exposure indicators for different structures within an epoch are generally dependent because they share the same cluster grouping. None of the expectations below require independence between structures.
\end{remark}

\subsection{Cumulative recovery probability}
\label{subsec:cum_prob_conv}

Cumulative recovery records whether a global structure has been exposed at least once by epoch $T$. For every $s\in S^\ast$, this event is equivalent to
\[
    \{s\in\widehat{S}_{T}^{\ast}\}=\{N_s(T)\geq1\}.
\]
By Proposition~\ref{prop:repeated_structural_exposure}, the probability that $s$ remains unobserved after $T$ epochs is
\[
    \mathbb{P}\left(s\notin\widehat{S}_{T}^{\ast}\right)
    =
    \mathbb{P}\left(N_s(T)=0\right)
    =
    (1-p_s)^T.
\]
Hence,
\begin{equation}
    \mathbb{P}\left(s\in\widehat{S}_{T}^{\ast}\right)=1-(1-p_{s})^{T}~.
    \label{eq:cumul_prob}
\end{equation}
\begin{theorem}[Recovery in probability of the $q$-observable structure set]
    For every $s\in S^\ast$,
    \begin{equation*}
        \lim_{T\rightarrow\infty}\mathbb{P}\left(s\in\widehat{S}_{T}^{\ast}\right)
        =
        \begin{cases}
            1, & c_s\leq q,\\
            0, & c_s>q.
        \end{cases}
    \end{equation*}
    Consequently, with
    \begin{equation*}
        S_q^\ast=\{s\in S^\ast:c_s\leq q\}~,
    \end{equation*}
    we have $\mathbb{P}(\widehat{S}_{T}^{\ast}=S_q^\ast)\rightarrow1$. In particular, if $\displaystyle q\geq\max_{s\in S^\ast}c_s$, then $\mathbb{P}(\widehat{S}_{T}^{\ast}=S^\ast)\rightarrow1$.
    \label{theorem:cumulative_recovery}
\end{theorem}
\begin{proof}
    If $c_s>q$, then Lemma~\ref{lemma:per_epoch_recovery} gives $p_s=0$, and Eq.~\eqref{eq:cumul_prob} gives $\mathbb{P}(s\in\widehat{S}_{T}^{\ast})=0$ for all $T$. If $c_s\leq q$, then $p_s>0$, so $(1-p_s)^T\rightarrow0$ and Eq.~\eqref{eq:cumul_prob} gives $\mathbb{P}(s\in\widehat{S}_{T}^{\ast})\rightarrow1$.

    Structures in $S^\ast\setminus S_q^\ast$ are never observed with $q$ clusters per mini-batch. Thus, $\widehat{S}_{T}^{\ast}\subseteq S_q^\ast$ for every $T$. Since $S_q^\ast$ is finite, the union bound gives
    \begin{equation*}
        \mathbb{P}\left(\widehat{S}_{T}^{\ast}\neq S_q^\ast\right)=\mathbb{P}\big(\bigcup_{s \in S_{q}^{\ast}}\{s \notin \widehat{S}_{T}^{\ast}\}\big)\leq\sum_{s \in S_{q}^{\ast}}\mathbb{P}\left(s \notin \widehat{S}_{T}^{\ast}\right)=\sum_{s\in S_q^\ast}(1-p_s)^{T}~.
    \end{equation*}
    The right-hand side tends to zero as $T\rightarrow\infty$, so $\mathbb{P}(\widehat{S}_{T}^{\ast}=S_q^\ast)\rightarrow1$. If $\displaystyle q\geq\max_{s\in S^\ast}c_s$, then $S_q^\ast=S^\ast$, which gives the full-recovery claim.
\end{proof}
\subsection{Structural coverage convergence}
\label{subsec:struct_cov_conv}

We now aggregate the per-structure recovery probabilities into a global coverage quantity. The realized structural coverage after $T$ epochs is the random variable
\begin{equation}
    R_{q,T}
    =
    \frac{\left|\widehat{S}_{T}^{\ast}\right|}{\left|S^\ast\right|}
    =
    \frac{1}{\left|S^\ast\right|}
    \sum_{s\in S_{q}^\ast}\mathbf{1}\{s\in\widehat{S}_{T}^{\ast}\}~.
    \label{eq:realized_coverage}
\end{equation}
Thus, $R_{q,T}$ is the fraction of global structures recovered by epoch $T$. Its expectation defines the corresponding deterministic coverage curve. By linearity of expectation and Eq.~\eqref{eq:cumul_prob},
\begin{equation}
    \mathbb{E}[R_{q,T}]
    =
    \frac{1}{\left|S^\ast\right|}
    \sum_{s\in S_{q}^\ast}\mathbb{P}\left(s\in\widehat{S}_{T}^{\ast}\right)
    =
    \frac{1}{\left|S^\ast\right|}
    \sum_{s\in S_{q}^\ast}\left(1-\left(1-p_s\right)^T\right)~.
    \label{eq:expected_coverage}
\end{equation}
\begin{proposition}[Expected structural coverage convergence]
    \label{proposition:expected_coverage_conv}
    We have
    \begin{equation*}
        \lim_{T\rightarrow\infty}\mathbb{E}[R_{q,T}]
        =
        \frac{\left|S_q^\ast\right|}{\left|S^\ast\right|}~.
    \end{equation*}
    In particular, if $\displaystyle q\geq\max_{s\in S^\ast}c_s$, then $\displaystyle\lim_{T\rightarrow\infty}\mathbb{E}[R_{q,T}]=1$.
\end{proposition}
\begin{proof}
    Since $S^\ast$ is finite, we can pass the limit through the finite sum in Eq.~\eqref{eq:expected_coverage}:
    \begin{equation*}
        \lim_{T\rightarrow\infty}\mathbb{E}[R_{q,T}]
        =
        \frac{1}{\left|S^\ast\right|}
        \sum_{s\in S_{q}^\ast}
        \lim_{T\rightarrow\infty}\left(1-\left(1-p_s\right)^T\right)~.
    \end{equation*}
    For $s\in S_q^\ast$, Lemma~\ref{lemma:per_epoch_recovery} gives $p_s>0$, so the corresponding limit is $1$. Therefore,
    \begin{equation*}
        \lim_{T\rightarrow\infty}\mathbb{E}[R_{q,T}]
        =
        \frac{1}{\left|S^\ast\right|}
        \sum_{s\in S_q^\ast}1
        =
        \frac{\left|S_q^\ast\right|}{\left|S^\ast\right|}~.
    \end{equation*}
    If $\displaystyle q\geq\max_{s\in S^\ast}c_s$, then $S_q^\ast=S^\ast$, which gives the full-coverage claim.
\end{proof}

\subsection{Repeated-exposure structural coverage}
\label{subsec:repeated_exposure_coverage}

One-exposure coverage does not distinguish structures observed once from those observed repeatedly. For a threshold $m\in\{1,\ldots,T\}$, define the realized $m$-exposure structural coverage as
\begin{equation}
    R_{q,T}^{(m)}
    =
    \frac{1}{|S^\ast|}
    \sum_{s\in S^\ast}\mathbf{1}\{N_s(T)\geq m\}
    =
    \frac{1}{|S^\ast|}
    \sum_{s\in S_q^\ast}\mathbf{1}\{N_s(T)\geq m\}~.
    \label{eq:repeated_exposure_coverage}
\end{equation}
The second equality follows because $N_s(T)=0$ almost surely for $s\notin S_q^\ast$. Moreover,
\begin{equation}
    R_{q,T}^{(1)}=R_{q,T},
    \label{eq:one_exposure_identity}
\end{equation}
Thus, cumulative structural coverage is the case $m=1$.

\begin{proposition}[Expected repeated-exposure coverage]
    \label{prop:expected_repeated_exposure_coverage}
    For every $m\in\{1,\ldots,T\}$,
    \begin{equation}
        \mathbb{E}\left[R_{q,T}^{(m)}\right]
        =
        \frac{1}{|S^\ast|}
        \sum_{s\in S_q^\ast}
        \left[
            1-
            \sum_{j=0}^{m-1}
            \binom{T}{j}p_s^j(1-p_s)^{T-j}
        \right]~.
        \label{eq:expected_repeated_exposure_coverage}
    \end{equation}
\end{proposition}

\begin{proof}
    By linearity of expectation,
    \[
        \mathbb{E}\left[R_{q,T}^{(m)}\right]
        =
        \frac{1}{|S^\ast|}
        \sum_{s\in S_q^\ast}\mathbb{P}\left(N_s(T)\geq m\right).
    \]
    Proposition~\ref{prop:repeated_structural_exposure} gives $N_s(T)\sim\operatorname{Binomial}(T,p_s)$. Complementing the events $N_s(T)=j$ for $j\in\{0,\ldots,m-1\}$ gives Eq.~\eqref{eq:expected_repeated_exposure_coverage}. Independence between different structures is not required.
\end{proof}

For a fixed partition, define the cluster-span counts and span-specific recovery probabilities as
\begin{equation}
    n_c=\left|\{s\in S^\ast:c_s=c\}\right|,
    \qquad
    p_c(q)=
    \begin{cases}
        \displaystyle\frac{\binom{q-1}{c-1}}{\binom{K-1}{c-1}}, & c\leq q,\\[1.1em]
        0, & c>q.
    \end{cases}
    \label{eq:span_counts_and_probabilities}
\end{equation}
Since $p_s=p_{c_s}(q)$, grouping Eq.~\eqref{eq:expected_repeated_exposure_coverage} by cluster span gives
\begin{equation}
    \mathbb{E}\left[R_{q,T}^{(m)}\right]
    =
    \frac{1}{|S^\ast|}
    \sum_{c=1}^{q}n_c
    \left[
        1-
        \sum_{j=0}^{m-1}
        \binom{T}{j}p_c(q)^j\left(1-p_c(q)\right)^{T-j}
    \right]~.
    \label{eq:repeated_exposure_span_hist}
\end{equation}
For $m=1$, Eq.~\eqref{eq:repeated_exposure_span_hist} reduces to the expected cumulative coverage in Eq.~\eqref{eq:expected_coverage}.

\begin{corollary}[Fixed-threshold coverage convergence]
    \label{cor:fixed_threshold_coverage}
    For every fixed $m\geq1$,
    \begin{equation}
        R_{q,T}^{(m)}
        \longrightarrow
        \frac{|S_q^\ast|}{|S^\ast|}
        \qquad\text{almost surely as }T\to\infty.
        \label{eq:fixed_m_coverage_limit}
    \end{equation}
    Consequently,
    \begin{equation*}
        \lim_{T\to\infty}\mathbb{E}\left[R_{q,T}^{(m)}\right]
        =
        \frac{|S_q^\ast|}{|S^\ast|}~.
    \end{equation*}
\end{corollary}

\begin{proof}
    Every $s\in S_q^\ast$ has $p_s>0$. Eq.~\eqref{eq:exposure_slln} therefore implies $N_s(T)\to\infty$ almost surely. Hence, $\mathbf{1}\{N_s(T)\geq m\}\to1$ almost surely for each fixed $m$. Structures outside $S_q^\ast$ have zero exposure, and $S_q^\ast$ is finite. Eq.~\eqref{eq:repeated_exposure_coverage} then gives the almost-sure limit. Bounded convergence gives the corresponding expectation because $0\leq R_{q,T}^{(m)}\leq1$.
\end{proof}

At finite $T$, exposure differs across cluster spans because the expected count is $Tp_c(q)$.

\subsection{Mean structural exposure rate}
\label{subsec:mean_structural_exposure_rate}

Thresholded coverage discards differences beyond its prescribed number of exposures. We therefore define the realized mean structural exposure rate as
\begin{equation}
    \Phi_{q,T}
    =
    \frac{1}{|S^\ast|}
    \sum_{s\in S^\ast}\frac{N_s(T)}{T}~.
    \label{eq:mean_structural_exposure_rate}
\end{equation}
Equivalently, $\Phi_{q,T}$ is the fraction of structure-epoch pairs in which the corresponding global structure is available.

\begin{proposition}[Mean structural exposure rate]
    \label{prop:mean_structural_exposure_rate}
    For every $T\geq1$,
    \begin{equation}
        \mathbb{E}[\Phi_{q,T}]
        =
        \Phi_q
        :=
        \frac{1}{|S^\ast|}\sum_{s\in S^\ast}p_s
        =
        \frac{1}{|S^\ast|}\sum_{c=1}^{q}n_c p_c(q)~.
        \label{eq:expected_mean_structural_exposure_rate}
    \end{equation}
    Moreover,
    \begin{equation}
        \Phi_{q,T}\longrightarrow\Phi_q
        \qquad\text{almost surely as }T\to\infty.
        \label{eq:mean_structural_exposure_rate_limit}
    \end{equation}
\end{proposition}

\begin{proof}
    Eq.~\eqref{eq:exposure_moments} gives $\mathbb{E}[N_s(T)/T]=p_s$. Linearity of expectation and grouping by cluster span give Eq.~\eqref{eq:expected_mean_structural_exposure_rate}. For each $s\in S^\ast$, Eq.~\eqref{eq:exposure_slln} gives $N_s(T)/T\to p_s$ almost surely. Averaging over the finite set $S^\ast$ gives Eq.~\eqref{eq:mean_structural_exposure_rate_limit}.
\end{proof}

Unlike fixed-threshold coverage, $\Phi_q$ retains span-dependent exposure imbalance as thresholded coverage approaches its limit. Since each epoch contains $M$ optimization steps, the corresponding average frequency among training steps is
\begin{equation}
    \frac{1}{|S^\ast|}
    \sum_{s\in S^\ast}\frac{N_s(T)}{MT}
    =
    \frac{\Phi_{q,T}}{M}
    \longrightarrow
    \frac{\Phi_q}{M}
    \qquad\text{almost surely}.
    \label{eq:mean_structural_exposure_step_rate}
\end{equation}
These quantities alone describe structural availability. They do not determine optimization behavior or predictive performance.

\subsection{Exposure epoch distribution}
\label{subsec:repeated_exposure_epoch}

The exposure count also identifies when a structure accumulates a prescribed number of exposures. For $m\geq1$, define the epoch of the $m$-th exposure as
\begin{equation}
    T_s^{(m)}
    =
    \inf\{t\in\mathbb{N}:N_s(t)\geq m\},
    \label{eq:repeated_exposure_epoch}
\end{equation}
with the convention that the infimum of the empty set is $\infty$.

\begin{proposition}[Epoch of the $m$-th structural exposure]
    \label{prop:repeated_exposure_epoch}
    If $0<p_s<1$, then $T_s^{(m)}$ follows the negative-binomial distribution
    \begin{equation}
        \mathbb{P}\left(T_s^{(m)}=\ell\right)
        =
        \binom{\ell-1}{m-1}p_s^m(1-p_s)^{\ell-m},
        \qquad \ell=m,m+1,\ldots.
        \label{eq:negative_binomial_exposure_epoch}
    \end{equation}
    If $p_s=1$, then $T_s^{(m)}=m$ almost surely. For every $p_s>0$, its expectation and variance are
    \begin{equation}
        \mathbb{E}\left[T_s^{(m)}\right]=\frac{m}{p_s},
        \qquad
        \operatorname{Var}\left[T_s^{(m)}\right]
        =\frac{m(1-p_s)}{p_s^2}~.
        \label{eq:exposure_epoch_moments}
    \end{equation}
    If $p_s=0$, then $T_s^{(m)}=\infty$ almost surely.
\end{proposition}

\begin{proof}
Let $X_s^{(t)} = \mathbf{1}\{s \in S^{(t)}\}$ denote the exposure indicator of structure $s$ at epoch $t$, which forms an i.i.d.\ sequence of $\operatorname{Bernoulli}(p_s)$ random variables. Let $N_s(t) = \sum_{i=1}^t X_s^{(i)}$ be the cumulative exposure count up to epoch $t$.

\paragraph{Case 1: $0 < p_s < 1$.}
For any integer $\ell \ge m$, the event $T_s^{(m)} = \ell$ indicates that the $m$-th exposure of structure $s$ occurs precisely at epoch $\ell$. This occurs if and only if structure $s$ is exposed exactly $m-1$ times in the first $\ell-1$ epochs and is exposed again at epoch $\ell$:
\begin{equation*}
    \{T_s^{(m)} = \ell\} = \{N_s(\ell - 1) = m - 1\} \cap \{X_s^{(\ell)} = 1\}.
\end{equation*}
Because epoch reshufflings are independent, $N_s(\ell - 1)$ and $X_s^{(\ell)}$ are independent random variables. Using $N_s(\ell - 1) \sim \operatorname{Binomial}(\ell - 1, p_s)$, we evaluate:
\begin{align*}
    \mathbb{P}\left(T_s^{(m)} = \ell\right)
    &= \mathbb{P}\left(N_s(\ell - 1) = m - 1\right) \cdot \mathbb{P}\left(X_s^{(\ell)} = 1\right) \\
    &= \left[ \binom{\ell - 1}{m - 1} p_s^{m-1} (1 - p_s)^{(\ell - 1) - (m - 1)} \right] \cdot p_s \\
    &= \binom{\ell - 1}{m - 1} p_s^m (1 - p_s)^{\ell - m}, \qquad \ell = m, m+1, \ldots,
\end{align*}
which establishes Eq.~\eqref{eq:negative_binomial_exposure_epoch}.

To derive the moments in Eq.~\eqref{eq:exposure_epoch_moments}, we represent $T_s^{(m)}$ as the sum of $m$ inter-arrival waiting times:
\begin{equation*}
    T_s^{(m)} = \sum_{k=1}^m W_k,
\end{equation*}
where $W_k$ represents the number of epochs elapsed between the $(k-1)$-th and $k$-th exposure of structure $s$. By the memoryless property of independent Bernoulli trials, $W_1, W_2, \ldots, W_m$ are independent and identically distributed Geometric random variables on $\{1, 2, \ldots\}$ with success probability $p_s$:
\begin{equation*}
    \mathbb{P}(W_k = w) = (1 - p_s)^{w-1} p_s, \qquad w \in \{1, 2, \ldots\}.
\end{equation*}

First, applying linearity of expectation gives $\mathbb{E}\left[T_s^{(m)}\right] = \displaystyle\sum_{k=1}^m \mathbb{E}[W_k]$. Setting $\overline p_s = 1 - p_s$, we compute the expectation of an individual waiting time $\mathbb{E}[W_k]$:
\begin{align*}
    \mathbb{E}[W_k] &= \sum_{w=1}^\infty w (1 - p_s)^{w-1} p_s = p_s \sum_{w=1}^\infty w \overline p_s^{w-1} \\
    &= p_s \frac{\mathrm{d}}{\mathrm{d}\overline p_s} \left( \sum_{w=0}^\infty \overline p_s^w \right) = p_s \frac{\mathrm{d}}{\mathrm{d}\overline p_s} \left( \frac{1}{1-\overline p_s} \right) = p_s \cdot \frac{1}{(1-\overline p_s)^2} = \frac{p_s}{p_s^2} = \frac{1}{p_s}.
\end{align*}
Summing over the $m$ independent inter-arrival periods yields:
\begin{equation*}
    \mathbb{E}\left[T_s^{(m)}\right] = \sum_{k=1}^m \mathbb{E}[W_k] = \sum_{k=1}^m \frac{1}{p_s} = \frac{m}{p_s}.
\end{equation*}

Second, for the variance of the finite sum, we expand:
\begin{equation*}
    \operatorname{Var}\left[T_s^{(m)}\right] = \sum_{k=1}^m \operatorname{Var}(W_k) + 2 \!\!\!\!\!\sum_{1 \le i < j \le m}\!\!\!\!\! \operatorname{Cov}(W_i, W_j).
\end{equation*}
Because the epoch trials determining each $W_k$ are stochastically independent across disjoint epoch intervals, all pairwise covariances vanish ($\operatorname{Cov}(W_i, W_j) = 0$ for $i \neq j$). Thus we have, $\operatorname{Var}\left[T_s^{(m)}\right] = \displaystyle\sum_{k=1}^m \operatorname{Var}(W_k)$.

To determine $\operatorname{Var}(W_k) = \mathbb{E}[W_k^2] - (\mathbb{E}[W_k])^2$, we first evaluate the second factorial moment:
\begin{align*}
    \mathbb{E}[W_k(W_k - 1)] &= \sum_{w=2}^\infty w(w - 1) (1 - p_s)^{w-1} p_s = p_s \overline p_s \sum_{w=2}^\infty w(w - 1) \overline p_s^{w-2} \\
    &= p_s \overline p_s \frac{\mathrm{d}^2}{\mathrm{d}\overline p_s^2} \left( \frac{1}{1-\overline p_s} \right) = p_s \overline p_s \cdot \frac{2}{(1-\overline p_s)^3} = \frac{2 p_s (1 - p_s)}{p_s^3} = \frac{2(1 - p_s)}{p_s^2}.
\end{align*}
The second raw moment is therefore:
\begin{equation*}
    \mathbb{E}[W_k^2] = \mathbb{E}[W_k(W_k - 1)] + \mathbb{E}[W_k] = \frac{2(1 - p_s)}{p_s^2} + \frac{1}{p_s} = \frac{2 - p_s}{p_s^2},
\end{equation*}
which yields the single-variable variance:
\begin{equation*}
    \operatorname{Var}(W_k) = \frac{2 - p_s}{p_s^2} - \frac{1}{p_s^2} = \frac{1 - p_s}{p_s^2}.
\end{equation*}
Summing across $m$ independent variables completes the variance expression in Eq.~\eqref{eq:exposure_epoch_moments}:
\begin{equation*}
    \operatorname{Var}\left[T_s^{(m)}\right] = \sum_{k=1}^m \operatorname{Var}(W_k) = \sum_{k=1}^m \frac{1 - p_s}{p_s^2} = \frac{m(1 - p_s)}{p_s^2}.
\end{equation*}

\paragraph{Case 2: $p_s = 1$.}
If $p_s = 1$, then $X_s^{(t)} = 1$ almost surely for all $t \in \mathbb{N}$, so $N_s(t) = t$ almost surely. The stopping condition $N_s(t) \ge m$ simplifies to $t \ge m$, yielding $T_s^{(m)} = m$ almost surely. Substituting $p_s = 1$ into Eq.~\eqref{eq:exposure_epoch_moments} confirms $\mathbb{E}\left[T_s^{(m)}\right] = m$ and $\operatorname{Var}\left[T_s^{(m)}\right] = 0$.

\paragraph{Case 3: $p_s = 0$.}
If $p_s = 0$, Proposition~\ref{prop:repeated_structural_exposure} gives $N_s(T) = 0$ almost surely for all $T \in \mathbb{N}$. Since $m \ge 1$, the exposure threshold is never satisfied, meaning $\{t \in \mathbb{N} : N_s(t) \ge m\} = \emptyset$ almost surely. By convention, $\inf \emptyset = \infty$, so $T_s^{(m)} = \infty$ almost surely.
\end{proof}

The expected exposure epoch characterizes how cluster span affects the timing of repeated structural exposure. A one-cluster structure reaches its $m$-th exposure at epoch $m$. When $q<K$, observable structures with larger spans reach the same threshold later on average. Structures with $c_s>q$ never reach it.

\subsection{Direct-incidence message availability}
\label{subsec:direct_incidence_message_availability}

Structural exposure affects message passing through the incidences available in each lifted mini-batch. Write $\tau\prec\sigma$ for a direct incidence satisfying $\vertices_\tau\subseteq\vertices_\sigma$. We call $\tau$ the lower structure and $\sigma$ the higher structure. A model may pass messages in either or both directions. We denote these directed instances by $\mu_{\tau\to\sigma}$ and $\mu_{\sigma\to\tau}$.

\begin{definition}[Direct-incidence message compatibility]
    \label{def:direct_incidence_message_compatibility}
    A direct-incidence message rule is compatible with mini-batch lifting if the following condition holds for every global direct incidence $\tau\prec\sigma$. Each direction used by the model is available exactly when the lifted mini-batch contains the higher structure $\sigma$. Equivalently, for $\mu\in\{\mu_{\tau\to\sigma},\mu_{\sigma\to\tau}\}$,
    \begin{equation*}
        \{\mu\mathrm{~is~available~in~mini}\text{-}\mathrm{batch~}i\mathrm{~of~epoch~}t\}
        \iff
        I_\sigma\subseteq Z_i^{(t)}~.
    \end{equation*}
\end{definition}

This condition isolates structural availability. It does not require the numerical value of a message to remain unchanged when it is computed on a mini-batch.

For a directed message $\mu$ associated with $\tau\prec\sigma$, define its per-epoch availability indicator as
\begin{equation}
    J_\mu^{(t)}
    =
    \mathbf{1}\{\mu\text{ is available in epoch }t\}~.
    \label{eq:direct_message_indicator}
\end{equation}
The inclusion $\vertices_\tau\subseteq\vertices_\sigma$ implies $I_\tau\subseteq I_\sigma$. Definitions~\ref{def:lifting_comp} and~\ref{def:direct_incidence_message_compatibility} therefore give
\begin{equation}
    J_\mu^{(t)}=X_\sigma^{(t)}
    \qquad
    \text{for either message direction.}
    \label{eq:message_higher_structure_identity}
\end{equation}
Thus, the unconditional availability of either directed message is governed by the cluster span of its higher endpoint.

For target updates, the relevant quantity is conditional message availability. Let $b(\mu)\in\{\tau,\sigma\}$ denote the target of $\mu$. For any target with $p_{b(\mu)}>0$, define
\begin{equation}
    \alpha_{\mu\mid b(\mu)}
    =
    \mathbb{P}\left(
        J_\mu^{(t)}=1
        \mid
        X_{b(\mu)}^{(t)}=1
    \right).
    \label{eq:conditional_message_availability}
\end{equation}

\begin{proposition}[Conditional availability of a direct-incidence message]
    \label{prop:conditional_direct_message_availability}
    Let $\tau\prec\sigma$. If $c_\sigma\leq q$, then a lower-to-higher message satisfies
    \begin{equation}
        \alpha_{\mu_{\tau\to\sigma}\mid\sigma}=1~.
        \label{eq:lower_to_higher_availability}
    \end{equation}
    If $c_\tau\leq q$, then a higher-to-lower message satisfies
    \begin{equation}
        \alpha_{\mu_{\sigma\to\tau}\mid\tau}
        =
        \frac{p_\sigma}{p_\tau}
        =
        \begin{cases}
            \displaystyle
            \frac{\binom{q-c_\tau}{c_\sigma-c_\tau}}
                 {\binom{K-c_\tau}{c_\sigma-c_\tau}},
            & c_\sigma\leq q,\\[1.1em]
            0,
            & c_\sigma>q.
        \end{cases}
        \label{eq:higher_to_lower_availability}
    \end{equation}
\end{proposition}

\begin{proof}
    Eq.~\eqref{eq:message_higher_structure_identity} shows that either directed message is available exactly when $\sigma$ is recovered. For the lower-to-higher direction, target recovery and message availability are the same event. This gives Eq.~\eqref{eq:lower_to_higher_availability}.

    For the higher-to-lower direction, $I_\tau\subseteq I_\sigma$ implies $\{X_\sigma^{(t)}=1\}\subseteq\{X_\tau^{(t)}=1\}$. Hence,
    \begin{equation*}
        \alpha_{\mu_{\sigma\to\tau}\mid\tau}
        =
        \mathbb{P}\left(X_\sigma^{(t)}=1\mid X_\tau^{(t)}=1\right)
        =
        \frac{p_\sigma}{p_\tau}~.
    \end{equation*}
    Conditional on recovery of $\tau$, the other $q-c_\tau$ batch members form a uniform subset of the $K-c_\tau$ clusters outside $I_\tau$. Message availability requires this subset to contain the fixed set $I_\sigma\setminus I_\tau$. When $c_\sigma\leq q$, its conditional probability is
    \[
        \frac{\binom{K-c_\sigma}{q-c_\sigma}}
             {\binom{K-c_\tau}{q-c_\tau}}
        =
        \frac{\binom{q-c_\tau}{c_\sigma-c_\tau}}
             {\binom{K-c_\tau}{c_\sigma-c_\tau}}~.
    \]
    If $c_\sigma>q$, the higher structure and its incident message are unavailable.
\end{proof}

The two directions have the same unconditional availability but different conditional availability. Every directly incident lower-to-higher message is present when $\sigma$ is updated. A higher-to-lower message is present only when the additional clusters supporting $\sigma$ accompany $I_\tau$. The asymmetry therefore comes from the target-update event, not from a direction-dependent sampling rule.

\subsection{Repeated direct-incidence message exposure}
\label{subsec:repeated_direct_message_exposure}

For a directed message $\mu$ associated with $\tau\prec\sigma$, define its exposure count by
\begin{equation}
    N_\mu(T)=\sum_{t=1}^{T}J_\mu^{(t)}~.
    \label{eq:direct_message_exposure_count}
\end{equation}

\begin{proposition}[Repeated direct-incidence message exposure]
    \label{prop:repeated_direct_message_exposure}
    Assume independent epoch reshuffling and both compatibility conditions. Either direction of a direct-incidence message then satisfies
    \begin{equation}
        N_\mu(T)=N_\sigma(T)
        \sim
        \operatorname{Binomial}(T,p_\sigma)~.
        \label{eq:direct_message_binomial_exposure}
    \end{equation}
    Let $b=b(\mu)$ be an observable target, so that $p_b>0$. For every $n$ with $\mathbb{P}(N_b(T)=n)>0$,
    \begin{equation}
        N_\mu(T)\mid N_b(T)=n
        \sim
        \operatorname{Binomial}\left(n,\alpha_{\mu\mid b}\right).
        \label{eq:conditional_direct_message_count}
    \end{equation}
    Consequently,
    \begin{equation}
        \mathbb{E}\left[N_\mu(T)\mid N_b(T)=n\right]
        =n\alpha_{\mu\mid b},
        \qquad
        \operatorname{Var}\left[N_\mu(T)\mid N_b(T)=n\right]
        =n\alpha_{\mu\mid b}\left(1-\alpha_{\mu\mid b}\right),
        \label{eq:conditional_direct_message_count_moments}
    \end{equation}
    and
    \begin{equation}
        \frac{N_\mu(T)}{N_b(T)}
        \longrightarrow
        \alpha_{\mu\mid b}
        \qquad
        \text{almost surely as }T\to\infty.
        \label{eq:conditional_direct_message_rate_limit}
    \end{equation}
    The ratio in Eq.~\eqref{eq:conditional_direct_message_rate_limit} is considered once $N_b(T)>0$, which occurs eventually almost surely.
\end{proposition}

\begin{proof}
By direct-incidence message compatibility,
\begin{equation*}
    J_\mu^{(t)} = X_\sigma^{(t)}
\end{equation*}
for every epoch $t$, where $X_\sigma^{(t)}$ indicates that the higher structure $\sigma$ is recovered. Hence
\begin{equation*}
    N_\mu(T)=\sum_{t=1}^T J_\mu^{(t)}=\sum_{t=1}^T X_\sigma^{(t)}=N_\sigma(T)~.
\end{equation*}
Under independent epoch reshuffling, the variables $X_\sigma^{(t)}$ are i.i.d. Bernoulli\((p_\sigma)\). Therefore
\begin{equation*}
    N_\mu(T)=N_\sigma(T)\sim\operatorname{Binomial}(T,p_\sigma)~.
\end{equation*}

Now let \(b=b(\mu)\) be the target and assume \(p_b>0\). Since message availability implies target availability,
\begin{equation*}
\mathbb P\!\left(J_\mu^{(t)}=1\mid X_b^{(t)}=1\right)
=
\frac{p_\sigma}{p_b}
=: \alpha_{\mu\mid b}.
\end{equation*}

Since message availability implies target availability, each epoch belongs to one of three categories. Their probabilities are $p_{\sigma}$, $p_{b} - p_{\sigma}$, and $1-p_{b}$. These categories correspond to target with message, target without message, and absent target. They are independent and identically distributed across epochs. Conditional on $N_b(T)=n$, the $n$ epochs in which the
target is present therefore contain a binomial number of message exposures:
\begin{equation*}
N_\mu(T) \mid N_b(T)=n \sim \operatorname{Binomial}(n,\alpha_{\mu\mid b}).
\end{equation*}
Thus
\begin{equation*}
\mathbb E[N_\mu(T)\mid N_b(T)=n]
=n\alpha_{\mu\mid b},
\end{equation*}
and
\begin{equation*}
\operatorname{Var}(N_\mu(T)\mid N_b(T)=n)
=n\alpha_{\mu\mid b}(1-\alpha_{\mu\mid b}).
\end{equation*}

Finally, by the strong law of large numbers we have
\begin{equation*}
\frac{N_\mu(T)}{T}\rightarrow p_\sigma,~~~\frac{N_b(T)}{T}\rightarrow p_b\qquad\text{almost surely as }T\to\infty.
\end{equation*}
Since \(p_b>0\),
\begin{equation*}
\frac{N_\mu(T)}{N_b(T)}=\frac{N_\mu(T)/T}{N_b(T)/T}\rightarrow\frac{p_\sigma}{p_b}=\alpha_{\mu\mid b}\qquad\text{almost surely as }T\to\infty.
\end{equation*}
\end{proof}

The conditional rate $\alpha_{\mu\mid b}$ is the asymptotic fraction of updates of $b$ that can include $\mu$. It equals one for lower-to-higher messages. For higher-to-lower messages, it depends on the additional cluster span of the higher structure. It may remain small even when the lower target is updated frequently.

\subsection{Layerwise incidence aggregation}
\label{subsec:layerwise_incidence_aggregation}

The availability probabilities yield a layerwise consequence. Let $\mathcal{M}_{\mathrm{inc}}(b)$ contain the full-domain direct-incidence messages entering an observable target $b$. Fix a layer $\ell$, its inputs, its parameters, and every message coefficient. Let $\mathbf{m}_\mu^{(\ell)}$ be the resulting full-domain contribution of $\mu\in\mathcal{M}_{\mathrm{inc}}(b)$. The additive direct-incidence component of the full-domain pre-activation is
\begin{equation}
    \mathbf{z}_{b,\mathrm{inc}}^{(\ell+1)}=\!\!\!\!\!\sum_{\mu\in\mathcal{M}_{\mathrm{inc}}(b)}\!\!\!\!\!\mathbf{m}_\mu^{(\ell)}~.
    \label{eq:full_direct_incidence_aggregation}
\end{equation}
When $b$ is recovered in epoch $t$, define the availability-masked component as
\begin{equation}
    \widehat{\mathbf{z}}_{b,\mathrm{inc}}^{(\ell+1,t)}=\!\!\!\!\!\sum_{\mu\in\mathcal{M}_{\mathrm{inc}}(b)}\!\!\!\!\!J_\mu^{(t)}\mathbf{m}_\mu^{(\ell)}~.
    \label{eq:sampled_direct_incidence_aggregation}
\end{equation}

\begin{proposition}[Expected direct-incidence aggregation]
    \label{prop:expected_direct_incidence_aggregation}
    Under the fixed-input availability mask above,
    \begin{equation}
        \mathbb{E}\left[\widehat{\mathbf{z}}_{b,\mathrm{inc}}^{(\ell+1,t)}\mid X_b^{(t)}=1\right]=\!\!\!\!\!\sum_{\mu\in\mathcal{M}_{\mathrm{inc}}(b)}\!\!\!\!\!\alpha_{\mu\mid b}\mathbf{m}_\mu^{(\ell)}~.
        \label{eq:expected_direct_incidence_aggregation}
    \end{equation}
    Therefore, the conditional mean aggregation discrepancy satisfies
    \begin{equation}
        \mathbf{z}_{b,\mathrm{inc}}^{(\ell+1)}-\mathbb{E}\left[\widehat{\mathbf{z}}_{b,\mathrm{inc}}^{(\ell+1,t)}\mid X_b^{(t)}=1\right]=\!\!\!\!\!\sum_{\mu\in\mathcal{M}_{\mathrm{inc}}(b)}\!\!\!\!\!\left(1-\alpha_{\mu\mid b}\right)
        \mathbf{m}_\mu^{(\ell)},
        \label{eq:direct_incidence_mean_discrepancy}
    \end{equation}
    and, for any norm,
    \begin{equation}
        \left\|\mathbf{z}_{b,\mathrm{inc}}^{(\ell+1)}-\mathbb{E}\left[\widehat{\mathbf{z}}_{b,\mathrm{inc}}^{(\ell+1,t)}\mid X_b^{(t)}=1\right]\right\|\leq\!\!\!\!\!\sum_{\mu\in\mathcal{M}_{\mathrm{inc}}(b)}\!\!\!\!\!\left(1-\alpha_{\mu\mid b}\right)\left\|\mathbf{m}_\mu^{(\ell)}\right\|~.
        \label{eq:direct_incidence_mean_discrepancy_bound}
    \end{equation}
\end{proposition}

\begin{proof}
    Conditional linearity of expectation and Eq.~\eqref{eq:conditional_message_availability} give
    \[
        \mathbb{E}\left[
            J_\mu^{(t)}\mathbf{m}_\mu^{(\ell)}
            \mid X_b^{(t)}=1
        \right]
        =
        \alpha_{\mu\mid b}\mathbf{m}_\mu^{(\ell)}
    \]
    for each incoming message. Summing proves Eq.~\eqref{eq:expected_direct_incidence_aggregation}. Subtraction from Eq.~\eqref{eq:full_direct_incidence_aggregation} gives Eq.~\eqref{eq:direct_incidence_mean_discrepancy}. The triangle inequality gives Eq.~\eqref{eq:direct_incidence_mean_discrepancy_bound}. Independence between incident messages is not required.
\end{proof}

For lower-to-higher aggregation into $\sigma$, every direct boundary message has unit availability whenever $\sigma$ is updated. For higher-to-lower aggregation into $\tau$, Eq.~\eqref{eq:higher_to_lower_availability} scales each coface contribution in conditional expectation. The coefficient reduction depends on the extra clusters required by that coface. Vector cancellation prevents these coefficients from implying a monotone change in aggregate norm. Eq.~\eqref{eq:expected_direct_incidence_aggregation}-\eqref{eq:direct_incidence_mean_discrepancy_bound} instead characterize the conditional mean and bound its discrepancy.

This result is deliberately layerwise. Batch-dependent degree, incidence, Hodge, or attention normalization can change numerical message values. Representations produced by earlier sampled layers also depend on the mini-batch. Nonlinear updates prevent the expectation from passing directly through later layers. The proposition isolates structural availability at fixed inputs and parameters. It does not establish unbiasedness or convergence of the complete nonlinear network. Same-rank messages mediated by incidence paths or Hodge products require a separate witness-level analysis and are not covered here.

\subsection{Cumulative recovery entropy}
\label{subsec:entropy_cov_conv}

We associate an entropy with the event that an observable global structure has appeared by epoch $T$. Recall that
\[
    S_q^\ast=\{s\in S^\ast:c_s\leq q\}
\]
is the set of global structures whose cluster span fits within a mini-batch of size $q$. For each $s\in S_q^\ast$, define the cumulative recovery indicator
\[
    Y_s^{(T)}
    =
    \mathbf{1}\{s\in\widehat{S}_{T}^{\ast}\}.
\]
By Eq.~\eqref{eq:cumul_prob},
\[
    Y_s^{(T)}
    \sim
    \operatorname{Bernoulli}(\rho_s(T)),
    \qquad
    \rho_s(T)
    =
    \mathbb{P}\left(s\in\widehat{S}_{T}^{\ast}\right)
    =
    1-(1-p_s)^T.
\]
Let
\[
    h(x)=-x\ln x-(1-x)\ln(1-x)
\]
denote the Bernoulli entropy in nats, with $0\ln 0=0$. Since $Y_s^{(T)}$ has success probability $\rho_s(T)$, its entropy is $h(\rho_s(T))$. We define the aggregate marginal cumulative recovery entropy as
\begin{equation}
    H_{q,T}
    =
    \frac{1}{\left|S^{\ast}\right|}\sum_{s\in S_q^\ast}h\left(\rho_s(T)\right)~.
    \label{eq:cumul_entropy}
\end{equation}
Equivalently, the normalization assigns zero entropy to unobservable structures because their recovery events have probability zero.

\begin{proposition}[Cumulative recovery entropy convergence]
    \label{prop:cumul_entropy_conv}
    The aggregate marginal cumulative recovery entropy satisfies
    \[
        H_{q,T}\rightarrow0
        \qquad
        \text{as}
        \qquad
        T\rightarrow\infty.
    \]
\end{proposition}
\begin{proof}
    For every $s\in S_q^\ast$, Lemma~\ref{lemma:per_epoch_recovery} gives $p_s>0$. Hence
    \[
        \rho_s(T)
        =
        1-(1-p_s)^T
        \rightarrow1
        \qquad
        \text{as}
        \qquad
        T\rightarrow\infty.
    \]
    The function $h$ is continuous on $[0,1]$ and satisfies $h(1)=0$. Hence, $h(\rho_s(T))\rightarrow0$ for every $s\in S_q^\ast$. Summing over the finite set $S_q^\ast$ gives $H_{q,T}\rightarrow0$.
\end{proof}
Thus, the marginal recovery uncertainty of every structure in $S_q^\ast$ vanishes with increasing $T$. Theorem~\ref{theorem:cumulative_recovery} also gives joint recovery of $S_q^\ast$ with probability tending to one. Full recovery occurs when $S_q^\ast=S^\ast$, equivalently when $\displaystyle q\geq\max_{s\in S^\ast}c_s$.

\subsection{Empirical analysis of structural recovery}
\label{subsec:exp_cora}

\subsubsection{Structural recovery across datasets}
\label{subsec:structural_observability_datasets}
\label{subsec:additional_structural_recovery}

\paragraph{Experimental setup.}
A full-graph structure can fit within some group of $q$ clusters if its cluster span satisfies $c_s\leq q$, giving the observable fraction $|S_q^\ast|/|S^\ast|$. Figures~\ref{fig:structural_observability} and~\ref{fig:dotted_q99_cutoff} instead measure how many structures' supporting nodes actually appear together within 200 epochs, with the latter extending the comparison to all four datasets. We fix one METIS partition per dataset, using $K=32$ for Cora Full and Amazon Ratings, $K=500$ for Questions, and $K=2\,000$ for Coauthor Physics. At each epoch, we independently reshuffle all clusters and group them into mini-batches of $q$ clusters. For each structure, we record whether all its supporting nodes occur together in at least one mini-batch by epoch 200. We divide the number of such structures by the total number of full-graph reference structures in the corresponding family, and report the mean and sample standard deviation over ten independent reshuffling runs for each $q$. No model is trained in this diagnostic.

The reference families comprise one centre-indexed neighbourhood hyperedge per node, cellular 2-cells from a fixed full-graph cycle basis with cycle length at most nine, and simplicial 2-cells corresponding to graph triangles. Membership-identical neighbourhood hyperedges retain their centre identities. The cellular curve measures joint sampling of a reference cycle's nodes, not whether recomputing a cycle basis on the induced mini-batch selects the identical cycle. A graph can admit multiple cycle bases, and changing the node or neighbour traversal order can change the basis selected by the lifting, even without changing the graph. Each partition and reference family remains fixed as $q$ varies, so the results characterize the evaluated graph-lifting-partition configurations.

At $q=1$, reshuffling cannot change which nodes occur together, so support recovery is determined entirely by structures contained within individual clusters. At $q=K$, all nodes occur in one batch and support recovery is necessarily complete. Both endpoint values are reached after the first epoch.

\begin{figure*}[!t]
    \centering
    \includegraphics[width=\textwidth]{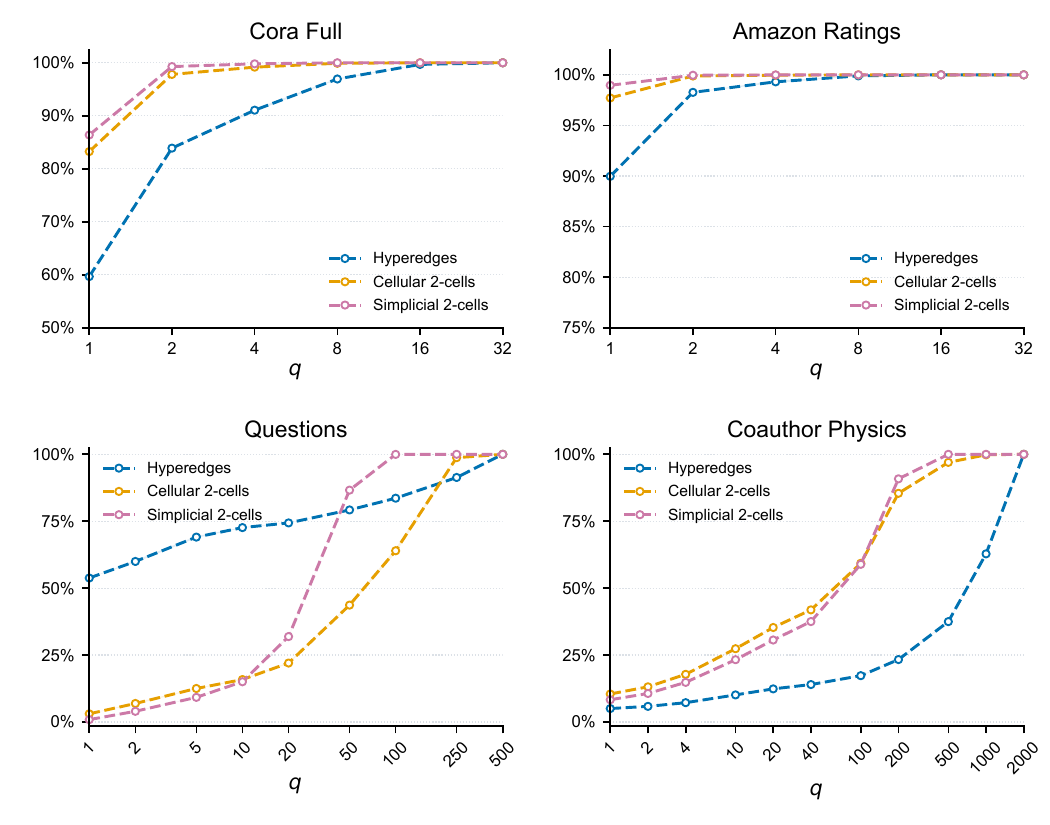}
    \caption{Fraction of full-graph reference structures whose supporting nodes appear together in at least one mini-batch within 200 epochs, as a function of the number of clusters per mini-batch $q$. Panels show Cora Full and Amazon Ratings ($K=32$, top row), Questions ($K=500$, bottom left), and Coauthor Physics ($K=2\,000$, bottom right). Each dataset uses one fixed METIS partition. Curves show means over ten independent reshuffling runs. Shading indicates $\pm1$ standard deviation, generally narrower than the plotted lines. Cellular recovery refers to the supporting nodes of fixed full-graph cycle-basis cells, not their selection by a recomputed local basis. Vertical scales begin at $50\%$ for Cora Full, $75\%$ for Amazon Ratings, and $0\%$ for the other datasets.}
    \label{fig:dotted_q99_cutoff}
\end{figure*}

\paragraph{Cora Full and Amazon Ratings.}
Most cellular and simplicial 2-cells already lie within individual clusters under these partitions. At $q=1$, the recovered fractions are $83.3\%$ and $86.4\%$ on Cora Full, and $97.7\%$ and $99.0\%$ on Amazon Ratings, respectively. Neighbourhood hyperedges have lower initial recovery, $59.7\%$ and $90.0\%$. Increasing $q$ closes much of the remaining gap. At $q=4$, recovery exceeds $99\%$ for both 2-cell families on Cora Full and for all three families on Amazon Ratings. Cora Full neighbourhood hyperedges reach $91.1\%$ at the same value. Thus, these partitions preserve much of the reference structure within single clusters, while recombination exposes additional structures crossing cluster boundaries.

\paragraph{Questions and Coauthor Physics.}
These configurations exhibit a larger gap between observability in principle and recovery within a finite number of epochs. On Questions at $q=20$, all cellular and simplicial 2-cells are observable, yet only $22.0\%$ and $31.9\%$ of their supports are encountered by epoch 200, respectively. On Coauthor Physics at $q=100$, only $59.3\%$ of cellular and $58.9\%$ of simplicial supports are encountered by epoch 200, although each individual reference 2-cell could fit within a mini-batch of this size. Having enough clusters in a mini-batch to accommodate a structure does not ensure that its particular supporting clusters are sampled together within the available epochs. Larger groups increase these opportunities. On Questions at $q=100$, simplicial recovery exceeds $99.9\%$, while cellular recovery is $64.0\%$. On Coauthor Physics at $q=500$, the corresponding fractions exceed $99.9\%$ and reach $97.0\%$.

\paragraph{Differences between higher-order domains.}
The lifting definitions help explain why the three families respond differently to recombination. Each simplicial 2-cell contains three nodes, and each retained cellular 2-cell contains at most nine. A neighbourhood hyperedge contains its centre and all its neighbours, with no corresponding fixed size limit. Its complete support can therefore span many more clusters, making joint sampling less likely. On Coauthor Physics at $q=1\,000$, both 2-cell families exceed $99.8\%$ recovery, whereas neighbourhood hyperedges reach $62.8\%$. However, neighbourhood hyperedges are not uniformly harder to recover. On Questions at $q=1$, $53.8\%$ of neighbourhood hyperedges lie within individual clusters, compared with $3.0\%$ of cellular and $0.9\%$ of simplicial 2-cells. Both support size and its distribution across the partition matter. Low recovery of complete neighbourhoods also does not imply an absence of local hypergraph structure, because lifting a mini-batch constructs hyperedges from the neighbors present. Likewise, a batch-local cycle basis can contain cells different from the full-graph references.

\paragraph{Implications for repeated sampling.}
Addressing the structural component of \ref{rq:structure}, the finite-epoch results distinguish structures that can become available from those encountered within a given sampling horizon. This distinction motivates considering both mini-batch size and repeated sampling when assessing structural exposure. Ensemble inference similarly evaluates nodes under several independently sampled cluster groupings. Table~\ref{tab:partitioning_ablations} provides complementary predictive evidence, with ensemble inference achieving a higher mean test metric than single-pass inference in all 21 reported model-dataset comparisons. These results support evaluating multiple structural contexts, but do not establish that recovery of full-graph reference structures explains the predictive improvement.

\subsubsection{Recovery across epochs on Cora Full}
\label{subsec:recovery_epochs_cora_full}
For liftings compatible with induced subgraphs, the span counts $n_c$ and probabilities $p_c(q)$ from Eq.~\eqref{eq:span_counts_and_probabilities} give the expected one-exposure coverage

\begin{equation}
    \mathbb{E}[R_{q,T}]
    =
    \frac{1}{\left|S^\ast\right|}
    \sum_{c=1}^{q}
    n_{c}\left(1-\left(1-p_c(q)\right)^T\right)~.
    \label{eq:coverage_span_hist}
\end{equation}

\begin{figure*}[!h]
    \centering
    \setlength{\abovecaptionskip}{4pt}
    \includegraphics[width=\textwidth]{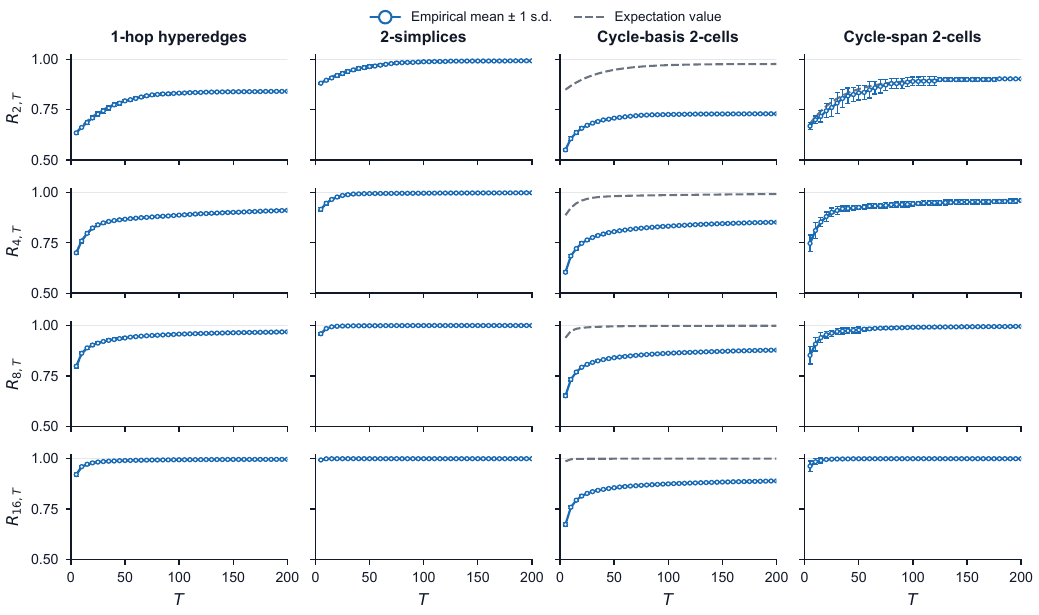}
    \caption{Comparison of empirical and analytical structural recovery on Cora Full with $K=32$ and $q\in\{2,4,8,16\}$. Empirical coverage is shown as the mean $\pm1$ sample standard deviation across ten independent runs and compared with the analytical coverage from Eq.~\eqref{eq:coverage_span_hist} over 200 epochs.}    
    \label{fig:recovery-dynamics}
\end{figure*}

We also evaluate the aggregate marginal entropy of cumulative support availability from the same span counts:
\begin{equation}
    H_{q,T}
    =
    \frac{1}{\left|S^{\ast}\right|}\sum_{c=1}^{q}
    n_{c}\,h\left(1-\left(1-p_c(q)\right)^T\right)~,
    \label{eq:entropy_span_hist}
\end{equation}
where $h$ is the Bernoulli entropy defined in Sec.~\ref{subsec:entropy_cov_conv}. Eq.~\eqref{eq:repeated_exposure_span_hist} provides the corresponding expression for any exposure threshold $m\in\{1,\ldots,T\}$.

We evaluate Cora Full using one fixed METIS partition with $K=32$ and mini-batches of $q\in\{2,4,8,16\}$ clusters. For each $q$, we simulate ten independent cluster-reshuffling sequences, each lasting 200 epochs and using a different seed. The reference families comprise center-indexed $1$-hop hyperedges, graph triangles, full-graph cycle-basis cells of length at most nine, and cycle-span cells.
\begin{figure*}[!h]
    \centering
    \setlength{\abovecaptionskip}{4pt}
    \includegraphics[width=\textwidth]{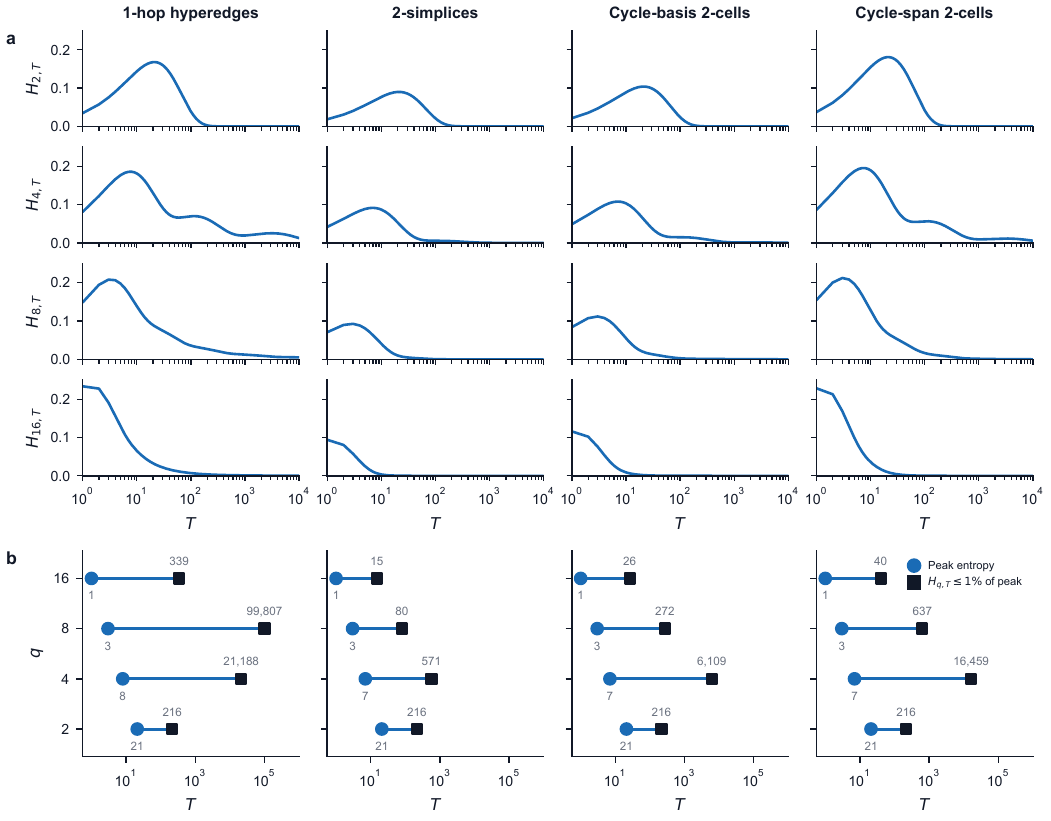}
    \caption{Analytical entropy of cumulative support availability on Cora Full ($K=32$). \textbf{a} Eq.~\eqref{eq:entropy_span_hist} for $T\leq10^4$. \textbf{b} Blue circles mark peak epochs, while black squares mark the final crossing below $1\%$ of peak entropy.}
    \label{fig:entropy}
\end{figure*}
The cycle-basis lifting need not satisfy the compatibility condition in Definition~\ref{def:lifting_comp}: recomputing a basis can omit a reference cycle even when its complete support is present. We therefore include the cycle-span family as a compatible reference for testing the support-based recovery prediction. This family comprises all distinct undirected simple cycles of lengths three to eight. Each cycle appears in an induced mini-batch whenever all its supporting clusters occur together, independently of any basis selection. We count recovery exactly by grouping cycles with the same supporting cluster set. For cycle-basis cells, we instead record whether the recomputed mini-batch basis selects each full-graph reference cycle. The expanded family is used only in this structural diagnostic. CWN and TopoTune retain the cycle-basis lifting described in Appendix~\ref{app:experimental_liftings}.

Figure~\ref{fig:recovery-dynamics} compares cumulative empirical recovery with Eq.~\eqref{eq:coverage_span_hist} for these reference families. Empirical means agree closely with the prediction for $1$-hop hyperedges, triangles, and expanded cycles. Cycle-basis recovery falls below the support-based prediction because restricting the graph can change which cycles the basis selects.

Figure~\ref{fig:entropy} evaluates Eq.~\eqref{eq:entropy_span_hist} analytically from each reference family's fixed cluster-span histogram, normalized by the total number of reference structures. It measures the average marginal uncertainty about whether each reference structure's complete support has appeared together by epoch $T$. For cycle-basis cells, support availability does not guarantee selection by a recomputed local basis. Their entropy curves therefore describe support availability rather than the actual basis recovery measured in Figure~\ref{fig:recovery-dynamics}.

We define the decay epoch as the earliest epoch after the entropy peak from which entropy remains at or below $1\%$ of its maximum. Both the peak and decay epochs are analytical quantities, including beyond the 200-epoch sampling horizon. They do not measure training convergence.

\section{Additional results}
\label{sec:additional_results}

\paragraph{GPU memory across partition configurations.}
Figure~\ref{fig:qk-gpu-memory} examines how partition granularity and mini-batch size affect peak GPU memory on Cora Full, Amazon Ratings, and Questions. We evaluate $K\in\{32,64,128,256\}$ with $q/K\in\{0.125,0.25,0.5,0.75\}$. For each model-dataset pair, we fix the model hyperparameters to the configuration selected for \our training and use the same settings for the full-graph baseline. We report peak CUDA allocated memory across training and validation in decimal GB. The measurements characterize memory sensitivity to $K$ and $q$ at fixed model hyperparameters.

On Cora Full, peak GPU memory increases with the sampled cluster fraction across all seven models. At $q/K=0.125$, Cluster-TNN reduces peak memory by 85.4-88.7\% relative to the corresponding full-graph baseline across the four partition counts. At $q/K=0.75$, these reductions narrow to 24.7-29.7\%. For SCCNN with $K=32$, increasing $q$ from $4$ to $24$ raises peak memory from 4.08 to 22.05~GB, compared with 30.58~GB for full-graph execution. Amazon Ratings and Questions exhibit the same qualitative trend. These results illustrate how Cluster-TNN can be configured through $K$ and $q$ to balance GPU memory use and the amount of graph context processed per mini-batch.

\begin{figure}[!h]
    \centering
    \setlength{\abovecaptionskip}{4pt}
    \includegraphics[width=\textwidth]{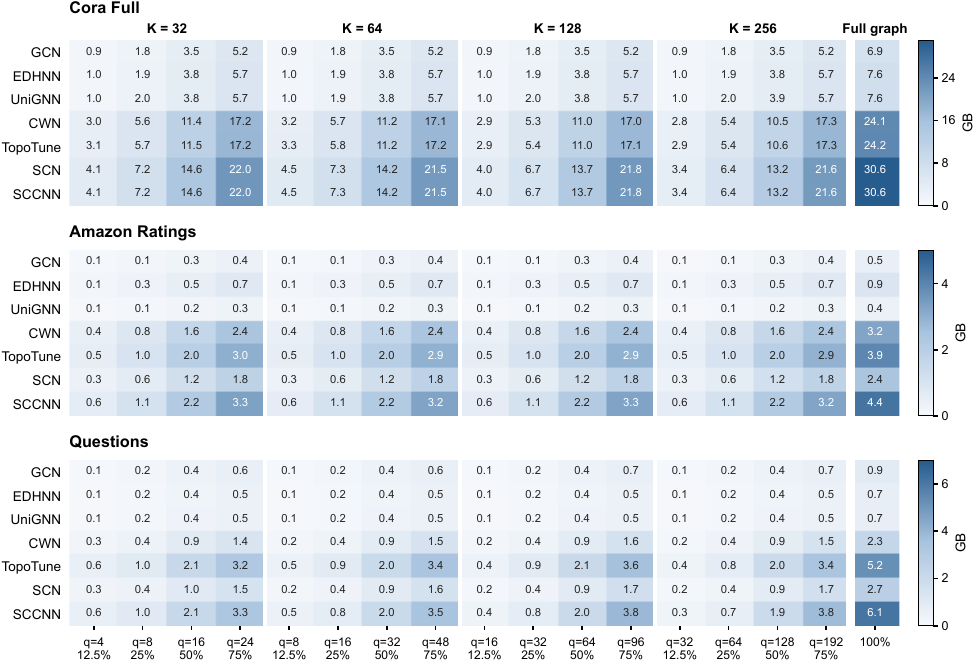}
    \caption{Peak GPU memory under Cluster-TNN and full-graph execution on
    three datasets.
    Rows correspond to models, while columns vary the number of partitions $K$
    and clusters per mini-batch $q$, with percentages indicating $q/K$.
    The rightmost column shows the full-graph baseline. Each dataset has a separate color scale.}
    \label{fig:qk-gpu-memory}
\end{figure}
\FloatBarrier

\section{Experimental protocol and hyperparameter optimization}
\label{sec:hyper_opt}

\paragraph{Dataset splits and random seeds.}
Our predictive-performance experiments use transductive node classification with generated, non-stratified random node splits rather than predefined public splits. The training/validation/test proportions are $70\%/15\%/15\%$ for Cora Full, Coauthor Physics, and OGBN Products, and $50\%/25\%/25\%$ for Amazon Ratings, Questions, and Reddit. For each dataset, we generate ten random splits with a fixed seed of $42$ and use split seeds $0,1,2,3,4$ for final evaluation. The same splits are used across models and between full-graph and Cluster-TNN execution wherever both are evaluated. For run $s$, the corresponding seed also controls model initialization, stochastic training, and cluster recombinations during partitioned training and inference. Thus, variation across the five runs reflects both data-split and model-training randomness.

\paragraph{Search objective and procedure.}
We optimize each model-dataset pair separately for full-graph and \our partitioned training using Optuna's TPESampler~\citep{optuna,watanabe2023tpe}. Each trial maximizes the best validation metric attained during training. Test data are not evaluated during hyperparameter optimization. Dataset-specific metrics are listed in Table~\ref{tab:datasets}. Graph and lifted-domain structural counts are reported separately in Table~\ref{tab:tdl_structural_counts}.

Each study comprises $100$ trials for Questions, Amazon Ratings, Cora Full, and Coauthor Physics, $50$ trials for Reddit, and $30$ trials for OGBN Products. The first ten trials use random sampling, after which the sampler uses previous trial outcomes to guide subsequent choices. Search trials run for at most $300$ epochs, with validation every five epochs and early stopping after five validation checks without improvement.

\begin{table}[!htbp]
  \centering
  \caption{Dataset tasks, input feature dimensions, and class counts. Structural counts are reported separately in Table~\ref{tab:tdl_structural_counts}.}
  \label{tab:datasets}
  \vspace{2pt}
  \small
  \setlength{\tabcolsep}{14pt}
  \begin{tabular}{lccc}
    \toprule
    \textbf{Dataset} & \textbf{Metric} & \makecell{\textbf{Input feature}\\\textbf{dimension}} & \textbf{Classes} \\
    \midrule
    Cora Full        & Accuracy & $8\,710$  & $70$ \\
    Amazon Ratings   & Accuracy & $300$    & $5$  \\
    Questions        & AUROC    & $301$    & $2$  \\
    Coauthor Physics & Accuracy & $8\,415$  & $5$  \\
    Reddit           & Accuracy & $602$    & $41$ \\
    OGBN Products    & Accuracy & $100$    & $47$ \\
    \bottomrule
  \end{tabular}
\end{table}

\paragraph{Model hyperparameters.}
The model search varies optimization step size, regularization, and representation width. Learning rates are selected from $\{10^{-5},10^{-4},10^{-3},10^{-2}\}$, covering several orders of magnitude. Weight decay is selected from $\{0,10^{-5},10^{-4},10^{-3}\}$, allowing comparison between no weight decay and different regularization strengths. We search encoder widths of $\{32,64,128\}$ channels to vary representation capacity and the memory occupied by hidden features during model execution.

Projection dropout is selected from $\{0,0.1,0.2,0.3\}$. For GCN, EDHNN, and TopoTune, we additionally search backbone dropout in $\{0,0.2,0.5\}$. Both dropout spaces include a setting without dropout. These model search spaces are shared across datasets and between full-graph and \our training wherever the corresponding parameter applies.

\paragraph{Partition hyperparameters.}
For partitioned training, we jointly optimize the number of graph partitions $K$ and the number of clusters combined into each training mini-batch $q$. Increasing $K$ produces smaller clusters on average. For a fixed partition, increasing $q$ expands the sampled graph context and allows higher-order structures spanning more clusters to be recovered, but can also increase lifting and model-execution costs. Searching both parameters allows validation-based selection among different partition granularities and amounts of graph context.

Table~\ref{tab:partition_search_spaces} lists the dataset-specific candidate sets. Within each dataset, all models use the same $K,q$ search space.

\begin{table}[!htbp]
  \centering
  \caption{Dataset-specific partition search spaces. $K$ denotes the total number of graph partitions, and $q$ denotes the number of clusters combined into one training mini-batch. These parameters are optimized jointly with the model hyperparameters.}
  \label{tab:partition_search_spaces}
  \small
  \setlength{\tabcolsep}{8pt}
  \begin{tabular}{lcc}
    \toprule
    \textbf{Dataset} & \textbf{Candidate $K$ values} & \textbf{Candidate $q$ values} \\
    \midrule
    Questions & $\{500, 1\,000, 1\,500, 2\,000\}$ & $\{20, 30, 40, 50\}$ \\
    Amazon Ratings & $\{32, 64, 128, 256\}$ & $\{4, 8, 16, 32\}$ \\
    Cora Full & $\{32,64,128,256\}$ & $\{4, 8, 16, 32\}$ \\
    Coauthor Physics & $\{500, 1\,000,1\, 500, 2\,000\}$ & $\{20, 30, 40, 50\}$ \\
    Reddit & $\{7\,000, 10\,000\}$ & $\{20, 25\}$ \\
    OGBN Products & $\{15\,000, 20\,000\}$ & $\{20, 25\}$ \\
    \bottomrule
  \end{tabular}
\end{table}

With the same number of partitions $K$, a graph containing more nodes has more nodes per cluster on average. Combining $q$ clusters can therefore produce very different mini-batch sizes across datasets. Connectivity further determines how many higher-order cells the lifting constructs from those nodes. These differences affect the time and memory required to construct and process each mini-batch.

We selected the dataset-specific candidate ranges through preliminary runs assessing the memory required for mini-batch lifting and model execution.

\paragraph{Selected configurations and final evaluation.}
Table~\ref{tab:optuna_hparams} reports the selected hyperparameters. Final evaluations use five seeds. For OGBN Products, final training runs for at most 130 epochs, with validation every five epochs and early stopping after two validation checks without improvement. For final training and evaluation, validation and test mini-batches use the selected training $q$, as reported in Table~\ref{tab:optuna_hparams}. Test predictions for P are obtained using ensemble inference with ten independent cluster recombinations.

{\footnotesize
\setlength{\tabcolsep}{1.9pt}
\renewcommand{\arraystretch}{1.05}
\setlength{\LTleft}{\fill}
\setlength{\LTright}{\fill}
\setlength{\LTcapwidth}{\textwidth}
\makeatletter
\patchcmd{\LT@makecaption}{\reset@font}{\reset@font\normalsize}{}{}
\makeatother
\newcommand{\NAcell}{\cellcolor{gray!7}}
\setlength{\LTcapwidth}{\textwidth}
\begin{longtable}{lcccccccc}
\caption{Selected hyperparameters for full-graph and \our training. Shading denotes parameters that do not apply to the corresponding model or protocol.}
\label{tab:optuna_hparams}\\
\toprule
\textbf{Dataset} & \textbf{Protocol} & \makecell{\textbf{Learning} \\ \textbf{Rate}} & \makecell{\textbf{Weight}\\\textbf{Decay}} & \makecell{\textbf{Hidden}\\\textbf{Channels}} & \makecell{\textbf{Projection}\\\textbf{Dropout}} & $\mathbf{K}$ & $\mathbf{q}$ & \textbf{Dropout} \\
\midrule
\endfirsthead
\multicolumn{9}{c}{\tablename~\thetable{} continued} \\
\toprule
\textbf{Dataset} & \textbf{Protocol} & \makecell{\textbf{Learning} \\ \textbf{Rate}} & \makecell{\textbf{Weight}\\\textbf{Decay}} & \makecell{\textbf{Hidden}\\\textbf{Channels}} & \makecell{\textbf{Projection}\\\textbf{Dropout}} & $\mathbf{K}$ & $\mathbf{q}$ & \textbf{Dropout} \\
\midrule
\endhead
\midrule
\multicolumn{9}{r}{Continued on next page} \\
\endfoot
\bottomrule
\endlastfoot

\rowcolor{gray!14}\multicolumn{9}{c}{\textbf{GCN}} \\*
\addlinespace[1pt]
\multirow{1}{*}{Questions} & Full-graph & $10^{-2}$ & $10^{-4}$ & $128$ & $0.3$ & \NAcell & \NAcell & $0.5$ \\*
& \our & $10^{-4}$ & $0$ & $128$ & $0.2$ & $500$ & $20$ & $0.2$ \\*
\multirow{1}{*}{Amazon Ratings} & Full-graph & $10^{-2}$ & $10^{-3}$ & $128$ & $0$ & \NAcell & \NAcell & $0$ \\*
& \our & $10^{-2}$ & $10^{-4}$ & $128$ & $0.3$ & $64$ & $4$ & $0$ \\*
\multirow{1}{*}{Cora Full} & Full-graph & $10^{-3}$ & $10^{-5}$ & $128$ & $0.3$ & \NAcell & \NAcell & $0$ \\*
& \our & $10^{-3}$ & $10^{-5}$ & $128$ & $0.3$ & $32$ & $8$ & $0$ \\*
Coauthor Physics & \our & $10^{-2}$ & $10^{-4}$ & $64$ & $0.3$ & $500$ & $20$ & $0.2$ \\*
Reddit & \our & $10^{-4}$ & $10^{-4}$ & $128$ & $0.2$ & $7\,000$ & $25$ & $0.2$ \\*
OGBN Products & \our & $10^{-4}$ & $10^{-4}$ & $128$ & $0.1$ & $15\,000$ & $25$ & $0.2$ \\
\addlinespace[2pt]

\rowcolor{gray!14}\multicolumn{9}{c}{\textbf{EDHNN}} \\*
\addlinespace[1pt]
\multirow{1}{*}{Questions} & Full-graph & $10^{-3}$ & $10^{-3}$ & $128$ & $0$ & \NAcell & \NAcell & $0.5$ \\*
& \our & $10^{-2}$ & $10^{-3}$ & $32$ & $0.2$ & $500$ & $50$ & $0$ \\*
\multirow{1}{*}{Amazon Ratings} & Full-graph & $10^{-3}$ & $10^{-3}$ & $128$ & $0.1$ & \NAcell & \NAcell & $0$ \\*
& \our & $10^{-3}$ & $10^{-4}$ & $128$ & $0.3$ & $128$ & $4$ & $0$ \\*
\multirow{1}{*}{Cora Full} & Full-graph & $10^{-3}$ & $10^{-3}$ & $128$ & $0.3$ & \NAcell & \NAcell & $0$ \\*
& \our & $10^{-2}$ & $10^{-4}$ & $64$ & $0.2$ & $32$ & $4$ & $0.2$ \\*
Coauthor Physics & \our & $10^{-2}$ & $0$ & $128$ & $0.2$ & $2\,000$ & $30$ & $0.5$ \\*
Reddit & \our & $10^{-2}$ & $10^{-4}$ & $128$ & $0.3$ & $7\,000$ & $20$ & $0.2$ \\*
OGBN Products & \our & $10^{-3}$ & $10^{-4}$ & $128$ & $0.1$ & $20\,000$ & $25$ & $0.2$ \\
\addlinespace[2pt]

\rowcolor{gray!14}\multicolumn{9}{c}{\textbf{UniGNN}} \\*
\addlinespace[1pt]
\multirow{1}{*}{Questions} & Full-graph & $10^{-2}$ & $0$ & $32$ & $0$ & \NAcell & \NAcell & \cellcolor{gray!7} \\*
& \our & $10^{-3}$ & $10^{-3}$ & $128$ & $0$ & $1\,000$ & $40$ & \cellcolor{gray!7} \\*
\multirow{1}{*}{Amazon Ratings} & Full-graph & $10^{-2}$ & $10^{-3}$ & $128$ & $0$ & \NAcell & \NAcell & \cellcolor{gray!7} \\*
& \our & $10^{-2}$ & $10^{-4}$ & $64$ & $0$ & $32$ & $8$ & \cellcolor{gray!7} \\*
\multirow{1}{*}{Cora Full} & Full-graph & $10^{-3}$ & $10^{-4}$ & $128$ & $0.3$ & \NAcell & \NAcell & \cellcolor{gray!7} \\*
& \our & $10^{-4}$ & $10^{-4}$ & $128$ & $0.1$ & $32$ & $4$ & \cellcolor{gray!7} \\*
Coauthor Physics & \our & $10^{-2}$ & $10^{-5}$ & $32$ & $0.2$ & $2\,000$ & $20$ & \cellcolor{gray!7} \\*
Reddit & \our & $10^{-3}$ & $10^{-3}$ & $64$ & $0.3$ & $7\,000$ & $25$ & \NAcell \\*
OGBN Products & \our & $10^{-4}$ & $10^{-4}$ & $128$ & $0$ & $15\,000$ & $20$ & \NAcell \\
\addlinespace[1pt]

\rowcolor{gray!14}\multicolumn{9}{c}{\textbf{CWN}} \\*
\addlinespace[1pt]
\multirow{1}{*}{Questions} & Full-graph & $10^{-2}$ & $10^{-3}$ & $128$ & $0.3$ & \NAcell & \NAcell & \cellcolor{gray!7} \\*
& \our & $10^{-2}$ & $10^{-3}$ & $32$ & $0.3$ & $1\,000$ & $40$ & \cellcolor{gray!7} \\*
\multirow{1}{*}{Amazon Ratings} & Full-graph & $10^{-2}$ & $10^{-3}$ & $128$ & $0.3$ & \NAcell & \NAcell & \cellcolor{gray!7} \\*
& \our & $10^{-2}$ & $10^{-3}$ & $128$ & $0.2$ & $64$ & $16$ & \cellcolor{gray!7} \\*
\multirow{1}{*}{Cora Full} & Full-graph & $10^{-2}$ & $10^{-3}$ & $64$ & $0.3$ & \NAcell & \NAcell & \cellcolor{gray!7} \\*
& \our & $10^{-2}$ & $10^{-3}$ & $32$ & $0.3$ & $32$ & $4$ & \cellcolor{gray!7} \\*
Coauthor Physics & \our & $10^{-2}$ & $0$ & $32$ & $0.2$ & $2\,000$ & $20$ & \cellcolor{gray!7} \\*
Reddit & \our & $10^{-2}$ & $10^{-4}$ & $64$ & $0.3$ & $7\,000$ & $20$ & \NAcell \\*
OGBN Products & \our & $10^{-4}$ & $10^{-4}$ & $128$ & $0$ & $15\,000$ & $20$ & \NAcell \\
\addlinespace[1pt]

\rowcolor{gray!14}\multicolumn{9}{c}{\textbf{TopoTune}} \\*
\addlinespace[1pt]
\multirow{1}{*}{Questions} & Full-graph & $10^{-2}$ & $10^{-5}$ & $128$ & $0.1$ & \NAcell & \NAcell & $0$ \\*
& \our & $10^{-2}$ & $10^{-3}$ & $128$ & $0.3$ & $500$ & $50$ & $0.2$ \\*
\multirow{1}{*}{Amazon Ratings} & Full-graph & $10^{-2}$ & $10^{-5}$ & $64$ & $0.1$ & \NAcell & \NAcell & $0.2$ \\*
& \our & $10^{-2}$ & $10^{-5}$ & $128$ & $0$ & $32$ & $8$ & $0.5$ \\*
\multirow{1}{*}{Cora Full} & Full-graph & $10^{-2}$ & $10^{-3}$ & $128$ & $0.2$ & \NAcell & \NAcell & $0.2$ \\*
& \our & $10^{-2}$ & $10^{-3}$ & $128$ & $0.2$ & $64$ & $8$ & $0.5$ \\*
Coauthor Physics & \our & $10^{-3}$ & $10^{-3}$ & $32$ & $0.3$ & $2\,000$ & $20$ & $0.5$ \\*
Reddit & \our & $10^{-3}$ & $10^{-4}$ & $64$ & $0.2$ & $7\,000$ & $20$ & $0.2$ \\*
OGBN Products & \our & $10^{-3}$ & $10^{-5}$ & $64$ & $0.2$ & $15\,000$ & $25$ & $0.5$ \\
\addlinespace[2pt]

\rowcolor{gray!14}\multicolumn{9}{c}{\textbf{SCN}} \\*
\addlinespace[1pt]
\multirow{1}{*}{Questions} & Full-graph & $10^{-2}$ & $10^{-4}$ & $128$ & $0.3$ & \NAcell & \NAcell & \cellcolor{gray!7} \\*
& \our & $10^{-3}$ & $10^{-3}$ & $64$ & $0.1$ & $500$ & $50$ & \cellcolor{gray!7} \\*
\multirow{1}{*}{Amazon Ratings} & Full-graph & $10^{-2}$ & $0$ & $64$ & $0.3$ & \NAcell & \NAcell & \cellcolor{gray!7} \\*
& \our & $10^{-3}$ & $0$ & $128$ & $0.2$ & $32$ & $8$ & \cellcolor{gray!7} \\*
\multirow{1}{*}{Cora Full} & Full-graph & $10^{-3}$ & $10^{-3}$ & $128$ & $0.3$ & \NAcell & \NAcell & \cellcolor{gray!7} \\*
& \our & $10^{-3}$ & $10^{-4}$ & $128$ & $0.3$ & $32$ & $16$ & \cellcolor{gray!7} \\*
Coauthor Physics & \our & $10^{-2}$ & $0$ & $32$ & $0.3$ & $2\,000$ & $20$ & \cellcolor{gray!7} \\*
Reddit & \our & $10^{-3}$ & $10^{-3}$ & $64$ & $0.3$ & $7\,000$ & $25$ & \NAcell \\*
OGBN Products & \our & $10^{-4}$ & $10^{-4}$ & $128$ & $0$ & $15\,000$ & $20$ & \NAcell \\
\addlinespace[1pt]

\rowcolor{gray!14}\multicolumn{9}{c}{\textbf{SCCNN}} \\*
\addlinespace[1pt]
\multirow{1}{*}{Questions} & Full-graph & $10^{-3}$ & $10^{-5}$ & $128$ & $0.1$ & \NAcell & \NAcell & \cellcolor{gray!7} \\*
& \our & $10^{-3}$ & $10^{-5}$ & $128$ & $0.3$ & $500$ & $30$ & \cellcolor{gray!7} \\*
\multirow{1}{*}{Amazon Ratings} & Full-graph & $10^{-3}$ & $0$ & $128$ & $0.3$ & \NAcell & \NAcell & \cellcolor{gray!7} \\*
& \our & $10^{-3}$ & $10^{-3}$ & $128$ & $0.2$ & $32$ & $4$ & \cellcolor{gray!7} \\*
\multirow{1}{*}{Cora Full} & Full-graph & $10^{-3}$ & $10^{-3}$ & $128$ & $0.3$ & \NAcell & \NAcell & \cellcolor{gray!7} \\*
& \our & $10^{-3}$ & $10^{-3}$ & $128$ & $0.1$ & $32$ & $8$ & \cellcolor{gray!7} \\*
Coauthor Physics & \our & $10^{-2}$ & $0$ & $128$ & $0.3$ & $2\,000$ & $20$ & \cellcolor{gray!7} \\*
Reddit & \our & $10^{-3}$ & $10^{-3}$ & $64$ & $0.3$ & $7\,000$ & $20$ & \NAcell \\*
OGBN Products & \our & $10^{-4}$ & $10^{-4}$ & $128$ & $0$ & $15\,000$ & $20$ & \NAcell \\
\addlinespace[1pt]

\end{longtable}
}
\end{document}